\documentclass{article}

     \PassOptionsToPackage{numbers}{natbib}

\usepackage[final]{neurips_2020}

\usepackage[utf8]{inputenc} 
\usepackage[T1]{fontenc}    
\usepackage[colorlinks=true, linkcolor=blue, citecolor=blue,urlcolor=black]{hyperref}       
\usepackage{url}            
\usepackage{booktabs}       
\usepackage{amsfonts}       
\usepackage{nicefrac}       
\usepackage{microtype}      

\newif\ifspacehack

\usepackage{graphicx}
\usepackage{mathtools}
\usepackage{footnote}
\usepackage{float}
\usepackage{xspace}
\usepackage{multirow}
\usepackage{xcolor}
\usepackage{wrapfig}
\usepackage{framed}
\usepackage{bbm}
\usepackage{footnote}
\usepackage{nicefrac}
\usepackage{makecell}
\usepackage[algo2e]{algorithm2e} 
\usepackage{algorithm}
\usepackage{amssymb}
\usepackage{amsthm}
\usepackage{bm}
\makesavenoteenv{tabular}
\makesavenoteenv{table}

\newtheorem{theorem}{Theorem}[section]
\newtheorem{corollary}{Corollary}[theorem]
\newtheorem{lemma}[theorem]{Lemma}
\newtheorem{remark}{Remark}

\definecolor{Green}{rgb}{0.13, 0.65, 0.3}
\usepackage[]{color-edits}

\addauthor{HL}{Green}
\addauthor{MZ}{cyan}
\addauthor{CL}{orange}

\newcommand{\expthree}{\textsc{Exp3}\xspace}

\newcommand{\scrible}{\textsc{SCRiBLe}\xspace}

\newcommand{\calS}{{\mathcal{S}}}
\newcommand{\calF}{{\mathcal{F}}}

\newcommand{\calK}{{\mathcal{K}}}

\newcommand{\calT}{{\mathcal{T}}}
\newcommand{\calP}{{\mathcal{P}}}

\newcommand{\Reg}{\text{\rm Reg}}

\newcommand{\one}{\boldsymbol{1}}

\newcommand{\Lstar}{{L^\star}}
\newcommand{\istar}{{i^\star}}
\newcommand{\Bstar}{{B^\star}}
\newcommand{\Ustar}{\rho}
\newcommand{\Aconst}{a}
\newcommand{\dplus}[1]{\bm{#1}}
\newcommand{\lambdamax}{\lambda_\text{\rm max}}
\newcommand{\biasone}{\textsc{Deviation}\xspace}
\newcommand{\bias}{\textsc{Bias-1}\xspace}
\newcommand{\biastwo}{\textsc{Bias-2}\xspace}
\newcommand{\biasthree}{\textsc{Bias-3}\xspace}
\newcommand{\errorterm}{\textsc{Error}\xspace}
\newcommand{\regterm}{\textsc{Reg-Term}\xspace}

\newcommand{\Bomega}{B_{\Omega}}
\newcommand{\UOB}{UOB-REPS\xspace}
\newcommand{\Holder}{{H{\"o}lder}\xspace}

\newcommand{\dpw}{\dplus{w}}
\newcommand{\dpu}{\dplus{u}}
\newcommand{\dpwtilde}{\dplus{\wtilde}}
\newcommand{\dps}{\dplus{s}}
\newcommand{\dpe}{\dplus{e}}
\newcommand{\dpH}{\dplus{H}}
\newcommand{\dpOmega}{\dplus{\Omega}}
\newcommand{\dpellhat}{\dplus{\ellhat}}
\newcommand{\dpell}{\dplus{\ell}}
\newcommand{\dpr}{\dplus{r}}
\newcommand{\dpxi}{\dplus{\xi}}
\newcommand{\dpv}{\dplus{v}}
\newcommand{\dpI}{\dplus{I}}
\newcommand{\dpA}{\dplus{A}}
\newcommand{\dph}{\dplus{h}}
\newcommand{\cprob}{6}

\DeclareMathOperator*{\argmin}{argmin}
\DeclareMathOperator*{\argmax}{argmax}
\DeclarePairedDelimiter\ceil{\lceil}{\rceil}

\newcommand{\field}[1]{\mathbb{#1}}

\newcommand{\fR}{\field{R}}

\newcommand{\fS}{\field{S}}
\newcommand{\E}{\field{E}}

\newcommand{\inner}[1]{ \left\langle {#1} \right\rangle }
\newcommand{\inn}[1]{ \langle {#1} \rangle }

\newcommand{\norm}[1]{\left\|{#1}\right\|}

\newcommand{\wh}{\widehat}
\newcommand{\wt}{\widetilde}

\newcommand{\Jd}{J}
\newcommand{\ellhat}{\wh{\ell}}

\newcommand{\wtilde}{\wt{w}}
\newcommand{\what}{\wh{w}}

\newcommand{\upconf}{\phi}
\newcommand{\ind}{\mathbbm{1}}

\newcommand{\LT}{L_T}
\newcommand{\LTbar}{\overline{L}_T}
\newcommand{\LTbarep}{\mathring{L}_T}
\newcommand{\Lubar}{\overline{L}_{u}}

\newcommand{\Lyr}{J}
\newcommand{\QQ}{{w}}
\newcommand{\Qt}{{\QQ_t}}
\newcommand{\Qstar}{{u}}
\newcommand{\Qpistar}{{\Qstar^{\star}}}
\newcommand{\Qhat}{\wh{\QQ}}
\newcommand{\Ut}{{\upconf_t}}
\newcommand{\intO}{\mathrm{int}(\Omega)}
\newcommand{\intK}{\mathrm{int}(K)}

\newcommand{\order}{\ensuremath{\mathcal{O}}}
\newcommand{\otil}{\ensuremath{\tilde{\mathcal{O}}}}

\newcommand{\rbr}[1]{\left(#1\right)}
\newcommand{\sbr}[1]{\left[#1\right]}
\newcommand{\cbr}[1]{\left\{#1\right\}}

\usepackage{prettyref}
\newcommand{\pref}[1]{\prettyref{#1}}
\newcommand{\pfref}[1]{Proof of \prettyref{#1}}
\newcommand{\savehyperref}[2]{\texorpdfstring{\hyperref[#1]{#2}}{#2}}
\newrefformat{eq}{\savehyperref{#1}{Eq. \textup{(\ref*{#1})}}}
\newrefformat{eqn}{\savehyperref{#1}{Equation~\ref*{#1}}}
\newrefformat{lem}{\savehyperref{#1}{Lemma~\ref*{#1}}}
\newrefformat{def}{\savehyperref{#1}{Definition~\ref*{#1}}}
\newrefformat{line}{\savehyperref{#1}{Line~\ref*{#1}}}
\newrefformat{thm}{\savehyperref{#1}{Theorem~\ref*{#1}}}
\newrefformat{corr}{\savehyperref{#1}{Corollary~\ref*{#1}}}
\newrefformat{cor}{\savehyperref{#1}{Corollary~\ref*{#1}}}
\newrefformat{sec}{\savehyperref{#1}{Section~\ref*{#1}}}
\newrefformat{app}{\savehyperref{#1}{Appendix~\ref*{#1}}}
\newrefformat{assum}{\savehyperref{#1}{Assumption~\ref*{#1}}}
\newrefformat{ex}{\savehyperref{#1}{Example~\ref*{#1}}}
\newrefformat{fig}{\savehyperref{#1}{Figure~\ref*{#1}}}
\newrefformat{alg}{\savehyperref{#1}{Algorithm~\ref*{#1}}}
\newrefformat{rem}{\savehyperref{#1}{Remark~\ref*{#1}}}
\newrefformat{conj}{\savehyperref{#1}{Conjecture~\ref*{#1}}}
\newrefformat{prop}{\savehyperref{#1}{Proposition~\ref*{#1}}}
\newrefformat{proto}{\savehyperref{#1}{Protocol~\ref*{#1}}}
\newrefformat{prob}{\savehyperref{#1}{Problem~\ref*{#1}}}
\newrefformat{claim}{\savehyperref{#1}{Claim~\ref*{#1}}}
\newrefformat{que}{\savehyperref{#1}{Question~\ref*{#1}}}
\newrefformat{op}{\savehyperref{#1}{Open Problem~\ref*{#1}}}
\newrefformat{fn}{\savehyperref{#1}{Footnote~\ref*{#1}}}

\title{Bias no more: high-probability data-dependent regret bounds for adversarial bandits and MDPs}

\author{%
  Chung-Wei Lee \\
  University of Southern California \\
  \texttt{leechung@usc.edu} \\
  \And
  Haipeng Luo \\
  University of Southern California \\
  \texttt{haipengl@usc.edu} \\
    \And
  Chen-Yu Wei\\
  University of Southern California \\
  \texttt{chenyu.wei@usc.edu} \\
    \And
  Mengxiao Zhang \\
  University of Southern California \\
  \texttt{mengxiao.zhang@usc.edu} \\
}

\begin{document}
\SetAlgoVlined
\DontPrintSemicolon
\maketitle

\begin{abstract}
We develop a new approach to obtaining high probability regret bounds for online learning with bandit feedback against an adaptive adversary.
While existing approaches all require carefully constructing optimistic and biased loss estimators,
our approach uses standard unbiased estimators and relies on a simple increasing learning rate schedule, together with the help of logarithmically homogeneous self-concordant barriers and a strengthened Freedman's inequality.

Besides its simplicity, our approach enjoys several advantages.
First, the obtained high-probability regret bounds are data-dependent and could be much smaller than the worst-case bounds, which resolves an open problem asked by~\citet{neu2015explore}.
Second, resolving another open problem of~\citet{bartlett2008high} and~\citet{abernethy2009beating}, our approach leads to the first general and efficient algorithm with a high-probability regret bound for adversarial linear bandits, while previous methods are either inefficient or only applicable to specific action sets.
Finally, our approach can also be applied to learning adversarial Markov Decision Processes and provides the first algorithm with a high-probability small-loss bound for this problem.
\end{abstract}


\section{Introduction}\label{sec:intro}
Online learning with partial information in an adversarial environment, such as the non-stochastic Multi-armed Bandit (MAB) problem~\citep{auer2002nonstochastic}, is by now a well-studied topic.
However, the majority of work in this area has been focusing on obtaining algorithms with sublinear {\it expected} regret bounds, 
and these algorithms can in fact be highly instable and suffer a huge variance.
For example, it is known that the classic \expthree algorithm~\citep{auer2002nonstochastic} for MAB suffers {\it linear regret with a constant probability} (over its internal randomness), despite having nearly optimal expected regret (see~\citep[Section 11.5, Note 1]{lattimore2018bandit}), making it a clearly undesirable choice in practice.

To address this issue, a few works develop algorithms with regret bounds that hold with {\it high probability}, including those for MAB~\citep{auer2002nonstochastic, audibert2009minimax, neu2015explore}, linear bandits~\citep{bartlett2008high, abernethy2009beating}, and even adversarial Markov Decision Processes (MDPs)~\citep{jin2019learning}.
Getting high-probability regret bounds is also the standard way of deriving guarantees against an {\it adaptive adversary} whose decisions can depend on learner's previous actions.
This is especially important for problems such as routing in wireless networks (modeled as linear bandits in~\citep{awerbuch2004adaptive}) where adversarial attacks can indeed adapt to algorithm's decisions on the fly.

As far as we know, all existing high-probability methods (listed above) are based on carefully constructing {\it biased} loss estimators that enjoy smaller variance compared to standard unbiased ones.
While this principle is widely applicable, the actual execution can be cumbersome;
for example, the scheme proposed in~\citep{abernethy2009beating} for linear bandits needs to satisfy seven conditions (see their Theorem~4), and other than two examples with specific action sets, no general algorithm satisfying these conditions was provided.

In this work, we develop a new and simple approach to obtaining high-probability regret bounds that works for a wide range of bandit problems with an adaptive adversary (including MAB, linear bandits, MDP, and more).
Somewhat surprisingly, in contrast to all previous methods, our approach uses standard {\it unbiased} loss estimators.
More specifically, our algorithms are based on Online Mirror Descent with a self-concordant barrier regularizer~\citep{abernethy2008competing}, a standard approach with expected regret guarantees.
The key difference is that we adopt an increasing learning rate schedule, inspired by several recent works using similar ideas for completely different purposes (e.g.,~\citep{agarwal2017corralling}).
At a high level, the effect of this schedule magically cancels the potentially large variance of the unbiased estimators. 

Apart from its simplicity, there are several important advantages of our approach.
First of all, our algorithms all enjoy {\it data-dependent regret bounds}, which could be much smaller than the majority of existing high-probability bounds in the form of $\otil(\sqrt{T})$ where $T$ is the number of rounds.
As a key example, we provide details for obtaining a particular kind of such bounds called ``small-loss'' bounds in the form $\otil(\sqrt{\Lstar})$, where $\Lstar \leq T$ is the loss of the benchmark in the regret definition.
For MAB and linear bandits, our approach also obtains bounds in terms of the variation of the environment in the vein of~\citep{hazan2011better, rakhlin2013online, wei2018more, bubeck2019improved}, resolving an open problem asked by~\citet{neu2015explore}.

Second, our approach provides the {\it first general and efficient} algorithm for adversarial linear bandits (also known as bandit linear optimization) with a high-probability regret guarantee.
As mentioned, \citet{abernethy2009beating} provide a general recipe for this task but in the end only show concrete examples for two specific action sets.
The problem of obtaining a general and efficient approach with regret $\otil(\sqrt{T})$ was left open since then.
The work of~\citep{bartlett2008high} proposes an inefficient but general approach, while the work of~\citep{gyorgy2007line, braun2016efficient} develop efficient algorithms for polytopes but with $\otil(T^{2/3})$ regret.
We not only resolve this long-standing open problem, but also provide improved data-dependent bounds.

Third, our approach is also applicable to learning episodic MDPs with unknown transition, adversarial losses, and bandit feedback.
The algorithm is largely based on a recent work~\citep{jin2019learning} on the same problem where a high-probability $\otil(\sqrt{T})$ regret bound is obtained.
We again develop the first algorithm with a high-probability small-loss bound $\otil(\sqrt{\Lstar})$ in this setting.
The problem in fact shares great similarity with the simple MAB problem.
However, none of the existing methods for obtaining small-loss bounds for MAB can be generalized to the MDP setting (at least not in a direct manner) as we argue in \pref{sec:MDP}.
Our approach, on the other hand, generalizes directly without much effort.


\paragraph{Techniques.}
Most new techniques of our work is in the algorithm for linear bandits (\pref{sec:LB}), which is based on the \scrible algorithm from the seminal work~\citep{abernethy2008competing, abernethy2012interior}.
The first difference is that we propose to lift the problem from $\fR^d$ to $\fR^{d+1}$ (where $d$ is the dimension of the problem) and use a {\it logarithmically homogeneous} self-concordant barrier of the conic hull of the action set (which always exists) as the regularizer for Online Mirror Descent.
The nice properties of such a regularizer lead to a smaller variance of the loss estimators.
Equivalently, this can be viewed as introducing a new sampling scheme for the original \scrible algorithm in the space of $\fR^d$.
The second difference is the aforementioned new learning rate schedule, where we increase the learning rate by a small factor whenever the Hessian of the regularizer at the current point is ``large'' in some sense.

In addition, we also provide a strengthened version of the Freedman's concentration inequality for martingales~\citep{freedman1975tail}, which is crucial to all of our analysis and  might be of independent interest.

\paragraph{Related work.}
In online learning, there are subtle but important differences and connections between the concept of pseudo-regret, expected regret, and the actual regret, in the context of either oblivious or adaptive adversary.
We refer the readers to~\citep{audibert2009minimax} for detailed related discussions. 

While getting expected small-loss regret is common~\citep{AllenbergAuGyOt06, neu2015first, FosterLiLySrTa16, allen2018make, abernethy2019online, LLZ20},
most existing high-probability bounds are of order $\otil(\sqrt{T})$.
Although not mentioned in the original paper, the idea of implicit exploration from~\citep{neu2015explore} can lead to high-probability small-loss bounds for MAB (see~\citep[Section 12.3, Note 4]{lattimore2018bandit}).
\citet{lykouris2018small} adopt this idea together with a clipping trick to derive small-loss bounds for more general bandit problems with graph feedback. 
We are not aware of other works with high-probability small-loss bounds in the bandit literature.
Note that in \citep[Section 6]{audibert2009minimax}, some high-probability ``small-reward'' bounds are derived, and they are very different in nature from small-loss bounds (specifically, the former is equivalent to $\otil(\sqrt{T-\Lstar})$ in our notation).
We are also not aware of high-probability version of other data-dependent regret bounds such as those from~\citep{hazan2011better, rakhlin2013online, wei2018more, bubeck2019improved}.

The idea of increasing learning rate was first used in the seminal work of~\citet{bubeck2017kernel} for convex bandits.
Inspired by this work, \citet{agarwal2017corralling} first combined this idea with the log-barrier regularizer for the problem of ``corralling bandits''.
Since then, this particular combination has proven fruitful for many other problems~\citep{wei2018more, luo2018efficient, LLZ20}.
We also use it for MAB and MDP, but our algorithm for linear bandits greatly generalizes this idea to any self-concordant barrier.

\paragraph{Structure and notation.}
In \pref{sec:MAB}, we start with a warm-up example on MAB, which is the cleanest illustration on the idea of using increasing learning rates to control the variance of unbiased estimators.
Then in \pref{sec:LB} and \pref{sec:MDP}, we greatly generalize the idea to linear bandits and MDPs respectively.
We focus on showing small-loss bounds as the main example, and only briefly discuss how to obtain other data-dependent regret bounds, since the ideas are very similar.

We introduce the notation for each setting in the corresponding section, but will use the following general notation throughout the paper: 
for a positive integer $n$, $[n]$ represents the set $\{1, \ldots, n\}$ and $\Delta_n$ represents the $(n-1)$-dimensional simplex;
$e_i$ stands for the $i$-th standard basis vector and $\one$ stands for the all-one vector (both in an appropriate dimension depending on the context);
for a convex function $\psi$, the associated Bregman divergence is $D_\psi(u, w) = \psi(u) - \psi(w) - \nabla\psi(w)^\top(u - w)$;
for a positive definite matrix $M \in \fR^{d\times d}$ and a vector $u\in \fR^d$, $\norm{u}_{M} \triangleq \sqrt{u^\top M u}$ is the quadratic norm of $u$ with respect to $M$;
$\lambdamax(M)$ denotes the largest eigenvalue of $M$;
$\E_t[\cdot]$ is a shorthand for the conditional expectation given the history before round $t$;
$\otil(\cdot)$ hides logarithmic terms.


\section{Multi-armed bandits: an illustrating example}
\label{sec:MAB}

We start with the most basic bandit problem, namely adversarial MAB~\citep{auer2002nonstochastic}, to demonstrate the core idea of using increasing learning rate to reduce the variance of standard algorithms.
The MAB problem proceeds in rounds between a learner and an adversary.
For each round $t = 1, \ldots, T$,
the learner selects one of the $d$ available actions $i_t \in [d]$,
while simultaneously the adversary decides a loss vector $\ell_t \in [0,1]^d$ with $\ell_{t,i}$ being the loss for arm $i$. An adaptive adversary can choose $\ell_t$ based on the learner's previous actions $i_1, \ldots, i_{t-1}$ in an arbitrary way, 
while an oblivious adversary cannot and essentially decides all $\ell_t$'s ahead of time (knowing the learner's algorithm).
At the end of round $t$, the learner observes the loss of the chosen arm $\ell_{t, i_t}$ and nothing else.
The standard measure of the learner's performance is the regret, defined as $\Reg = \sum_{t=1}^T \ell_{t, i_t} - \min_{i \in [d]} \sum_{t=1}^T \ell_{t,i}$, that is, the difference between the total loss of the learner and that of the best fixed arm in hindsight.

A standard framework to solve this problem is Online Mirror Descent (OMD),
which at time $t$ samples $i_t$ from a distribution $w_t$, updated in the following recursive form:
$w_{t+1} = \argmin_{w \in \Delta_d} \big\langle w, \ellhat_t \big\rangle + D_{\psi_t}(w, w_t)$,
where $\psi_t$ is the regularizer and $\ellhat_t$ is an estimator for $\ell_t$.
The standard estimator is the importance-weighted estimator: $\ellhat_{t,i} = \ell_{t, i}\ind\{i_{t}=i\}/w_{t,i}$, which is clearly unbiased.
Together with many possible choices of the regularizer (e.g., the entropy regularizer recovering \expthree~\citep{auer2002nonstochastic}), 
this ensures (nearly) optimal expected regret bound $\E[\Reg] = \otil(\sqrt{dT})$ against an oblivious adversary.

To obtain high-probability regret bounds (and also as a means to deal with adaptive adversary),
various more sophisticated loss estimators have been proposed.
Indeed, the key challenge in obtaining high-probability bounds lies in the potentially large variance of the unbiased estimators: $\E_t\big[\ellhat_{t,i}^2\big] = \ell_{t, i}^2/ w_{t,i}$ is huge if $w_{t,i}$ is small.
The idea of all existing approaches to addressing this issue is to introduce a slight bias to the estimator, making it an optimistic underestimator of $\ell_t$ with lower variance (see e.g.,~\citep{auer2002nonstochastic, audibert2009minimax, neu2015explore}).
Carefully balancing the bias and variance, these algorithms achieve $\Reg = \otil(\sqrt{dT\ln(d/\delta)})$ with probability at least $1 - \delta$ against an adaptive adversary.

\paragraph{Our algorithm.}
In contrast to all these existing approaches, we next show that, perhaps surprisingly, using the standard unbiased estimator can also lead to the same (in fact, an even better) high-probability regret bound.
We start by choosing a particular regularizer called log-barrier with time-varying and individual learning rate $\eta_{t,i}$: $\psi_t(w) = \sum_{i=1}^d \frac{1}{\eta_{t,i}} \ln\frac{1}{w_i}$, which is a self-concordant barrier for the positive orthant~\citep{nesterov1994interior} and has been used for MAB in several recent works~\citep{FosterLiLySrTa16, agarwal2017corralling, bubeck2018sparsity, wei2018more, bubeck2019improved}.
As mentioned in \pref{sec:intro}, the combination of log-barrier and a particular increasing learning rate schedule has been proven powerful for many different problems since the work of~\citep{agarwal2017corralling}, which we also apply here.
Specifically, the learning rates start with a fixed value $\eta_{1, i} = \eta$ for all arm $i\in [d]$, and every time the probability of selecting an arm $i$ is too small, in the sense that $1/w_{t+1,i} > \rho_{t,i}$ for some threshold $\rho_{t,i}$ (starting with $2d$), we set the new threshold to be $2/w_{t+1,i}$ and increase the corresponding learning rate $\eta_{t,i}$ by a small factor $\kappa$.

The complete pseudocode is shown in \pref{alg:mab}.
The only slight difference compared to the algorithm of~\citep{agarwal2017corralling} is that instead of enforcing a $1/T$ amount of uniform exploration explicitly (which makes sure that each learning rate is increased by a most $\order(\ln T)$ times), we directly perform OMD over a truncated simplex $\Omega = \{w \in \Delta_d: w_{i} \geq 1/T, \forall i \in[d]\}$, making the analysis cleaner.


\begin{algorithm}[t]
\caption{OMD with log-barrier and increasing learning rates for Multi-armed Bandits}\label{alg:mab}
\textbf{Input:} initial learning rate $\eta$.

\textbf{Define:} increase factor $\kappa=e^{\frac{1}{\ln T}}$, truncated simplex $\Omega=\left\{w \in \Delta_d: w_{i} \geq \frac{1}{T}, \forall i \in[d]\right\}$.

\textbf{Initialize:} for all $i\in [d]$, $w_{1, i}=1/d, \rho_{1, i}=2d, \eta_{1,i}=\eta$.

\For{$t=1,2,\dots, T$}{
    Sample $i_{t} \sim {w}_{t}$, observe $\ell_{t, i_{t}}$, and construct estimator $\ellhat_{t, i}=\frac{\ell_{t, i} \ind\left\{i_{t}=i\right\}}{{w}_{t, i}}$ for all $i\in[d]$.
    
    Compute $w_{t+1}=\argmin_{w \in \Omega}\; \big\langle w, \ellhat_{t}\big\rangle+D_{\psi_{t}}\left(w, w_{t}\right)$ where $\psi_{t}(w)=\sum_{i=1}^{d} \frac{1}{\eta_{t, i}} \ln \frac{1}{w_{i}}$.
    
    \For{$i \in [d]$}{
        \lIf {$\frac{1}{w_{t+1,i}}>\rho_{{t},i}$}{set $\rho_{t+1,i}=\frac{2}{w_{t+1,i}},\eta_{t+1,i}=\eta_{t,i}\kappa$;} \lElse {set $\rho_{t+1,i}=\rho_{{t},i},\eta_{t+1,i}=\eta_{t,i}$.}
    }
}
\end{algorithm}

As explained in~\citep{agarwal2017corralling}, increasing the learning rate in this way allows the algorithm to quickly realize that some arms start to catch up even though they were underperforming in earlier rounds, which is also the hardest case in our context of obtaining high-probability bounds because these arms have low-quality estimators at some point.
At a technical level, this effect is neatly presented through a {\it negative} term in the regret bound, which we summarize below.

\begin{lemma}\label{lem:MAB_lemma}
\pref{alg:mab} ensures 
$
\sum_{t=1}^T \ell_{t, i_t} - \sum_{t=1}^T \big\langle u, \ellhat_{t} \big\rangle \leq \order\left(\frac{d\ln T}{\eta} + \eta \sum_{t=1}^T \ell_{t, i_t} \right) - \frac{\inner{\rho_T, u}}{10\eta\ln T}
$ for any $u \in \Omega$.
\end{lemma}

The important part is the last negative term involving the last threshold $\rho_T$ whose magnitude is large whenever an arm has a small sampling probability at some point over the $T$ rounds.
This bound has been proven in previous works such as~\citep{agarwal2017corralling} (see a proof in~\pref{app:lemma2.1}), and next we use it to show that the algorithm in fact enjoys a high-probability regret bound, which is not discovered before.

Indeed, comparing \pref{lem:MAB_lemma} with the definition of regret, one sees that as long as we can relate the estimated loss of the benchmark $\sum_t \big\langle u, \ellhat_t \big\rangle$ with its true loss $\sum_t \big\langle u, \ell_t \big\rangle$, then we immediately obtain a regret bound by setting $u=(1-\frac{d}{T})e_{i^\star}+\frac{1}{T}\one \in \Omega$ where $i^\star = \argmin_i \sum_t \ell_{t,i}$ is the best arm.
A natural approach is to apply standard concentration inequality, in particular Freedman's inequality~\citep{freedman1975tail}, to the martingale difference sequence $\big\langle u, \ellhat_t - \ell_t \big\rangle$.
The deviation from Freedman's inequality is in terms of the variance of $\big\langle u, \ellhat_t\big\rangle$, which in turn depends on $\sum_i u_i/w_{t,i}$. 
As explained earlier, the negative term is exactly related to this and can thus cancel the potentially large variance!

One caveat, however, is that the deviation from Freedman's inequality also depends on a {\it fixed} upper bound of the random variable $\big\langle u, \ellhat_t\big\rangle \leq \sum_i u_i/w_{t,i}$, which could be as large as $T$ (since $w_{t,i}\geq 1/T$) and ruin the bound.
If the dependence on such a fixed upper bound could be replaced with the (random) upper bound $\sum_i u_i/w_{t,i}$, then we could again use the negative term to cancel this dependence.
Fortunately, since $\sum_i u_i/w_{t,i}$ is measurable with respect to the $\sigma$-algebra generated by everything {\it before} round $t$, 
we are indeed able to do so.
Specifically, we develop the following strengthened version of Freedman's inequality, which might be of independent interest.

\begin{theorem}\label{thm:Freedman}
Let $X_1, \ldots, X_T$ be a martingale difference sequence with respect to a filtration $\calF_1 \subseteq \cdots \subseteq \calF_T$ such that $\E[X_t |\calF_t] = 0$.
Suppose $B_t \in [1,b]$ for a fixed constant $b$ is $\calF_t$-measurable and such that $X_t \leq B_t$ holds almost surely.
Then with probability at least $1 -\delta$ we have 
$
    \sum_{t=1}^T  X_t \leq  C\big(\sqrt{8V\ln\left(C/\delta\right)} + 2\Bstar  \ln\left(C/\delta\right)\big),
$
where $V = \max\big\{1, \sum_{t=1}^T \E[X_t^2 | \calF_t]\big\}$, $\Bstar =\max_{t\in[T]} B_t$, and 
$C = \ceil{\log(b)}\ceil{\log(b^2T)}$.
\end{theorem}

This strengthened Freedman's inequality essentially recovers the standard one when $B_t$ is a fixed quantity.
In our application, $B_t$ is exactly $\inner{\rho_t, u}$ which is $\calF_t$-measurable.
With the help of this concentration result, we are now ready to show the high-probability guarantee of \pref{alg:mab}.

\begin{theorem}\label{thm:MAB}
\pref{alg:mab} with a suitable choice of $\eta$ ensures that with probability at least $1-\delta$, 
$
\Reg=\widetilde{\mathcal{O}}\big(\sqrt{d\Lstar\ln(\nicefrac{d}{\delta})}+d\ln(\nicefrac{d}{\delta})\big),
$
where $\Lstar = \min_i \sum_{t=1}^T \ell_{t,i}$ is the loss of the best arm. 
\end{theorem}

The proof is a direct combination of \pref{lem:MAB_lemma} and \pref{thm:Freedman} and can be found in \pref{app:theorem2.3}.
Our high-probability guarantee is of the same order $\otil(\sqrt{dT\ln(d/\delta)})$ as in previous works~\citep{auer2002nonstochastic, audibert2009minimax} since $\Lstar =\order(T)$.
However, as long as $\Lstar =o(T)$ (that is, the best arm is of high quality), our bound becomes much better.
This kind of high-probability small-loss bounds appears before  (e.g.,~\citep{lykouris2018small}).
Nevertheless, in \pref{sec:MDP} we argue that only our approach can directly generalize to learning MDPs.

Finally, we remark that 
the same algorithm can also obtain other data-dependent regret bounds by changing the estimator to $\ellhat_{t,i} = (\ell_{t, i}-m_{t,i})\ind\{i_{t}=i\}/w_{t,i} + m_{t,i}$ for some optimistic prediction $m_t$.
We refer the reader to~\citep{wei2018more} for details on how to set $m_t$ in terms of observed data and what kind of bounds this leads to,
but the idea of getting the high-probability version is completely the same as what we have illustrated here.
This resolves an open problem mentioned in~\citep[Section~5]{neu2015explore}.


\section{Generalization to adversarial linear bandits}
\label{sec:LB}

Next, we significantly generalize our approach to adversarial linear bandits, which is the main algorithmic contribution of this work.
Linear bandits generalize MAB from the simplex decision set $\Delta_d$ to an arbitrary convex body $\Omega \subseteq \fR^d$.
For each round $t = 1, \ldots, T$, the learner selects an action $\wtilde_t \in \Omega$ while simultaneously the adversary decides a loss vector $\ell_t \in \fR^d$, assumed to be normalized such that $\max_{w\in\Omega} |\inner{w, \ell_t}| \leq 1$.
Again, an adaptive adversary can choose $\ell_t$ based on the learner's previous actions, while an oblivious adversary cannot.
At the end of round $t$, the learner suffers and only observes loss $\inner{\wtilde_t, \ell_t}$.
The regret of the learner is defined as $\Reg = \max_{u\in\Omega}\sum_{t=1}^T \inner{\wtilde_t - u, \ell_t}$, which is the difference between the total loss of the learner and that of the best fixed action within $\Omega$.
Linear bandits subsume many other well-studied problems such as online shortest path for network routing, online matching, and other combinatorial bandit problems (see e.g.,~\citep{audibert2011minimax, cesa2012combinatorial}). 

The seminal work of~\citet{abernethy2008competing} develops the first general and efficient linear bandit algorithm (called \scrible in its journal version~\citep{abernethy2012interior}) with expected regret $\otil(d\sqrt{\nu T})$ (against an oblivious adversary), which uses a $\nu$-self-concordant barrier as the regularizer for OMD. 
It is known that any convex body in $\fR^d$ admits a $\nu$-self-concordant barrier with $\nu=\order(d)$~\citep{nesterov1994interior}.
The minimax regret of this problem is known to be of order $\otil(d\sqrt{T})$~\citep{dani2008price, bubeck2012towards}, but efficiently achieving this bound (in expectation) requires a log-concave sampler and a volumetric spanner of $\Omega$~\citep{hazan2016volumetric}.

High-probability bounds for linear bandits are very scarce, especially for a general decision set $\Omega$.
In~\citep{bartlett2008high}, an algorithm with high-probability regret $\otil(\sqrt{d^3T}\ln(1/\delta))$ was developed, but it cannot be implement efficiently.
In~\citep{abernethy2009beating}, a general recipe was provided, but seven conditions need to be satisfied to arrive at a high-probability guarantee, and only two concrete examples were shown (when $\Omega$ is the simplex or the Euclidean ball).
We propose a new algorithm based on \scrible, which is the first general and efficient linear bandit algorithm with a high-probability regret guarantee, resolving the problem left open since the work of~\citep{bartlett2008high, abernethy2009beating}.

\paragraph{Issues of \scrible.}
To introduce our algorithm, we first review \scrible.
As mentioned, it is also based on OMD and maintains a sequence $w_1, \ldots, w_T\in\Omega$ updated as $w_{t+1}= \argmin_{w\in\Omega}\big\langle w, \ellhat_t \big\rangle + \tfrac{1}{\eta}D_\psi(w, w_t)$ where $\ellhat_t$ is an estimator for $\ell_t$, $\eta$ is some learning rate, and importantly, $\psi$ is a $\nu$-self-concordant barrier for $\Omega$ which, again, always exists.
Due to space limit, we defer the definition and properties of self-concordant barriers to \pref{app:LB-pre}.
To incorporate exploration, the actual point played by the algorithm at time $t$ is $\wtilde_t = w_t + H_t^{-1/2} s_t$ where $H_t = \nabla^2\psi(w_t)$ and $s_t$ is uniformly randomly sampled from the $d$-dimensional unit sphere $\fS^d$.\footnote{%
In fact, $s_t$ can be sampled from any orthonormal basis of $\fR^d$ together with their negation.
For example, in the original \scrible, the eigenbasis of $H_t$ is used as this orthonormal basis.
The version of sampling from a unit sphere first appears in~\citep{saha2011improved}, which works more generally for convex bandits.
}
The point $\wtilde_t$ is on the boundary of the {\it Dikin ellipsoid} centered at $w_t$ (defined as $\{w: \norm{w-w_t}_{H_t}\leq1\}$) and is known to be always within $\Omega$.
Finally, the estimator $\ellhat_t$ is constructed as $d\inner{\wtilde_t, \ell_t}H_t^{1/2}s_t$, which can be computed using only the feedback $\inner{\wtilde_t, \ell_t}$ and is unbiased as one can verify.

The analysis of~\citep{abernethy2008competing} shows the following bound related to the loss estimators: $\sum_{t=1}^T \big\langle w_t - u, \ellhat_t\big\rangle \leq \otil(\frac{\nu}{\eta}+\eta d^2T)$ for any $u \in \Omega$ (that is not too close to the boundary).
Since $\E_t[\big\langle w_t - u, \ellhat_t\big\rangle] = \E_t[\big\langle \wtilde_t - u, \ell_t \big\rangle]$,
this immediately yields an expected regret bound (for an oblivious adversary).
However, to obtain a high-probability bound, one needs to consider the deviation of $\sum_{t=1}^T \big\langle w_t - u, \ellhat_t\big\rangle$ from $\sum_{t=1}^T \big\langle \wtilde_t - u, \ell_t \big\rangle$.
Applying our strengthened Freedman's inequality (\pref{thm:Freedman}) with $X_t = \big\langle \wtilde_t - u, \ell_t \big\rangle - \big\langle w_t - u, \ellhat_t\big\rangle$, with some direct calculations one can see that both the variance term $V$ and the range term $B^\star$ from the theorem are related to $\max_{t} \norm{w_t}_{H_t}$ and $\max_{t} \norm{u}_{H_t}$, both of which can be prohibitively large.
We next discuss how to control each of these two terms, leading to the two new ideas of our algorithm (see \pref{alg:linear bandit d+1}).

\setcounter{AlgoLine}{0}
\begin{algorithm}[t]
	\caption{\scrible with lifting and increasing learning rates}\label{alg:linear bandit d+1}
	\textbf{Input:} decision set $\Omega\subseteq \mathbb{R}^d$, a $\nu$-self-concordant barrier $\psi$ for $\Omega$, initial learning rate $\eta$.
    
	\textbf{Define:} increase factor $\kappa=e^{\frac{1}{100d\ln (\nu T)}}$,  normal barrier $\Psi(\dpw) = \Psi(w,b)=400\left(\psi\left(\frac{w}{b}\right)-2\nu \ln b\right)$.
	
	\textbf{Initialize:} $w_1 = \argmin_{w\in \Omega}\psi(w)$, $\dpw_1 = (w_1,1)$, $\dpH_1=\nabla^2\Psi(\dplus{w_1})$, $\eta_1=\eta$, $\calS=\{1\}$.

 	\textbf{Define:} shrunk lifted decision set $\dpOmega' = \{\dpw = (w,1): w\in \Omega, \pi_{w_1}(w)\le 1-\frac{1}{T}\}$. 
	
	\nl \For{$t=1,2,\dots, T$}{
		
		\nl Uniformly at random sample $\dps_t$ from $\big(\dpH_t^{-\frac{1}{2}}\dpe_{d+1}\big)^\perp\cap \mathbb{S}^{d+1}$. \label{line:sample}
		
		\nl Compute $\dpwtilde_t=\dpw_t+\dpH_t^{-\frac{1}{2}}\dps_t \triangleq (\wtilde_t, 1)$.  \label{line:play}
		
		\nl Play $\wtilde_t$, observe loss $\langle \wtilde_t,\ell_t\rangle$, and construct loss estimator $\dpellhat_t=d \langle \wtilde_t,\ell_t\rangle\dpH_t^{\frac{1}{2}}\dps_t$. \label{line:estimator}
				
		\nl Compute $\dpw_{t+1}=\argmin_{\dpw \in \dpOmega'} \big\langle \dpw, \dpellhat_t \big\rangle+D_{\Psi_t}(\dpw, \dpw_t)$, where $\Psi_t=\frac{1}{\eta_t}\Psi$. \label{line:OMD}
		
		\nl Compute $\dpH_{t+1}=\nabla^2\Psi(\dpw_{t+1})$.
		
		\nl \lIf {$\lambdamax(\dpH_{t+1}-\sum_{\tau\in \calS}\dpH_{\tau}) > 0$}{
			$\calS\leftarrow \calS\cup\{t+1\}$ and set $\eta_{t+1}=\eta_{t}\kappa$;
		} \label{line:check}
		\nl \lElse{
		set $\eta_{t+1}=\eta_{t}$.} 
		 
	}
\end{algorithm}

\begin{figure}
\centering
\includegraphics[trim=30 0 0 0,clip,width=12cm]{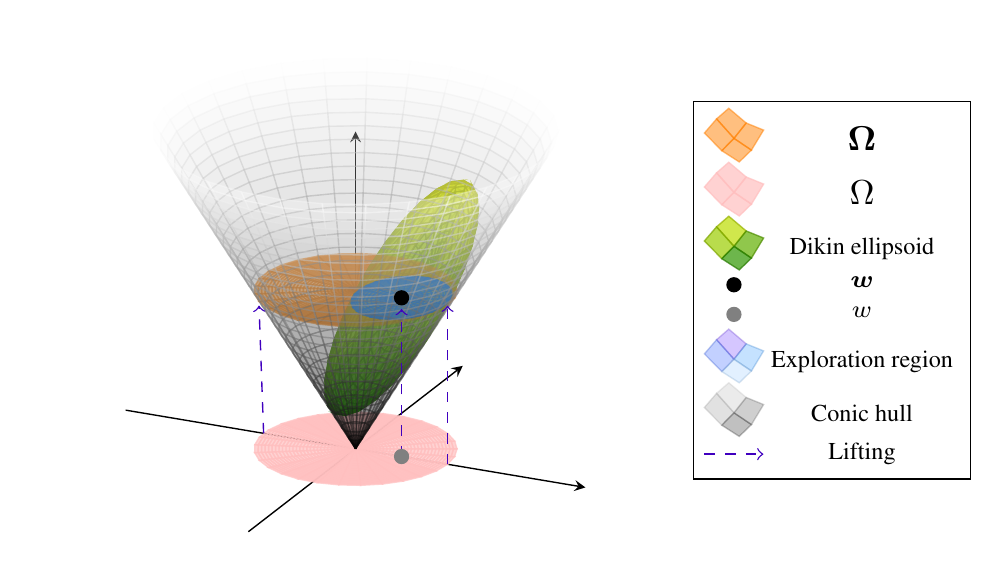}\caption{
An illustration of the concept of lifting, the conic hull, and the Dikin ellipsoid.
In this example $d$ is $2$, and the pink disk at the bottom is the original decision set $\Omega$.
The gray dot $w$ is a point in $\Omega$.
In \pref{alg:linear bandit d+1}, we lift the problem from $\fR^2$ to $\fR^3$, and obtain the lifted, orange, decision set $\dpOmega$.
For example, $w$ is lifted to the black dot $\dpw=(w,1)$.
Then we construct the conic hull of the lifted decision set, that is, the gray cone, and construct a normal barrier for this conic hull.
By \pref{lem:Dikin}, the Dikin ellipsoid centered at $\dpw$ of this normal barrier (the green ellipsoid), is alway within the cone.
In \pref{alg:linear bandit d+1}, if $\dpw$ is the OMD iterate, we explore and play an action within the intersection of $\dpOmega$ and the Dikin ellipsoid centered at $\dpw$, that is, the (boundary of) the blue ellipse.
}
\label{fig:fig1}
\end{figure}

\paragraph{Controlling $\norm{w_t}_{H_t}$.}
Readers who are familiar with self-concordant functions would quickly realize that $\norm{w_t}_{H_t} = \sqrt{w_t^\top\nabla^2\psi(w_t)w_t}$ is simply $\sqrt{\nu}$ provided that $\psi$ is also {\it logarithmically homogeneous}.
A logarithmically homogeneous self-concordant barrier is also called a {\it normal barrier} (see \pref{app:LB-pre} for formal definitions and related properties).
However, normal barriers are only defined for cones instead of convex bodies.

Inspired by this fact, we propose to {\it lift the problem to $\fR^{d+1}$}.
To make the notation clear, we use bold letters for vectors in $\fR^{d+1}$ and matrices in $\fR^{(d+1)\times(d+1)}$.
The lifting is done by operating over a lifted decision set $\dplus{\Omega}=\{\dplus{w} = (w,1) \in \fR^{d+1}: w\in \Omega\}$, that is, we append a dummy coordinate with value $1$ to all actions.
The conic hull of this set is $\calK = \{(w, b): w\in \fR^d, b \geq 0, \frac{1}{b}w \in \Omega\}$.
We then perform OMD over the lifted decision set but with a normal barrier defined over the cone $\calK$
as the regularizer to produce the sequence $\dplus{w}_1, \ldots, \dplus{w}_T$ (\pref{line:OMD}).
In particular, using the original regularizer $\psi$  we construct the normal barrier as: $\Psi(w,b)=400\left(\psi\left(\frac{w}{b}\right)-2\nu \ln b\right)$.\footnote{%
Our algorithm works with any normal barrier, not just this particular one.
We use this particular form to showcase that we only require a self-concordant barrier of the original set $\Omega$, exactly the same as \scrible. \label{fn:normal_barrier}
}
Indeed, Proposition~5.1.4 of~\citep{nesterov1994interior} asserts that this is a normal barrier for $\calK$ with self-concordant parameter $\order(\nu)$.

So far nothing really changes since $\Psi(w, 1) = 400\psi(w)$.
However, the key difference is in the way we sample the point $\dplus{\wtilde}_t$.
If we still follow \scrible to sample from the Dikin ellipsoid centered at $\dplus{w}_t$, it is possible that the sampled point leaves $\dplus{\Omega}$.
To avoid this, it is natural to sample only the intersection of the Dikin ellipsoid and $\dplus{\Omega}$ (again an ellipsoid).
Algebraically, this means setting $\dplus{\wtilde}_t = \dplus{w}_t + \dplus{H}_t^{-1/2}\dplus{s}_t$ where $\dplus{H}_t = \nabla^2\Psi(\dplus{w}_t)$ and $\dplus{s}_t$ is sampled uniformly at random from 
$(\dplus{H}_t^{-1/2}\dpe_{d+1})^\perp\cap \mathbb{S}^{d+1}$ ($v^\perp$ is the space orthogonal to $v$). 
Indeed, since $\dplus{s}_t$ is orthogonal to $\dplus{H}_t^{-1/2}\dpe_{d+1}$, the last coordinate of $\dplus{H}_t^{-1/2}\dplus{s}_t$ is zero, making $\dplus{\wtilde}_t=(\wtilde_t, 1)$ stay in $\dplus{\Omega}$.
See \pref{line:sample} and \pref{line:play}. 
To sample $\dplus{s}_t$ efficiently, one can either sample a vector uniformly randomly from $\mathbb{S}^{d+1}$, project it onto the subspace perpendicular to $\dplus{H}_{t}^{-1/2}\dpe_{d+1}$, and then normalize; or sample a vector $s_t$ uniformly randomly from $\mathbb{S}^{d}$, then normalize $\bm{H}_t^{\frac{1}{2}}(s_t^\top,0)^\top$ to obtain $\bm{s}_t$.

Finally, after playing $\wtilde_t$  and observing $\inner{\wtilde_t, \ell_t}$, we construct the loss estimator the same way as \scrible: $\dplus{\ellhat}_t = d\inner{\wtilde_t, \ell_t}\dplus{H}_t^{1/2}\dplus{s}_t$ (\pref{line:estimator}).
\pref{lem:LB-unbiased} shows that the first $d$ coordinates of $\dplus{\ellhat}_t$ is indeed an unbiased estimator of $\ell_t$.
This makes the entire analysis of \scrible hold in $\fR^{d+1}$,
but now the key term $\norm{\dplus{w}_t}_{\dplus{H}_t}$ we want to control is {\it exactly} $20\sqrt{2\nu}$ (see \pref{lem:prop5-1-4} and \pref{lem:NBprpty})!

We provide an illustration of the lifting idea in \pref{fig:fig1}.
One might ask whether this lifting is necessary;
indeed, one can also spell out the algorithm in $\fR^d$ (see \pref{app:alg-d}).
Importantly, compared to \scrible, the key difference is still that the sampling scheme has changed: 
the sampled point is not necessarily on the Dikin ellipsoid with respect to $\psi$.
In other words, another view of our algorithm is that it is \scrible with a new sampling scheme.
We emphasize that, however, it is important (or at least much cleaner) to perform the analysis in $\fR^{d+1}$.
In fact, even in \pref{alg:mab} for MAB, similar lifting implicitly happens already since $\Delta_d$ is a convex body in dimension $d-1$ instead of $d$,
and log-barrier is indeed a canonical normal barrier for the positive orthant.

\paragraph{Controlling $\norm{u}_{H_t}$.}
Next, we discuss how to control the term $\norm{u}_{H_t}$, or rather $\norm{\dpu}_{\dpH_t}$ after the lifting.
This term is the analogue of $\sum_i \frac{u_i}{w_{t,i}}$ for the case of MAB, and our goal is again to cancel it with the negative term introduced by increasing the learning rate.
Indeed, a closer look at the OMD analysis reveals that increasing the learning rate at the end of time $t$ brings a negative term involving $-D_\Psi(\dpu, \dpw_{t+1})$ in the regret bound.
In \pref{lem:LB-Bregman}, we show that this negative term is upper bounded by $-\norm{\dpu}_{\dpH_{t+1}}+800\nu\ln(800\nu T+1)$, making the canceling effect possible.

It just remains to figure out when to increase the learning rate and how to make sure we only increase it logarithmic (in $T$) times as in the case for MAB.
Borrowing ideas from~\pref{alg:mab}, intuitively one should increase the learning rate only when $\dpH_t$ is ``large'' enough,
but the challenge is how to measure this quantitatively.
Only looking at the eigenvalues of $\dpH_t$, a natural idea, does not work as it does not account for the fact that the directions of eigenvectors are changing over time.

Instead, we propose the following condition: at the end of time $t$, increase the learning rate by a factor of $\kappa$ if $\lambdamax(\dpH_{t+1}-\sum_{\tau\in \calS}\dpH_{\tau}) > 0$, with $\calS$ containing all the previous time steps prior to time $t$ where the learning rate was increased (\pref{line:check}).
First, note that this condition makes sure that we always have enough negative terms to cancel $\max_t\norm{\dpu}_{\dpH_t}$.
Indeed, suppose $t$ is the time with the largest $\norm{\dpu}_{\dpH_{t+1}}$.
If we have increased the learning rate at time $t$, then the introduced negative term exactly matches $\norm{\dpu}_{\dpH_{t+1}}$ as mentioned above;
otherwise, the condition did not hold and by definition we have $\norm{\dpu}_{\dpH_{t+1}} \leq \sqrt{\sum_{\tau\in\calS}\norm{\dpu}_{\dpH_s}^2} \leq \sum_{\tau\in\calS}\norm{\dpu}_{\dpH_\tau}$, meaning that the negative terms introduced in previous steps are already enough to cancel $\norm{\dpu}_{\dpH_{t+1}}$.

Second, in \pref{lem:LB-increase} we show that this schedule indeed makes sure that the learning rate is increased by only $\otil(d)$ times.
The key idea is to prove that $\det(\sum_{\tau\in \calS}\dpH_{\tau})$ is at least doubled each time we add one more time step to $\calS$.
Thus, if the eigenvalues of $\dpH_t$ are bounded, $|\calS|$ cannot be too large.
Ensuring the last fact requires a small tweak to the OMD update (\pref{line:OMD}), where we constrain the optimization over a slightly shrunk version of $\dpOmega$ defined as $\dpOmega' = \{\dpw\in\dpOmega: \pi_{\dpw_1}(\dpw)\le 1-\frac{1}{T}\}$. 
Here, $\pi$ is the Minkowsky function and we defer its formal definition to~\pref{app:LB-pre}, but intuitively $\dpOmega'$ is simply obtained by shrinking the lifted decision set by a small amount of $1/T$ with respect to the center $\dpw_1$ (which is the analogue of the truncated simplex for MAB).
This makes sure that $\dpw_t$ is never too close to the boundary, and in turn makes sure that the eigenvalues of $\dpH_t$ are bounded.

This concludes the two main new ideas of our algorithm; see \pref{alg:linear bandit d+1} for the complete pseudocode.
Clearly, our algorithm can be implemented as efficiently as the original \scrible.
The regret guarantee is summarized below.

\begin{theorem}\label{thm:LB}
	\pref{alg:linear bandit d+1} with a suitable choice of $\eta$ ensures that with probability at least $1-\delta$:
\[
\Reg = \begin{cases}
\otil\left(d^2\nu\sqrt{ T\ln\frac{1}{\delta}}+d^2\nu\ln\frac{1}{\delta}\right), &\text{against an oblivious adversary;} \\
\otil\left(d^2\nu\sqrt{dT\ln\frac{1}{\delta}}+d^3\nu\ln\frac{1}{\delta}\right), &\text{against an adaptive adversary.}
\end{cases}
\]
Moreover, if $\inner{w, \ell_t} \geq 0$ for all $w\in\Omega$ and all $t$, then $T$ in the bounds above can be replaced by $L^\star = \min_{u\in\Omega}\sum_{t=1}^T \inner{u, \ell_t}$, that is, the total loss of the best action.
\end{theorem}

Our results are the first general high-probability regret guarantees achieved by an efficient algorithm (for either oblivious or adaptive adversary).
We not only achieve $\sqrt{T}$-type bounds, but also improve it to $\sqrt{\Lstar}$-type small-loss bounds, which does not exist before.
Note that the latter holds only when losses are nonnegative, which is a standard setup for small-loss bounds and is true, for instance, for all combinatorial bandit problems where $\Omega \subseteq [0,1]^d$ lives in the positive orthant.
Similarly to MAB, we can also obtain other data-dependent regret bounds by only changing the estimator to $d \langle \wtilde_t,\ell_t-m_t\rangle H_t^{1/2}s_t + m_t$ for some predictor $m_t$ (see~\citep{rakhlin2013online, bubeck2019improved}).\footnote{%
One caveat is that this requires measuring the learner's loss in terms of $\inner{w_t, \ell_t}$, as opposed to $\big\langle \wtilde_t, \ell_t \big\rangle$, since the deviation between these two is not related to $m_t$.
}

Ignoring lower order terms, our bound for oblivious adversaries is $d\sqrt{\nu}$ times worse than the expected regret of \scrible. 
For adaptive adversary, we pay extra dependence on $d$, which is standard since an extra union bound over $u$ is needed and is discussed in~\citep{abernethy2009beating} as well.
The minimax regret for adaptive adversary is still unknown.
Reducing the dependence on $d$ for both cases is a key future direction.


\section{Generalization to adversarial MDPs}
\label{sec:MDP}

Finally, we briefly discuss how to generalize \pref{alg:mab} for MAB to learning adversarial Markov Decision Processes (MDPs), leading to the first algorithm with a high-probability small-loss regret guarantee for this problem.
We consider an episodic MDP setting with finite horizon, unknown transition kernel, bandit feedback, and adversarial losses, the exact same setting as the recent work~\citep{jin2019learning} (which is the state-of-the-art for adversarial tabular MDPs; see~\citep{jin2019learning} for related work).

Specifically, the problem is parameterized by a state space $X$, an action space $A$, and an unknown transition kernel $P: X\times A \times X \rightarrow [0,1]$ with $P(x'|x,a)$ being the probability of reaching state $x'$ after taking action $a$ at state $x$.
Without loss of generality (see discussions in~\citep{jin2019learning}), the state space is assumed to be partitioned into $\Lyr+1$ layers $X_0, \ldots, X_{\Lyr}$ where $X_0 = \{x_0\}$ and $X_{\Lyr} = \{x_{\Lyr}\}$ contain only the start and end state respectively, and transitions are only possible between consecutive layers.

The learning proceeds in $T$ rounds/episodes. 
In each episode $t$,
the learner starts from state $x_0$ and decides a stochastic policy $\pi_t: X\times A \rightarrow [0,1]$, where $\pi_t(a|x)$ is the probability of selecting action $a$ at state $x$.
Simultaneously, the adversary decides a loss function $\ell_t: X\times A \rightarrow[0,1]$, with $\ell_t(x,a)$ being the loss of selecting action $a$ at state $x$.
Once again, an adaptive adversary chooses $\ell_t$ based on all learner's actions in previous episodes, while an oblivious adversary chooses $\ell_t$ only knowing the learner's algorithm.
Afterwards, the learner executes the policy in the MDP for $\Lyr$ steps and generates/observes a state-action-loss sequence $(x_0, a_0, \ell_t(x_0, a_0)), \ldots, (x_{\Lyr-1}, a_{\Lyr-1}, \ell_t(x_{\Lyr-1}, a_{\Lyr-1}))$ before reaching the final state $x_{\Lyr}$.
With a slight abuse of notation, we use $\ell_t(\pi) = \E\sbr{\sum_{k=1}^{\Lyr-1}\ell_t(x_k, a_k) \;\vert\; P, \pi}$ to denote the expected loss of executing policy $\pi$ in episode $t$.
The regret of the learner is then defined as $\Reg = \sum_{t=1}^T \ell_t(\pi_t) - \min_{\pi}\sum_{t=1}^T \ell_t(\pi)$, where the min is over all possible policies.

Based on several prior works~\citep{zimin2013, rosenberg2019online}, \citet{jin2019learning} showed the deep connection between this problem and adversarial MAB.
In fact, with the help of the ``occupancy measure'' concept, this problem can be reformulated in a way that becomes very much akin to adversarial MAB and can be essentially solved using OMD with some importance-weighted estimators. 
We refer the reader to~\citep{jin2019learning} and \pref{app:MDP-pre} for details.
The algorithm of~\citep{jin2019learning} achieves $\Reg = \otil(\Lyr|X|\sqrt{|A|T})$ with high probability.

Since the problem has great similarity with MAB, the natural idea to improve the bound to a small-loss bound is to borrow techniques from MAB.
Prior to our work, obtaining high-probability small-loss bounds for MAB can only be achieved by either the implicit exploration idea from~\citep{neu2015explore} or the clipping idea from~\citep{AllenbergAuGyOt06, lykouris2018small}.
Unfortunately, in \pref{app:MDP_discussions}, we argue that neither of them works for MDPs, at least not in a direct way we can see, from perspectives of both the algorithm and the analysis.

On the other hand, our approach from \pref{alg:mab} immediately generalizes to MDPs without much effort.
Compared to the algorithm of~\citep{jin2019learning}, the only essential differences are to replace their regularizer with log-barrier and to apply a similar increasing learning rate schedule.
Due to space limit, we defer the algorithm to \pref{app:MDP-alg} and show the main theorem below.

\begin{theorem}\label{thm:MDP}
	\pref{alg:MDP} with a suitable choice of $\eta$ ensures that with probability at least $1-\delta$, $\Reg = 
\otil\left(|X|\sqrt{\Lyr|A| L^{\star}\ln \frac{1}{\delta}}+|X|^5 |A|^2\ln^2 \frac{1}{\delta}\right)$, where $\Lstar = \min_{\pi}\sum_{t=1}^T \ell_t(\pi) \leq \Lyr T$ is the total loss of the best policy.
\end{theorem}

We remark that our bound holds for  both oblivious and adaptive adversaries, and is the first high-probability small-loss bounds for adversarial MDPs.\footnote{%
Obtaining other data-dependent regret bounds as in MAB and linear bandits is challenging in this case,
since there are several terms in the regret bound that are naturally only related to $\Lstar$.
}
This matches the bound of~\citep{jin2019learning} in the worst case (including the lower-order term $\otil(|X|^5 |A|^2)$ hidden in their proof), but could be much smaller as long as a good policy exists with $\Lstar = o(T)$. It is still open whether this bound is optimal or not. 


\section{Conclusions}
\label{sec:conclusions}

In this work, based on the idea of increasing learning rates we develop a new technique  for obtaining high-probability regret bounds against an adaptive adversary under bandit feedback, showing that sophisticated biased estimators used in previous approaches are not necessary.
We provide three examples (MAB, linear bandits, and MDPs) to show the versatility of our general approach, leading to several new algorithms and results.
Although not included in this work, we point out that our approach can also be straightforwardly applied to other problems such as semi-bandits and convex bandits, based on the algorithms from~\citep{wei2018more} and~\citep{saha2011improved} respectively,
since they are also based on log-barrier OMD or \scrible.

%

\begin{ack}
HL thanks Ashok Cutkosky and Dirk van der Hoeven for many helpful discussions on normal barriers and the lifting idea.
We are grateful for the support of NSF Awards IIS-1755781 and IIS-1943607, and a Google Faculty Research Award.
\end{ack}

\bibliography{ref.bib}

\begin{thebibliography}{38}
\providecommand{\natexlab}[1]{#1}
\providecommand{\url}[1]{\texttt{#1}}
\expandafter\ifx\csname urlstyle\endcsname\relax
  \providecommand{\doi}[1]{doi: #1}\else
  \providecommand{\doi}{doi: \begingroup \urlstyle{rm}\Url}\fi

\bibitem[Abernethy and Rakhlin(2009)]{abernethy2009beating}
Jacob Abernethy and Alexander Rakhlin.
\newblock Beating the adaptive bandit with high probability.
\newblock In \emph{Conference on Learning Theory}, 2009.

\bibitem[Abernethy et~al.(2008)Abernethy, Hazan, and
  Rakhlin]{abernethy2008competing}
Jacob~D Abernethy, Elad Hazan, and Alexander Rakhlin.
\newblock Competing in the dark: An efficient algorithm for bandit linear
  optimization.
\newblock In \emph{Conference on Learning Theory}, 2008.

\bibitem[Abernethy et~al.(2012)Abernethy, Hazan, and
  Rakhlin]{abernethy2012interior}
Jacob~D Abernethy, Elad Hazan, and Alexander Rakhlin.
\newblock Interior-point methods for full-information and bandit online
  learning.
\newblock \emph{IEEE Transactions on Information Theory}, 58\penalty0
  (7):\penalty0 4164--4175, 2012.

\bibitem[Abernethy et~al.(2019)Abernethy, Jung, Lee, McMillan, and
  Tewari]{abernethy2019online}
Jacob~D Abernethy, Young~Hun Jung, Chansoo Lee, Audra McMillan, and Ambuj
  Tewari.
\newblock Online learning via the differential privacy lens.
\newblock In \emph{Advances in Neural Information Processing Systems}, 2019.

\bibitem[Agarwal et~al.(2017)Agarwal, Luo, Neyshabur, and
  Schapire]{agarwal2017corralling}
Alekh Agarwal, Haipeng Luo, Behnam Neyshabur, and Robert~E Schapire.
\newblock Corralling a band of bandit algorithms.
\newblock In \emph{Conference on Learning Theory}, 2017.

\bibitem[Allen-Zhu et~al.(2018)Allen-Zhu, Bubeck, and Li]{allen2018make}
Zeyuan Allen-Zhu, S{\'e}bastien Bubeck, and Yuanzhi Li.
\newblock Make the minority great again: First-order regret bound for
  contextual bandits.
\newblock In \emph{International Conference on Machine Learning}, 2018.

\bibitem[Allenberg et~al.(2006)Allenberg, Auer, Gy{\"o}rfi, and
  Ottucs{\'a}k]{AllenbergAuGyOt06}
Chamy Allenberg, Peter Auer, L{\'a}szl{\'o} Gy{\"o}rfi, and Gy{\"o}rgy
  Ottucs{\'a}k.
\newblock Hannan consistency in on-line learning in case of unbounded losses
  under partial monitoring.
\newblock In \emph{International conference on Algorithmic Learning Theory},
  2006.

\bibitem[Audibert and Bubeck(2009)]{audibert2009minimax}
Jean-Yves Audibert and S{\'e}bastien Bubeck.
\newblock Minimax policies for adversarial and stochastic bandits.
\newblock In \emph{Conference on Learning Theory}, 2009.

\bibitem[Audibert et~al.(2011)Audibert, Bubeck, and
  Lugosi]{audibert2011minimax}
Jean-Yves Audibert, S{\'e}bastien Bubeck, and G{\'a}bor Lugosi.
\newblock Minimax policies for combinatorial prediction games.
\newblock In \emph{Proceedings of the 24th Annual Conference on Learning
  Theory}, 2011.

\bibitem[Auer et~al.(2002)Auer, Cesa-Bianchi, Freund, and
  Schapire]{auer2002nonstochastic}
Peter Auer, Nicol{\`o} Cesa-Bianchi, Yoav Freund, and Robert~E Schapire.
\newblock The nonstochastic multiarmed bandit problem.
\newblock \emph{SIAM Journal on Computing}, 32\penalty0 (1), 2002.

\bibitem[Awerbuch and Kleinberg(2004)]{awerbuch2004adaptive}
Baruch Awerbuch and Robert~D Kleinberg.
\newblock Adaptive routing with end-to-end feedback: Distributed learning and
  geometric approaches.
\newblock In \emph{Symposium on Theory of Computing}, 2004.

\bibitem[Bartlett et~al.(2008)Bartlett, Dani, Hayes, Kakade, Rakhlin, and
  Tewari]{bartlett2008high}
Peter~L Bartlett, Varsha Dani, Thomas Hayes, Sham Kakade, Alexander Rakhlin,
  and Ambuj Tewari.
\newblock High-probability regret bounds for bandit online linear optimization.
\newblock In \emph{Conference On Learning Theory}, 2008.

\bibitem[Braun and Pokutta(2016)]{braun2016efficient}
G{\'a}bor Braun and Sebastian Pokutta.
\newblock An efficient high-probability algorithm for linear bandits.
\newblock \emph{arXiv preprint arXiv:1610.02072}, 2016.

\bibitem[Bubeck et~al.(2012)Bubeck, Cesa-Bianchi, and
  Kakade]{bubeck2012towards}
S{\'e}bastien Bubeck, Nicolo Cesa-Bianchi, and Sham~M Kakade.
\newblock Towards minimax policies for online linear optimization with bandit
  feedback.
\newblock In \emph{Conference on Learning Theory}, 2012.

\bibitem[Bubeck et~al.(2017)Bubeck, Lee, and Eldan]{bubeck2017kernel}
S{\'e}bastien Bubeck, Yin~Tat Lee, and Ronen Eldan.
\newblock Kernel-based methods for bandit convex optimization.
\newblock In \emph{Symposium on Theory of Computing}, 2017.

\bibitem[Bubeck et~al.(2018)Bubeck, Cohen, and Li]{bubeck2018sparsity}
S{\'e}bastien Bubeck, Michael Cohen, and Yuanzhi Li.
\newblock Sparsity, variance and curvature in multi-armed bandits.
\newblock In \emph{Algorithmic Learning Theory}, 2018.

\bibitem[Bubeck et~al.(2019)Bubeck, Li, Luo, and Wei]{bubeck2019improved}
S{\'e}bastien Bubeck, Yuanzhi Li, Haipeng Luo, and Chen-Yu Wei.
\newblock Improved path-length regret bounds for bandits.
\newblock In \emph{Conference On Learning Theory}, 2019.

\bibitem[Cesa-Bianchi and Lugosi(2012)]{cesa2012combinatorial}
Nicolo Cesa-Bianchi and G{\'a}bor Lugosi.
\newblock Combinatorial bandits.
\newblock \emph{Journal of Computer and System Sciences}, 78\penalty0
  (5):\penalty0 1404--1422, 2012.

\bibitem[Dani et~al.(2008)Dani, Kakade, and Hayes]{dani2008price}
Varsha Dani, Sham~M Kakade, and Thomas~P Hayes.
\newblock The price of bandit information for online optimization.
\newblock In \emph{Advances in Neural Information Processing Systems}, 2008.

\bibitem[Foster et~al.(2016)Foster, Li, Lykouris, Sridharan, and
  Tardos]{FosterLiLySrTa16}
Dylan~J Foster, Zhiyuan Li, Thodoris Lykouris, Karthik Sridharan, and Eva
  Tardos.
\newblock Learning in games: Robustness of fast convergence.
\newblock In \emph{Advances in Neural Information Processing Systems}, 2016.

\bibitem[Freedman(1975)]{freedman1975tail}
David~A Freedman.
\newblock On tail probabilities for martingales.
\newblock \emph{the Annals of Probability}, pages 100--118, 1975.

\bibitem[Gy{\"o}rgy et~al.(2007)Gy{\"o}rgy, Linder, Lugosi, and
  Ottucs{\'a}k]{gyorgy2007line}
Andr{\'a}s Gy{\"o}rgy, Tam{\'a}s Linder, G{\'a}bor Lugosi, and Gy{\"o}rgy
  Ottucs{\'a}k.
\newblock The on-line shortest path problem under partial monitoring.
\newblock \emph{Journal of Machine Learning Research}, 8\penalty0
  (Oct):\penalty0 2369--2403, 2007.

\bibitem[Hazan and Kale(2011)]{hazan2011better}
Elad Hazan and Satyen Kale.
\newblock Better algorithms for benign bandits.
\newblock \emph{Journal of Machine Learning Research}, 12\penalty0
  (Apr):\penalty0 1287--1311, 2011.

\bibitem[Hazan and Karnin(2016)]{hazan2016volumetric}
Elad Hazan and Zohar Karnin.
\newblock Volumetric spanners: an efficient exploration basis for learning.
\newblock \emph{The Journal of Machine Learning Research}, 17\penalty0
  (1):\penalty0 4062--4095, 2016.

\bibitem[Jin et~al.(2020)Jin, Jin, Luo, Sra, and Yu]{jin2019learning}
Chi Jin, Tiancheng Jin, Haipeng Luo, Suvrit Sra, and Tiancheng Yu.
\newblock Learning adversarial markov decision processes with bandit feedback
  and unknown transition.
\newblock In \emph{International Conference on Machine Learning}, 2020.

\bibitem[Lattimore and Szepesv{\'a}ri(2018)]{lattimore2018bandit}
Tor Lattimore and Csaba Szepesv{\'a}ri.
\newblock \emph{Bandit algorithms}.
\newblock Cambridge University Press (preprint), 2018.

\bibitem[Lee et~al.(2020)Lee, Luo, and Zhang]{LLZ20}
Chung-Wei Lee, Haipeng Luo, and Mengxiao Zhang.
\newblock A closer look at small-loss bounds for bandits with graph feedback.
\newblock In \emph{Conference on Learning Theory}, 2020.

\bibitem[Luo et~al.(2018)Luo, Wei, and Zheng]{luo2018efficient}
Haipeng Luo, Chen-Yu Wei, and Kai Zheng.
\newblock Efficient online portfolio with logarithmic regret.
\newblock In \emph{Advances in Neural Information Processing Systems}, 2018.

\bibitem[Lykouris et~al.(2018)Lykouris, Sridharan, and
  Tardos]{lykouris2018small}
Thodoris Lykouris, Karthik Sridharan, and {\'E}va Tardos.
\newblock Small-loss bounds for online learning with partial information.
\newblock In \emph{Conference on Learning Theory}, 2018.

\bibitem[Nesterov and Nemirovskii(1994)]{nesterov1994interior}
Yurii Nesterov and Arkadii Nemirovskii.
\newblock \emph{Interior-point polynomial algorithms in convex programming},
  volume~13.
\newblock Siam, 1994.

\bibitem[Neu(2015{\natexlab{a}})]{neu2015explore}
Gergely Neu.
\newblock Explore no more: Improved high-probability regret bounds for
  non-stochastic bandits.
\newblock In \emph{Advances in Neural Information Processing Systems},
  2015{\natexlab{a}}.

\bibitem[Neu(2015{\natexlab{b}})]{neu2015first}
Gergely Neu.
\newblock First-order regret bounds for combinatorial semi-bandits.
\newblock In \emph{Conference on Learning Theory}, 2015{\natexlab{b}}.

\bibitem[Rakhlin and Sridharan(2013)]{rakhlin2013online}
Alexander Rakhlin and Karthik Sridharan.
\newblock Online learning with predictable sequences.
\newblock In \emph{Conference on Learning Theory}, 2013.

\bibitem[Rosenberg and Mansour(2019{\natexlab{a}})]{pmlr-v97-rosenberg19a}
Aviv Rosenberg and Yishay Mansour.
\newblock Online convex optimization in adversarial {M}arkov decision
  processes.
\newblock In \emph{Proceedings of the 36th International Conference on Machine
  Learning}, 2019{\natexlab{a}}.

\bibitem[Rosenberg and Mansour(2019{\natexlab{b}})]{rosenberg2019online}
Aviv Rosenberg and Yishay Mansour.
\newblock Online stochastic shortest path with bandit feedback and unknown
  transition function.
\newblock In \emph{Advances in Neural Information Processing Systems},
  2019{\natexlab{b}}.

\bibitem[Saha and Tewari(2011)]{saha2011improved}
Ankan Saha and Ambuj Tewari.
\newblock Improved regret guarantees for online smooth convex optimization with
  bandit feedback.
\newblock In \emph{International Conference on Artificial Intelligence and
  Statistics}, 2011.

\bibitem[Wei and Luo(2018)]{wei2018more}
Chen-Yu Wei and Haipeng Luo.
\newblock More adaptive algorithms for adversarial bandits.
\newblock In \emph{Conference On Learning Theory}, 2018.

\bibitem[Zimin and Neu(2013)]{zimin2013}
Alexander Zimin and Gergely Neu.
\newblock Online learning in episodic markovian decision processes by relative
  entropy policy search.
\newblock In \emph{Advances in Neural Information Processing Systems}, 2013.

\end{thebibliography}
\bibliographystyle{plainnat}

\newpage
\appendix

\section{Omitted details for \pref{sec:MAB}}
\subsection{\pfref{thm:Freedman}}

First we generalize the proof of the standard Freedman's inequality 
in the following way.
For any $\lambda_t$ that is $\calF_t$-measurable and such that $\lambda_t \leq 1/B_t$, we have with $\E_t[\cdot] \triangleq \E[\cdot |\calF_t]$:
\begin{equation}\label{eq:Eexp}
\mathbb{E}_t\left[e^{\lambda_tX_t}\right]\le \mathbb{E}_t\left[1+{\lambda_tX_t}+{\lambda_t^2X_t^2}\right]=1+\lambda_t^2\mathbb{E}_t\left[X_t^2\right]\le \exp\left({\lambda_t^2\mathbb{E}_t\left[X_t^2\right]}\right).
\end{equation}
Now for any $t$ define random variable $Z_t$ such that $Z_0=1$ and 
\[
Z_t\triangleq Z_{t-1}\cdot \exp({\lambda_tX_t}-\lambda_t^2\mathbb{E}_t\left[X_t^2\right])=\exp\left(\sum_{s=1}^t \lambda_s X_s -\sum_{s=1}^t \lambda_s^2 \mathbb{E}_s\left[X_s^2\right] \right).
\]
From \pref{eq:Eexp}, we have
\[
\mathbb{E}_t\left[Z_t\right]=Z_{t-1}\cdot\exp(-\lambda_t^2\mathbb{E}_t\left[X_t^2\right])\mathbb{E}_t\left[e^{\lambda_tX_t}\right]\le Z_{t-1}\cdot\exp(-\lambda_t^2\mathbb{E}_t\left[X_t^2\right])\exp(\lambda_t^2\mathbb{E}_t\left[X_t^2\right])\le Z_{t-1}.
\]
Therefore, taking the overall expectation we have
\[
\mathbb{E}\left[Z_T\right]\le \mathbb{E}\left[Z_{T-1}\right]\le \cdots\le \mathbb{E}\left[Z_0\right] = 1.
\]
Using Markov's inequality, we have $\Pr\left[Z_T\ge \frac{1}{\delta'}\right]\le \delta'$. In other words, we have with probability at least $1-\delta'$,
    \begin{equation}\label{eq:lambda}
    \sum_{t=1}^T \lambda_t X_t \leq \ln(1/\delta')+\sum_{t=1}^T \lambda_t^2 \mathbb{E}_t\left[X_t^2\right] .
    \end{equation}
    The proof of the standard Freedman's inequality takes all $\lambda_t$ to be the same fixed value,
    while in our case it is important to apply \pref{eq:lambda} several times with different sets of values of $\lambda_t$. 
    Specifically, for each $i\in [\lceil\log_2(b^2T)\rceil]$ and $j\in [\lceil\log_2 b\rceil]$, set
    \[
    \lambda_t =  \lambda\triangleq\min \left\{2^{-j}, \sqrt{ \ln(1/\delta') / 2^i}\right\}, 
    \]
    for $t\in \mathcal{T}_{j}$, where 
    \[\mathcal{T}_{j} \triangleq \left\{t:2^{j-1} \leq \max_{s\leq t} B_s \leq 2^j\right\},\] and
$\lambda_t = 0$ otherwise. 
Clearly $\lambda_t$ is $\calF_t$-measurable (since $B_1, \ldots, B_t$ are $\calF_t$-measurable).
Applying \pref{eq:lambda} gives 
\begin{align*}
    \sum_{t\in \mathcal{T}_{j} } X_t&\le \frac{\ln(1/\delta')}{\lambda}+\sum_{t\in\mathcal{T}_{j}}\lambda\mathbb{E}_t\left[X_t^2\right]\\
    &\le 2^j{\ln(1/\delta')}+\sqrt{2^i\ln(1/\delta')}+ \lambda \sum_{t=1}^{T}\mathbb{E}_t\left[X_t^2\right]\tag{$\frac{1}{\lambda}\le\max\{2^j,\sqrt{2^i/\ln(1/\delta')}\}$}\\
    &\leq  2\left(\max_{s\in \calT_j} B_s\right)  \ln(1/\delta') + \sqrt{2^i\ln(1/\delta')} + V\sqrt{\frac{\ln(1/\delta')}{2^i}}  \tag{$2^{j-1} \leq \max_{s\in \calT_j} B_s $}\\
    &\leq  2\Bstar  \ln(1/\delta')+ \sqrt{2^i\ln(1/\delta')} + V\sqrt{\frac{\ln(1/\delta')}{2^i}}.
\end{align*}
By a union bound, the above holds with probability at least $1-C\delta'$ for all $i\in [\lceil\log_2(b^2T)\rceil]$ and $j\in [\lceil\log_2 b\rceil]$.
In particular, since $1\le V\le b^2 T$ (almost surely), there exists an $i^\star \in[\lceil\log_2(b^2 T)\rceil]$ such that $2^{i^\star-1}\le V\le 2^{i^\star}$, and thus
\begin{align*}
    \sum_{t=1}^T X_t  =  \sum_{j\in [\lceil\log_2 b\rceil]}\sum_{t\in \mathcal{T}_{j}} X_t &\le C\cdot \left(2\Bstar  \ln(1/\delta')+\sqrt{2^{i^\star}\ln(1/\delta')} + V\sqrt{\frac{\ln(1/\delta')}{2^{i^\star}}}\right) \\
    &\le C\cdot \left(2\Bstar  \ln(1/\delta')+\sqrt{2V\ln(1/\delta')} + \sqrt{V\ln(1/\delta')} \right)  \\
    &\leq C\cdot \left(2\Bstar  \ln(1/\delta')+\sqrt{8V\ln(1/\delta')}\right).
\end{align*}
Finally replacing $\delta'$ with $\delta/C$ finishes the proof.

\subsection{\pfref{lem:MAB_lemma}}\label{app:lemma2.1}
First note that $\ell_{t,i_t}=\inner{w_t,\ellhat_{t}}$.
	Using standard OMD analysis (e.g.,~\citep[Lemma~12]{agarwal2017corralling}),  we have 
\begin{align}
\ell_{t, i_t}-\inner{u,\ellhat_{t}}&\le D_{\psi_t}(u,w_t)-D_{\psi_t}(u,w_{t+1})+\sum^d_{i=1}\eta_{t,i}w_{t,i}^2\ellhat_{t,i}^2.\label{eq:DDell}
\end{align}
Summing the first two terms on the right hand side over $t$ shows (here $h(y) = y - 1 - \ln y$):
\begin{align}
	&\sum_{t=1}^T \left(D_{\psi_t}(u, w_t)-D_{\psi_t}(u,w_{t+1})\right)\nonumber\\
    &\le D_{\psi_1}(u, w_1)+\sum_{t=1}^{T-1}\left(D_{\psi_{t+1}}(u, w_{t+1})-D_{\psi_t}(u, w_{t+1})\right)\nonumber\tag{$D_{\psi_T}(u,w_{T+1})\ge 0$}\\
	&= \frac{1}{\eta}\sum^d_{j=1}h\left(\frac{u_j}{w_{1,j}}\right)+\sum_{j=1}^{d}\sum_{t=1}^{T-1}\left(\frac{1}{\eta_{t+1,j}}-\frac{1}{\eta_{t,j}}\right)h\left(\frac{u_j}{w_{t+1},j}\right) \label{eq:etah}.
\end{align}
For the first term, since $u_j\ge \frac{1}{T}$ and $w_{1,j}=\frac{1}{d}$ for each $j$, we have \[\frac{1}{\eta}\sum^d_{j=1}h\left(\frac{u_j}{w_{1,j}}\right)=\frac{1}{\eta}\sum_{j=1}^d -\ln(du_j)\le \frac{d\ln T}{\eta}.
\]
Now we analyze the second term for each $j$. 
Note that $\eta_{T,j}=\kappa^{n_j}\eta_{1,j}$ where $n_j$ is the number of times \pref{alg:mab} increases the learning rate for arm $j$.
Let $t_j$ be the time step such that $\eta_{T,j}=\eta_{t_j+1,j}=\kappa\eta_{t_j,j}$, that is, the last time step where the learning rate for arm $j$ is increased. Then we have
\begin{align*} \left(\frac{1}{\eta_{t_j+1,j}}-\frac{1}{\eta_{t_j,j}}\right)h\left(\frac{u_j}{w_{t_j+1},j}\right)&=\frac{1-\kappa}{\kappa^{n_j}\eta} h\left(\frac{u_j}{w_{t_{j}+1,j}}\right) \le \frac{-h\left(\frac{u_j}{w_{t_{j}+1,j}}\right)}{5\eta \ln T}=\frac{-h\left(\frac{u_j\rho_{T,j}}{2}\right)}{5\eta \ln T},
\end{align*}
	where we use the facts $1 - \kappa \leq -\frac{1}{\ln T}$ and $\kappa^{n_j}\leq 5$. The term  $-h\left(\frac{u_j\rho_{T,j}}{2}\right)$ is bounded by 
    $$
    -h\left(\frac{u_j\rho_{T,j}}{2}\right)= \ln \left(\frac{u_j\rho_{T,j}}{2}\right)-\frac{u_j\rho_{T,j}}{2}+1\le 1+\ln T -\frac{u_j\rho_{T,j}}{2},
    $$
    where the inequality is because $\frac{u_j\rho_{T,j}}{2}\le\frac{1}{w_{t_j+1,j}}\le T$. Plugging this result for every $j$ back to \pref{eq:etah}, we get
    \begin{align*}
    \sum_{t=1}^T D_{\psi_t}(u, w_t)-D_{\psi_t}(u,w_{t+1}) &\le \frac{d\ln T}{\eta}+\sum_{j=1}^{d}\frac{2+2\ln T -{u_j\rho_{T,j}}}{10\eta\ln T}=\order\left(\frac{d\ln T}{\eta}\right)-\frac{\inner{\rho_{T},u}}{10\eta\ln T}.
    \end{align*}
    Finally, since $\eta_{t,i}w_{t,i}^2\ellhat_{t,i}^2\le \eta_{t,i_t}\ell_{t, i_t}\le  \eta_{T,i_t}\ell_{t, i_t}\le 5 \eta\ell_{t, i_t}$, summing \pref{eq:DDell} over $t$ gives:
    \begin{align*}
    \sum^T_{t=1}\left(\ell_{t, i_t}-\inner{u,\ellhat_{t}}\right)&\le\sum^T_{t=1} \left(D_{\psi_t}(u,w_t)-D_{\psi_t}(u,w_{t+1})\right)+\sum^T_{t=1}\sum^d_{i=1}\eta_{t,i}w_{t,i}^2{\ellhat_{t,i}}^2\\
    &\le \order\left(\frac{d\ln T}{\eta}\right)-\frac{\inner{\rho_{T},u}}{10\eta\ln T}+5\eta\sum^T_{t=1}\ell_{t, i_t}\\
    &=\order\left(\frac{d\ln T}{\eta} + \eta \sum_{t=1}^T \ell_{t, i_t} \right) - \frac{\inner{\rho_T, u}}{10\eta\ln T}.
    \end{align*}
\subsection{\pfref{thm:MAB}}\label{app:theorem2.3}
    Fix any $\istar \in [d]$ and let $u=(1-\frac{d}{T})e^\star+\frac{1}{T}\bf 1$, where $e^\star$ is the one-hot vector for $\istar$. First note that
\begin{align*}
\sum^T_{t=1}\left(\ell_{t,i_t}-\ell_{t,\istar}\right)&=\sum^T_{t=1}\left(\ell_{t,i_t}-\inner{u,\hat{\ell}_{t}}\right)+\sum_{t=1}^T\inner{u,\ellhat_t-\ell_t}+\sum_{t=1}^T\inner{u-e^\star,\ell_t}\nonumber \\
&\le \sum^T_{t=1}\left(\ell_{t,i_t}-\inner{u,\hat{\ell}_{t}}\right)+\sum_{t=1}^T\inner{u,\ellhat_t-\ell_t}+d. 
\end{align*}
For the first term, using \pref{lem:MAB_lemma}, we have
\begin{align}\label{eq:mabreg}
\sum^T_{t=1}\left(\ell_{t,i_t}-\inner{u,\hat{\ell}_{t}}\right)&\le \order\left(\frac{d\ln T}{\eta}+\eta\LT\right)  -\frac{\inner{\rho_{T},u}}{10\eta\ln T},
\end{align}
where $\LT = \sum_{t=1}^T\ell_{t,i_t}$.

For the second term above, we use \pref{thm:Freedman} with $X_t=\inner{u,\ellhat_t-\ell_t}$, $B_t=\inner{\rho_{t},u}\in[1,T]$, $b=T$, and the fact
\begin{align*}
 \E_t[X_t^2] &\leq  \E_t\sbr{\inner{u,\ellhat_t}^2}
=\E_t\sbr{\frac{u_{i_t}^2\ell_{t,i_t}^2}{p_{t,i_t}^2}}  
\leq \sum_{i=1}^d u_i^2 \ell_{t,i} \rho_{T,i} \le\inner{\rho_{T},u} \inner{u, \ell_t},
\end{align*}
showing that with probability at least $1-\delta'$,
\begin{equation}\label{eq:General}
\sum^T_{t=1}\inner{u,\ellhat_t-\ell_t}\leq  C\left(\sqrt{8 L_u \inner{\rho_T,u}\ln\left(C/\delta'\right)} + 2\inner{\rho_{T},u} \ln\left(C/\delta'\right)\right),
\end{equation}
where $L_u = \inner{u, \sum^T_{t=1}\ell_{t}}$ 
and $C=\ceil{\log(b)}\ceil{\log(b^2T)}= \ceil{\log(T)}\ceil{3\log(T)}$.
With $\eta\le\frac{1}{40C\ln T\ln (C/\delta')}$, we then have with probability at least $1-\delta'$,
\begin{align*}
&\sum^T_{t=1}\ell_{t,i_t}-\ell_{t,\istar}\\
&\le\otil\left(\frac{d}{\eta}+\eta \LT\right) -\frac{\inner{\rho_{T},u}}{10\eta\ln T}+C\left(\sqrt{8L_u\inner{\rho_T,u}\ln\left(C/\delta'\right)} + 2\inner{\rho_{T},u} \ln\left(C/\delta'\right)\right)\tag{\pref{eq:mabreg} and \pref{eq:General}}\\
&\le\otil\left(\frac{d}{\eta}+\eta \LT \right) +\eta 40C^2 L_u\ln(C/\delta')\ln T -\frac{\inner{\rho_{T},u}}{20\eta\ln T}+2C\inner{\rho_{T},u}\ln (C/\delta')\tag{AM-GM inequality}\\
&\le\otil\left(\frac{d}{\eta}+\eta \LT\ln(1/\delta') + \eta L_u \ln(1/\delta')\right)\tag{$\eta<\frac{1}{40C\ln T\ln (C/\delta)}$}.
\end{align*}
Therefore, rearranging the terms, using the fact $L_u \leq  L^\star + d$, and choosing $\eta=\min\left\{\sqrt{\frac{d}{\Lstar}\ln(1/\delta')},\frac{1}{40C\ln T\ln(C/\delta')},\frac{1}{2}\right\}$, we have with probability $1-\delta'$, 
\[
\sum^T_{t=1}\ell_{t,i_t}-\ell_{t,\istar} = \otil\left(\sqrt{d\Lstar\ln(1/\delta')}+d\ln(1/\delta')\right),
\]
where $\Lstar = \sum_{t=1}^T\ell_{t,\istar}$. This finishes the proof when the adversary is oblivious. For  adaptive adversaries, taking a union bound over all possible best arms $\istar\in[d]$ and setting $\delta'=\delta/d$, we have with probability $1-\delta$,  $\Reg = \widetilde{\order}\left(\sqrt{d\Lstar\ln(d/\delta)}+d\ln(d/\delta)\right)$, finishing the proof.

\begin{remark}\emph{
Although the proof above requires tuning the initial learning rate $\eta$ in terms of the unknown quantity $\Lstar$, standard doubling trick can remove this restriction (even in the bandit setting).
We refer the reader to a recent work by~\citet{LLZ20} for detailed exposition on how to achieve so. 
}\label{rem:doubling_trick}
\end{remark}


\section{Omitted details for Section~\ref{sec:LB}}


\subsection{More explanation on \pref{alg:linear bandit d+1}}\label{app:alg-d}



Here, we provide a $d$-dimensional version of \pref{alg:linear bandit d+1} by removing the explicit lifting and performing OMD in $\fR^d$; see \pref{alg:LB_ddim}.
It is clear that this version is exactly the same as \pref{alg:linear bandit d+1}.
Compared to the original \scrible, one can see that besides the increasing learning rate schedule, the only difference is how the point $\wtilde_t$ is computed.
In particular, one can verify that $\wtilde_t$ does not necessarily satisfy $\|\wtilde_t - w_t\|_{\nabla^2\psi(w_t)}=1$, meaning that $\wtilde_t$ is not necessarily on the boundary of the Dikin ellipsoid centered at $w_t$ with respect to $\psi$.
In other words, our algorithm provides a new sampling scheme for \scrible.

\setcounter{AlgoLine}{0}
\begin{algorithm}[t]
	\caption{$d$-dimensional version of \pref{alg:linear bandit d+1}}\label{alg:linear bandit}
	\label{alg:LB_ddim}
	\textbf{Input:} decision set $\Omega\subseteq \mathbb{R}^d$, a $\nu$-self-concordant barrier $\psi(w)$ for $\Omega$, initial learning rate $\eta$.
    
	\textbf{Define:} increase factor $\kappa=e^{\frac{1}{100d\ln(\nu T)}}$, $\Psi(w,b)=400\left(\psi\left(\frac{w}{b}\right)-2\nu \ln b\right)$.
    
    \textbf{Initialize:} $w_1 = \argmin_{w\in \Omega}\psi(w)$, $\dpH_1=\nabla^2\Psi(w_1,1)$,  $\eta_1=\eta$,  $\calS=\{1\}$.
    
    \textbf{Define:} shrunk decision set $\Omega' = \{w\in \Omega$: $\pi_{w_1}(w)\le 1-\frac{1}{T}\}$, $\Jd=[I_d,\mathbf{0}_d]\in \mathbb{R}^{d\times (d+1)}$.
	
	\For{$t=1,2,\dots, T$}{
		
		\nl Uniformly at random sample $\dps_t$ from $\left(\dpH_t^{-\frac{1}{2}}e_{d+1}\right)^\perp\cap \mathbb{S}^{d+1}$.
		
		\nl Compute $\wtilde_t=w_t+\Jd \dpH_t^{-\frac{1}{2}}\dps_t$.
		
		\nl Play $\wtilde_t$, observe loss $\langle \wtilde_t,\ell_t\rangle$, and construct loss estimator $\ellhat_t=d \langle \wtilde_t,\ell_t\rangle \Jd \dpH_t^{1/2}\dps_t$.
		
		\nl Compute $w_{t+1}=\argmin_{w \in \Omega'}\left\{\left\langle w, \ellhat_{t}\right\rangle+D_{\psi_t}\left(w, w_{t}\right)\right\}$, where $\psi_t = \frac{1}{\eta_t}\psi$.
		
		\nl Compute $\dpH_{t+1}=\nabla^2\Psi(w_{t+1}, 1)$.
		
		\nl \lIf {$\lambdamax(\dpH_{t+1}-\sum_{\tau\in \calS}\dpH_{\tau}) > 0$}{
			 $\calS\leftarrow \calS\cup\{t+1\}$ and set $\eta_{t+1}=\eta_{t}\kappa$;
		}
		\nl \lElse{
		set $\eta_{t+1}=\eta_{t}$.	
		}
	}
\end{algorithm}

\subsection{Preliminary for analysis}\label{app:LB-pre}

In this section, we introduce the preliminary of self-concordant barriers and normal-barriers, including the definitions and some useful properties that will be used frequently in later analysis.

\paragraph{Self-concordant barriers.}
Let $\psi:\intO\to \fR$ be a $C^3$ smooth convex function. $\psi$ is called a self-concordant barrier on $\Omega$ if it satisfies: 
\begin{itemize}
\item $\psi(x_i)\to \infty$ as $i\to \infty$ for any sequence $x_1,x_2,\dots\in\intO\subset \fR^d$ converging to the boundary of $\Omega$;
\item for all $w\in\intO$ and $h\in\fR^d$, the following inequality always holds:
$$
\sum^d_{i=1}\sum^d_{j=1}\sum^d_{k=1}\frac{\partial^3\psi(w)}{\partial w_i\partial w_j\partial w_k}h_ih_jh_k\le 2\|h\|_{\nabla^2\psi(w)}^3.
$$

\end{itemize}
We further call $\psi$ is a $\nu$-self-concordant barrier if it satisfies the conditions above and also
$$\inner{\nabla\psi(w),h}\le \sqrt{\nu}\|h\|_{\nabla^2\psi(w)}$$
for all $w\in\intO$ and $h\in\fR^d$.
\begin{lemma}[Theorem 2.1.1 in \citep{nesterov1994interior}]\label{lem:Dikin}
If $\psi$ is a self-concordant barrier on $\Omega$, then the {Dikin ellipsoid} centered at $w\in \intO$, defined as $\{v: \norm{v-w}_{\nabla^2\psi(w)}\leq 1\}$, is always within $\Omega$. Moreover, 
\begin{equation*}
\|h\|_{\nabla^2\psi(v)} \geq\|h\|_{_{\nabla^2\psi(w)}}\left(1-\left\|v-w\right\|_{_{\nabla^2\psi(w)}}\right)
\end{equation*} holds for any $h \in \mathbb{R}^{d}$ and any $v$ with $\norm{v-w}_{\nabla^2\psi(w)}\leq 1$.
\end{lemma}
\begin{lemma}[Theorem 2.5.1 in \citep{nesterov1994interior}]\label{lem:exist}
For any closed convex body $\Omega\subset \fR^d$, there exists an $\order(d)$-self-concordant barrier on $\Omega$.
\end{lemma}
\begin{lemma}[Corollary 2.3.1 in \citep{nesterov1994interior}]\label{lem:2-3-1}
Let $\psi$ be a self-concordant barrier for $\Omega\subset \fR^d$. Then for any $w\in\intO$ and any $h\in\Omega$ such that $w+bh\in \Omega$ for all $b\ge 0$, we have
\[
\|h\|_{\nabla^2\psi(w)}\le-\inner{\nabla\psi(w),h}.
\]
\end{lemma}
\paragraph{Normal Barriers.}
Let $K\subseteq \fR^d$ be a closed and proper convex cone and let $\theta\ge 1$. A function $\psi:\intK\to \fR$ is called a $\theta$-logarithmically homogeneous self-concordant barrier (or simply a $\theta$-normal barrier) on $K$ if it is self-concordant on $\intK$ and is logarithmically homogeneous with parameter $\theta$, which means
\[
\psi(tw)=\psi(w)-\theta \ln t,~\forall w\in\intK,~t>0.
\]

The following two lemmas show the relationship between $\theta$-normal barriers and $\theta$-self-concordant barriers.

\begin{lemma}[Corollary 2.3.2 in \citep{nesterov1994interior}]\label{lem:NBisSCB}
A $\theta$-normal barrier on $K$ is a $\theta$-self-concordant barrier on $K$.
\end{lemma}

\begin{lemma}[Proposition 5.1.4 in \citep{nesterov1994interior}]\label{lem:prop5-1-4}
	Suppose $f$ is a $\theta$-self-concordant barrier on $K \subseteq \fR^d$. Then the function
	\[F(w,b)=400\left(f\left(\frac{w}{b}\right)-2\theta\ln b\right), \]
	is a $800\theta$-normal barrier for $con(K) \subseteq \fR^{d+1}$, where $con(K)=\{\mathbf{0}\}\cup\{(w,b):\frac{w}{b}\in K, w \in \fR^d, b> 0\}$ is the conic hull of $K$ lifted to $\fR^{d+1}$ (by appending $1$ to the last coordinate).
\end{lemma}

Note that our regularizer $\Psi$ defined in \pref{alg:linear bandit d+1} is exactly based on this formula.
We point out that, however, our entire analysis works for any $\order(\nu)$-normal barrier $\Psi$, as we will only use the following general properties of normal barriers, instead of the concrete form of $\Psi$.
As mentioned in \pref{fn:normal_barrier}, we use this concrete formula only to emphasize that, just as \scrible, our algorithm requires only a self-concordant barrier of the original set $\Omega$.

\begin{lemma}[Proposition 2.3.4 in \citep{nesterov1994interior}]\label{lem:NBprpty}
If $\psi$ is a $\theta$-normal barrier on $K$, then we have for all $w,u\in\intK$,
\begin{enumerate}
\item $\|w\|_{\nabla^2\psi(w)}^2=w^\top\nabla^2\psi(w)w=\theta$,
\item $\nabla^2\psi(w)w=-\nabla\psi(w)$,
\item $\psi(u)\ge\psi(w)-\theta\ln\frac{-\inner{\nabla\psi(w),u}}{\theta}$.
\end{enumerate}
\end{lemma}

Next, we show the definition of Minkowsky functions, which is used to define the shrunk decision domain similar to the clipped simplex in multi-armed bandit setting.

\paragraph{Minkowsky functions.}
The Minkowsky function of a convex body $\Omega$ with the pole at $w\in \intO$ is a function $\pi_w:\Omega\to \fR$ defined as
\[
\pi_w(u)=\inf\left\{t>0\left|w+\frac{u-w}{t}\in \Omega\right.\right\}.
\]

The last lemma shows several useful properties using the Minkowsky function.
\begin{lemma}[Proposition 2.3.2 in \cite{nesterov1994interior}]\label{lem:MFprpty}
Let $\psi$ be a $\nu$-self-concordant barrier on $\Omega\subseteq \fR^d$ and $u,w\in\intO$. Then for any $h\in \fR^d$, we have
\begin{align*}
\|h\|_{\nabla^2\psi(u)}&\le \left(\frac{1+3{\nu}}{1-\pi_w(u)}\right)\|h\|_{\nabla^2\psi(w)},\\
|\inner{\nabla\psi(u),h}|&\le \left(\frac{\nu}{1-\pi_w(u)}\right)\|h\|_{\nabla^2\psi(w)},\\
\psi(u)-\psi(w)&\le \nu\ln\rbr{\frac{1}{1-\pi_w(u)}}.
\end{align*}
\end{lemma}


\subsection{\pfref{thm:LB}}

To prove the theorem, we decompose the regret against any fixed $u^\star \in \Omega$ (with $\dpu^\star = (u^\star, 1) \in \dpOmega$) into the following three terms:
\begin{align}
	&\sum_{t=1}^T\inner{\wtilde_t-u^\star,\ell_t}\nonumber\\&=\sum_{t=1}^T\inner{\dpwtilde_t-\dpu^\star,\dpell_t}\tag{define $\dpell_t = (\ell_t, 0)$}\nonumber\\
	&= \underbrace{\sum_{t=1}^T\left(\inner{\dpwtilde_t,\dpell_t}-\inner{\dpw_t,\dpellhat_t} + \inner{\dpu, \dpellhat_t-\dpell_t}\right)}_{\biasone} + \underbrace{\sum_{t=1}^T\inner{\dpw_t-\dpu, \dpellhat_t}}_{\regterm}+ \sum_{t=1}^T\inner{\dpu-\dpu^\star,\dpell_t}\label{eq:decomp},
\end{align}
where $\dpu=\left(1-\frac{1}{T}\right)\cdot \dpu^\star+\frac{1}{T}\cdot \dpw_1 \in \dpOmega'$. 
Note that the last term is trivially bounded by $2$ as $\sum_{t=1}^T\inner{\dpu-\dpu^\star, \dpell_t} = \sum_{t=1}^T\inner{u-u^\star, \ell_t} = \frac{1}{T}\sum_{t=1}^T\inner{u^\star-w_1,\ell_t}\le 2$, where the last inequality is because $|\inner{w,\ell_t}|\le 1$ for all $w\in \Omega$.
In the following sections, we show how to bound other terms.
Specifically, we bound \biasone in  \pref{sec:LB_BIAS1} and \regterm in \pref{sec:LB_REG}.  Finally we prove \pref{thm:LB} in \pref{sec:thm3.1}. 

We will use the following notations in the remaining of this section (the first two are mentioned above already):
\begin{equation}\label{eq:LB_def1}
\dpell_t \triangleq (\ell_t, 0), \quad \dpu\triangleq(u,1) \triangleq \left(1-\frac{1}{T}\right)\cdot \dpu^\star+\frac{1}{T}\cdot \dpw_1 \in \dpOmega',
\quad \Ustar\triangleq\max_{t\in [T]}\|\dpu\|_{\dpH_t}, 
\end{equation}
\begin{equation}\label{eq:LB_def2}
\LT\triangleq\sum_{t=1}^T\inner{\dpwtilde_t, \dpell_t},\quad 
 \LTbar\triangleq\sum_{t=1}^T\left|\inner{\dpwtilde_t, \dpell_t}\right|, 
\end{equation}
\begin{equation}\label{eq:LB_def3}
\LTbarep\triangleq\sum_{t=1}^T\mathbb{E}_t\left[\left|\inner{\dpwtilde_t, \dpell_t}\right|\right], \quad
\Lubar \triangleq \sum_{t=1}^T\left|\inner{u,\ell_t}\right|.
\end{equation}

Before proceeding, we provide one useful lemma.

%

\begin{lemma}\label{lem:LB-uNorm}
	We have $\|\dpu\|_{\dpH_1}\le 800\nu.$
\end{lemma}
\begin{proof}
   Clearly, for any $b>0$, we have $\dpw_1 + b\dpu$ still in the conic hull of $\dpOmega$.
	According to \pref{lem:2-3-1}, we thus have $\|\dpu\|_{\dpH_1} \le  \inner{-\nabla\Psi(\dpw_1), \dpu}$. Note that $\Psi$ is a $800\nu$-normal barrier by \pref{lem:prop5-1-4}. By the first order optimality condition of $\dpw_1$ and \pref{lem:NBprpty}, we then have
	\[0\le\inner{\nabla\Psi(\dpw_1),\dpu-\dpw_1}=\inner{\nabla\Psi(\dpw_1),\dpu}+800\nu.\]
	Combining the above gives $\|\dpu\|_{\dpH_1}\le\inner{-\nabla\Psi(\dpw_1), \dpu}\le 800\nu.$
\end{proof}


\subsubsection{Bounding \biasone}\label{sec:LB_BIAS1}
We first show that $\dpellhat_t$ is an unbiased estimator of $\dpell_t$ for the first $d$ coordinates.
\begin{lemma}\label{lem:LB-unbiased}
We have $\E_t\sbr{\dpellhat_{t,i}} = \dpell_{t,i}$ for $i \in [d]$.
\end{lemma}

\begin{proof}
	Let $\dpv=\frac{\dpH_t^{-1/2}e_{d+1}}{\left\|\dpH_t^{-1/2}e_{d+1}\right\|_2}$. First note that
	\begin{equation}\label{eq:orthnml}
	\E_t[\dps_t \dps_t^\top]= \frac{1}{d}\rbr{\dpI-\dpv\dpv^\top}
	\end{equation}
	by the definition of $\dps_t$. Then by the definition of $\dpellhat_t$, we have
	\begin{align*}
	\mathbb{E}_t\left[\dpellhat_t\right] &= \mathbb{E}_t\left[d \inner{\dpw_t+\dpH_t^{-\frac{1}{2}}\dps_t, \dpell_t}\cdot\dpH_t^{\frac{1}{2}}\dps_t\right] \\
	&= \mathbb{E}_t\left[d\inner{\dpw_t,\dpell_t}\cdot\dpH_t^{\frac{1}{2}}\dps_t+d\cdot\dpH_t^{\frac{1}{2}}\dps_t\inner{\dpH_t^{-\frac{1}{2}}\dps_t,\dpell_t}\right]\\
	&= d\inner{\dpw_t,\dpell_t}\cdot\dpH_t^{\frac{1}{2}}\mathbb{E}_t\left[\dps_t\right]+\mathbb{E}_t\left[d\cdot\dpH_t^{\frac{1}{2}}\dps_t\dps_t^\top\dpH_t^{-\frac{1}{2}}\dpell_t\right]\\
	&= d\cdot \dpH_t^{\frac{1}{2}}\mathbb{E}_t\left[\dps_t\dps_t^\top\right]\dpH_t^{-\frac{1}{2}}\dpell_t \tag{$\mathbb{E}_t\left[\dps_t\right]=\mathbf{0}$ by symmetry}\\
	&= \dpH_t^{\frac{1}{2}}\left(\dpI-\dpv\dpv^\top\right)\dpH_t^{-\frac{1}{2}}\dpell_t\tag{\pref{eq:orthnml}}\\
	&=\dpell_t-\frac{e_{d+1}e_{d+1}^\top \dpH_t^{-1}\dpell_t}{\left\|\dpH_t^{-1/2}e_{d+1}\right\|_2^2}.\\
	\end{align*}
	Noticing that the first $d$ coordinates of ${e_{d+1}e_{d+1}^\top \dpH_t^{-1}\dpell_t}$ are all zeros  concludes the proof.
\end{proof}
Now we are ready to bound \biasone.
\begin{lemma}\label{lem:LB-BIAS-1}
	With probability at least $1-\delta$, we have
	\begin{align*}
	\biasone&\le 161Cd\sqrt{\left(\nu+\Ustar^2\right)\LTbarep\ln(C/\delta)}+C\sqrt{32\Lubar\ln(C/\delta)}+ 64Cd\left(\sqrt{\nu}+\Ustar\right)\ln(C/\delta),
	\end{align*}
	where $C = \Theta(\ln^2(d\nu T))$.
\end{lemma}

\begin{proof}
	Define $X_t\triangleq \inner{\dpwtilde_t,\dpell_t}-\inner{\dpw_t, \dpellhat_t}+\inner{\dpu, \dpellhat_t-\dpell_t}$ and we have $\biasone = \sum_{t=1}^TX_t$. 
	The goal is to apply our strengthened Freedman's inequality \pref{thm:Freedman}.
	To this end,
 first we show $\E_t[X_t]=0$. Indeed, we have $\E_t[\dpwtilde_t] = \dpw_t$ and 
	\begin{align*}
		\mathbb{E}_t\left[X_t\right] &= \inner{w_t,\ell_t}-\inner{w_t, \ell_t} + \inner{u,\ell_t-\ell_t}-\mathbb{E}_t\left[(\dpw_{t,d+1}-\dpu_{t,d+1})\dpellhat_{t,d+1}\right]\tag{\pref{lem:LB-unbiased}}\\
		&= 0. \tag{$\dpw_{t,d+1}=\dpu_{t,d+1}=1$}
    \end{align*}
    Next, we bound $X_t$ by a $\calF_t$-measurable random variable $B_t\triangleq 32d\sqrt{\nu}+d\|\dpu\|_{\dpH_t}  $. This can be shown using the properties of a normal barrier:
    \begin{align*} 
		X_t&=\inner{\dpwtilde_t,\dpell_t}-\inner{\dpw_t, \dpellhat_t}+\inner{\dpu,\dpellhat_t-\dpell_t}\\
		&= \inner{\dpwtilde_t,\dpell_t}-\inner{\dpw_t, d\cdot \inner{\dpwtilde_t,\dpell_t}\cdot \dpH_t^{\frac{1}{2}}\dps_t} + \inner{\dpu,d\cdot \inner{\dpwtilde_t,\dpell_t}\dpH_t^{\frac{1}{2}}\dps_t-\dpell_t}\\
		&= \inner{\dpwtilde_t,\dpell_t}\left(1-d \dpw_t^\top \dpH_t^{\frac{1}{2}}\dps_t\right)+d\inner{\dpwtilde_t, \dpell_t}\dpu^\top \dpH_t^{\frac{1}{2}}\dps_t-\inner{\dpu,\dpell_t}\\
		&\le 2+d\left|\dpw_t^\top \dpH_t^{\frac{1}{2}}\dps_t\right| + d\left|\dpu^\top \dpH_t^{\frac{1}{2}}\dps_t\right| \tag{$|\inner{w, \ell_t}|\leq 1$ for any $w\in\Omega$}\\
		&\le 2+d\|\dpw_t\|_{\dpH_t}+d\|\dpu\|_{\dpH_t} \tag{by Cauchy-Schwarz inequality and $\dps_t^\top \dps_t = 1$}\\
		&\le 2+20d\sqrt{2\nu}+d\|\dpu\|_{\dpH_t} \tag{\pref{lem:prop5-1-4} and \pref{lem:NBprpty}}\\
        &\le 32d\sqrt{\nu}+d\|\dpu\|_{\dpH_t} \tag{$\nu\ge 1$}.
        \end{align*}
    Then, we show that $B_t$ is bounded by a constant $b\triangleq2\times 10^6d\nu^2T $ for all $t$:
        \begin{align*}
	B_t &\le 32d\sqrt{\nu}+d\|\dpu\|_{\dpH_1}\cdot\left(\frac{1+2400\nu}{1-\pi_{\dpw_1}(\dpw_t)}\right) \tag{\pref{lem:MFprpty}}\\
    &\le 32d\sqrt{\nu}+d\|\dpu\|_{\dpH_1}(1+2400\nu)T  \tag{$\dpw_t\in \dpOmega'$}\\
		&\le 32d\sqrt{\nu}+800d\nu(1+2400\nu)T \tag{\pref{lem:LB-uNorm}}\\
        &\le 2\times  10^6d\nu^2T\tag{$\nu\ge 1$}.
    \end{align*}   
    The last step before applying \pref{thm:Freedman} is to calculate $\E_t[X_t^2]$. We first write
        \begin{align}
		\E_t[X_t^2]&=\mathbb{E}_t\left[\left(\inner{\dpwtilde_t,\dpell_t}-\inner{\dpw_t, \dpellhat_t}+\inner{\dpu, \dpellhat_t-\dpell_t}\right)^2\right]\nonumber\\
        &\le 2\E_t\left[\left(\inner{\dpwtilde_t,\dpell_t}-\inner{\dpw_t, \dpellhat_t}\right)^2\right]+ 2\E_t\left[\inner{\dpu, \dpellhat_t-\dpell_t}^2\right]\label{eq:varX}.
        \end{align}
    The first term is bounded by:
        \begin{align*}
		&\E_t\left[\left(\inner{\dpwtilde_t,\dpell_t}-\inner{\dpw_t, \dpellhat_t}\right)^2\right]\\
		& = \mathbb{E}_t\left[\inner{\dpwtilde_t,\dpell_t}^2\left(1-\inner{\dpw_t, d\cdot \dpH_t^{\frac{1}{2}}\dps_t}\right)^2\right] \\
        &\le \mathbb{E}_t\left[ \inner{\dpwtilde_t,\dpell_t}^2\left(2d^2\left(\dpw_t^\top \dpH_t^{\frac{1}{2}}\dps_t\right)^2+2\right)\right] \\
		&\le \mathbb{E}_t\left[ \left|\inner{\dpwtilde_t,\dpell_t}\right|\left(2d^2\left(\dpw_t^\top \dpH_t^{\frac{1}{2}}\dps_t\right)^2+2\right)\right] \tag{$\inner{\dpwtilde_t,\dpell_t}\le 1$}\\
		&\le \mathbb{E}_t\left[ \left|\inner{\dpwtilde_t,\dpell_t}\right|\left(2d^2\|\dpw_t\|_{\dpH_t}^2\|\dps_t\|_2^2+2\right)\right] \tag{Cauchy-Schwarz inequality}\\
		&\le \mathbb{E}_t\left[ \left|\inner{\dpwtilde_t,\dpell_t}\right|\left(1600d^2\nu+2\right)\right] \tag{$\|\dps_t\|^2_2=1$ and \pref{lem:NBprpty}}\\
		& \le 1602d^2\nu\mathbb{E}_t\left[ \left|\inner{\dpwtilde_t,\dpell_t}\right|\right].
	\end{align*}
    Similarly, the second term is bounded by:
    \begin{align*}
    		 \E_t\left[\inner{\dpu, \dpellhat_t-\dpell_t}^2\right]& \le \mathbb{E}_t\left[ \left(-\inner{\dpu,\dpell_t}+d\inner{\dpwtilde_t, \dpell_t}\dpu^\top \dpH_t^{\frac{1}{2}}\dps_t\right)^2\right] \\
    		&\le  \mathbb{E}_t\left[2 \left|\inner{\dpu, \dpell_t}\right|+2d^2\left|\inner{\dpwtilde_t,\dpell_t}\right|\cdot \left(\dpu^\top \dpH_t^{\frac{1}{2}}\dps_t\right)^2\right] \tag{$\inner{\dpwtilde_t,\dpell_t}\le 1$}\\
    		&\le \mathbb{E}_t\left[2 \left|\inner{\dpu,\dpell_t}\right|+2d^2 \left|\inner{\dpwtilde_t, \dpell_t}\right|\cdot\|\dpu\|_{\dpH_t}^2\right].
    	\end{align*}
    Plugging these bounds to \pref{eq:varX}, we have
    $$
    \E_t[X_t^2]\le3204d^2\nu\mathbb{E}_t\left[ \left|\inner{\dpwtilde_t,\dpell_t}\right|\right]+4 \left|\inner{\dpu,\dpell_t}\right|+4d^2\mathbb{E}_t\left[ \left|\inner{\dpwtilde_t, \dpell_t}\right|\right]\|\dpu\|_{\dpH_t}^2.
    $$
    Summing over $t$ gives
    \begin{align*}
    \sum^T_{t=1}\E_t[X_t^2]&\le 3204d^2\sum^T_{t=1}\left(\nu+\|\dpu\|_{\dpH_t}^2\right)\mathbb{E}_t\left[ \left|\inner{\dpwtilde_t,\dpell_t}\right|\right]+4\sum^T_{t=1} \left|\inner{\dpu,\dpell_t}\right|\\
    &\le 3204d^2\left(\nu+\max_{t\in [T]}\|\dpu\|_{\dpH_t}^2\right)\sum^T_{t=1}\mathbb{E}_t\left[ \left|\inner{\dpwtilde_t,\dpell_t}\right|\right]+4\sum^T_{t=1} \left|\inner{\dpu,\dpell_t}\right|\\
    &= 3204d^2\left(\nu+\Ustar^2\right)\LTbarep+4\Lubar.
    \end{align*}   
	Therefore, choosing $\Bstar=32d(\sqrt{\nu}+\rho)$, $b=2\times 10^6d\nu^2 T$, $C=\ceil{\log_2 b}\ceil{\log_2 b^2T}=\Theta(\ln^2(d\nu T))$ and using \pref{thm:Freedman}, we obtain with probability $1-\delta$,
	\begin{align*}
		\sum^T_{t=1}X_t &=\sum_{t=1}^T\left(\inner{\dpwtilde_t, \dpell_t}-\inner{\dpw_t, \dpellhat_t}+\inner{\dpu,\dpellhat_t-\dpell_t}\right)\\
		&\le C\sqrt{25632d^2\left(\nu+\Ustar^2\right)\LTbarep\ln(C/\delta)+32\Lubar\ln(C/\delta)}+ 64Cd\left(\sqrt{\nu}+\Ustar\right)\ln(C/\delta).
	\end{align*}
    Finally, using $\sqrt{a+b}\le\sqrt{a}+\sqrt{b}$, the first term above is bounded by
    	\begin{align*}
    		161C\sqrt{d^2\left(\nu+\Ustar^2\right)\LTbarep\ln(C/\delta)}+C\sqrt{32\Lubar\ln(C/\delta)},
    	\end{align*}
    which finishes the proof.
\end{proof}


\subsubsection{Bounding \regterm}\label{sec:LB_REG}

The goal of this section is to prove the following bound on \regterm.

\begin{lemma}\label{lem:LB-REG} 
Let $\calS$ be its final value after running \pref{alg:linear bandit d+1} for $T$ rounds and $\calS' = \calS \setminus\{1,T+1\}$.
Then as long as  $\eta\le \frac{1}{80d}$, we have
	\begin{align*}
		\text{\regterm}\le \mathcal{\tilde{O}}\left(\frac{\nu}{\eta}\right)-\frac{\sum_{s\in \calS'}\|\dpu\|_{\dpH_s}}{5\eta\Aconst d\ln(\nu T)}+40\eta d^2\LTbar.
	\end{align*}
	for $\Aconst = 100$.
\end{lemma}

To prove this lemma, we first prove three useful lemmas. The first one shows that the number of times \pref{alg:linear bandit d+1} increases the learning rate is upper bounded by $\mathcal{O}(d\log_2(d\nu T))$.

\begin{lemma}\label{lem:LB-increase}
	Assume that $T\ge 8$. Let $n$ be the number of times \pref{alg:linear bandit d+1} increases the learning rate. Then $n\le \Aconst d\log_2(\nu T)$ for $\Aconst =100$.
    Consequently, we have $\eta_t\le 5\eta$ for all $t\in[T]$.
\end{lemma}

\begin{proof}
Let $\calS=\{t_1,\dots,t_{n+1}\}$ be its final value after running \pref{alg:linear bandit d+1} for $T$ rounds, which means $n$ is the number of times the algorithm has increased the learning rate, $t_1 = 1$, and for $i = 2, \ldots, n+1$, $\eta_{t_i} = \eta_{t_i - 1}\kappa$ holds.
    Let $\dpA_i=\sum_{j=1}^{i}\dpH_{t_j}$. Then for any $i>1$, according to the update rule, there exists a vector $p\in \mathbb{R}^{d+1}$ such that $p^\top \dpH_{t_i} p\ge  p^\top \dpA_{i-1} p$ and thus $p^\top \dpA_i p\ge 2 p^\top \dpA_{i-1} p$.
	Since a self-concordant function is strictly convex, $\dpA_i$ is positive definite for all $i\in [n]$.
    Therefore, let $q=\dpA_{i-1}^{\frac{1}{2}}p$ and we have $q^\top \dpA_{i-1}^{-\frac{1}{2}}\dpA_i\dpA_{i-1}^{-\frac{1}{2}}q\geq2\|q\|_2^2$.
    This implies that the largest eigenvalue of $\dpA_{i-1}^{-\frac{1}{2}}\dpA_i\dpA_{i-1}^{-\frac{1}{2}}$ is at least $2$.
    Furthermore, the smallest eigenvalue of $\dpA_{i-1}^{-\frac{1}{2}}\dpA_i\dpA_{i-1}^{-\frac{1}{2}}$ is at least $1$ since
	\begin{align*}
	\dpA_{i-1}^{-\frac{1}{2}}\dpA_i\dpA_{i-1}^{-\frac{1}{2}} = \dpA_{i-1}^{-\frac{1}{2}}\left(\dpA_{i-1}+\dpH_{t_i}\right)\dpA_{i-1}^{-\frac{1}{2}} = I+\dpA_{i-1}^{-\frac{1}{2}}\dpH_{t_i}\dpA_{i-1}^{-\frac{1}{2}}\succeq I.
	\end{align*}
	Therefore, we have
	\begin{align*}
	2\le\det(\dpA_{i-1}^{-\frac{1}{2}}\dpA_i\dpA_{i-1}^{-\frac{1}{2}})=\frac{\det(\dpA_i)}{\det(\dpA_{i-1})},
	\end{align*}
	which implies that $\det(\dpA_{n+1})\ge 2^{n}\det(\dpA_1)$. 
	
	Next we show an upper bound for $\frac{\det(\dpA_{n+1})}{\det(\dpA_1)}$. Consider any $(d+1)$-dimensional unit vector $\dpr$. 
	For each $i\in [n+1]$, applying \pref{lem:MFprpty} with $h=\dpH_1^{-\frac{1}{2}}\dpr$, $u=\dpw_{t_i}$ and $w=\dpw_1$, we have ,
	\begin{align*}
		\|h\|_{\dpH_{t_i}}^2 = \dpr^\top \dpH_1^{-\frac{1}{2}}\dpH_{t_i}\dpH_1^{-\frac{1}{2}}\dpr \le \left(\frac{1+2400\nu}{1-\pi_{\dpw_1}(\dpw_{t_i})}\right)^2\|h\|_{\dpH_1}^2\le (1+2400\nu)^2T^2.
	\end{align*} 
	Taking a summation over all $i\in[n+1]$, we obtain
	\begin{align*}
	\dpr^\top \dpA_1^{-\frac{1}{2}}\dpA_{n+1}\dpA_{1}^{-\frac{1}{2}}\dpr\le (n+1)(1+2400\nu)^2T^2,
	\end{align*}
	which means that 
    \begin{align*}
    \lambdamax\left(\dpA_1^{-\frac{1}{2}}\dpA_{n+1}\dpA_{1}^{-\frac{1}{2}}\right)\le (n+1)(1+2400\nu)^2T^2,
    \end{align*}
    and thus
    \[
    \frac{\det(\dpA_{n+1})}{\det(\dpA_1)} =  \det \left(\dpA_1^{-\frac{1}{2}}\dpA_{n+1}\dpA_{1}^{-\frac{1}{2}}\right)\le \left((n+1)(1+2400\nu)^2T^2\right)^{d+1}.
     \]
    Combining with $\frac{\det(\dpA_{n+1})}{\det(\dpA_1)}\ge 2^{n}$, we have
	\begin{align*}
	n\le (d+1)\log_2 (n+1) +2(d+1)\log_2 \left((1+2400\nu)T\right)
	\leq \Aconst d\log_2(\nu T),
	\end{align*}
	for $\Aconst = 100$.
    To show that $\eta_t\le 5\eta$ for $t$, notice that $\exp({\log_2 (\nu T)}/{\ln (\nu T)})\le 5$. Therefore,
    \[
    \eta_t\le \kappa^n \eta =  \exp\left(\frac{n}{\Aconst d\ln (\nu T)}\right)\eta\le 5\eta,
    \]
    finishing the proof.
\end{proof}

The second lemma gives a lower bound of the Bregman divergence between $\dpu$ and $\dpw_{t}$, which contains an important term to cancel \biasone in later analysis.

\begin{lemma}\label{lem:LB-Bregman}		For all $t\in [T]$, $D_{\Psi}(\dpu,\dpw_t)\ge -800\nu\ln\left(800\nu T\right)-800\nu+\|\dpu\|_{\dpH_t}.$
\end{lemma}
\begin{proof}
	Note again that $\Psi$ is a $800\nu$-normal barrier of $\dpOmega$ by \pref{lem:prop5-1-4}. By the definition of Bregman divergence, we have
	\begin{align}
		D_{\Psi}(\dpu,\dpw_t) &= \Psi(\dpu)-\Psi(\dpw_t)-\inner{\nabla\Psi(\dpw_t),\dpu-\dpw_t} \notag \\
		&\ge -800\nu\ln\frac{-\dpu^\top\nabla\Psi(\dpw_t)}{800\nu}-\inn{\nabla\Psi(\dpw_t),\dpu}-800\nu.\nonumber  \tag{\pref{lem:NBprpty} and \pref{lem:LB-uNorm}}
	\end{align}
	According to \pref{lem:MFprpty} and \pref{lem:LB-uNorm}, we know that
	\begin{align*}
		\left|\dpu^\top\nabla\Psi(\dpw_t)\right| 
		\leq \left(\frac{800\nu}{1-\pi_{\dpw_1}(\dpw_t)}\right)\|\dpu\|_{\dpH_1}
		\le 800\nu T\|\dpu\|_{\dpH_1}\le 640000\nu^2 T.
	\end{align*}
	
	On the other hand, according to \pref{lem:2-3-1}, we have
	\begin{align*}
		-\nabla\Psi(\dpw_t)^\top\dpu \ge \|\dpu\|_{\dpH_t}.
	\end{align*} Combining everything, we have
	\begin{align*}
		D_{\Psi}(\dpu,\dpw_t)\ge -800\nu\left(\ln(800\nu T)+1\right)+\|\dpu\|_{\dpH_t},
	\end{align*}
	finishing the proof.
\end{proof}

The third lemma gives a bound for the so-called stability term.
\begin{lemma}\label{lem:LB-stab}
	If $\eta\le\frac{1}{80d}$, then \pref{alg:linear bandit d+1} guarantees  $\|\dpw_t-\dpw_{t+1}\|_{\dpH_t}\le 40\eta\|\dpellhat_t\|_{\dpH_t^{-1}}$ for all $t\in [T]$.
\end{lemma}
\begin{proof}
	Let $F_t(\dpw)=\inner{\dpw,\dpellhat_t}+\frac{1}{\eta_t}D_\Psi(\dpw,\dpw_{t})$. 
	We have
	\begin{align}\label{eq:LB-stab-lower}
		F_{t}(\dpw_t)-F_{t}(\dpw_{t+1}) &= (\dpw_t-\dpw_{t+1})^\top\dpellhat_t - \frac{1}{\eta_t}D_\Psi(\dpw_{t+1},\dpw_{t}) \nonumber\\
		&\le (\dpw_t-\dpw_{t+1})^\top\dpellhat_t \le \|\dpw_t-\dpw_{t+1}\|_{\dpH_t}\cdot \|\dpellhat_t\|_{\dpH_t^{-1}},
	\end{align}
	where the last line uses the nonnegativity of Bregman divergence and also \Holder's inequality.
	On the other hand, by Taylor's theorem, there exists a point $\dpxi$ on the segment connecting $\dpw_t$ and $\dpw_{t+1}$ such that
	\begin{align}
		&F_{t}(\dpw_t)-F_{t}(\dpw_{t+1}) \nonumber\\
		&= \nabla F_{t}(\dpw_{t+1})^\top(\dpw_t-\dpw_{t+1})+\frac{1}{2}(\dpw_t-\dpw_{t+1})\nabla^2 F_{t}(\dpxi)(\dpw_t-\dpw_{t+1})\nonumber\\
		&\ge \frac{1}{2}(\dpw_t-\dpw_{t+1})\nabla^2 F_{t}(\dpxi)(\dpw_t-\dpw_{t+1}) \tag{by first order optimality of $\dpw_{t+1} = \argmin_{\dpw\in \dpOmega'}F_t(\dpw)$}\nonumber\\
		&=\frac{1}{2\eta_t}\|\dpw_t-\dpw_{t+1}\|_{\nabla^2 \Psi(\dpxi)}^2. \label{eq:LB-stab-w}
	\end{align}
	Next we will prove $\|\dpw_t-\dpw_{t+1}\|_{\nabla^2 \Psi(\dpxi)} \geq \frac{1}{2}\|\dpw_t-\dpw_{t+1}\|_{\dpH_t}$.
	To do so, we first show $\|\dpw_t-\dpw_{t+1}\|_{\dpH_{t}}\le \frac{1}{2}$.
    It is in turn sufficient to show
    \begin{align*}
    F_{t}(\dpw')\ge F_{t}(\dpw_t),~\text{for all $\dpw'$ such that $\|\dpw'-\dpw_t\|_{\dpH_{t}}=\frac{1}{2}$},
    \end{align*}
    	since $\dpw_{t+1}$ is the minimizer of the convex function $F_t$.
    Indeed, using Taylor's theorem again and denoting $\dpw'-\dpw_t$ by $\dph$, we have a point $\dpxi'$ on the segment between $\dpw'$ and $\dpw_t$ such that
    \begin{align*} 
    F_{t}\left(\dpw'\right) &=F_{t}\left(\dpw_{t}\right)+\nabla F_{t}\left(\dpw_{t}\right)^{\top}\dph+\frac{1}{2} \dph^{\top} \nabla^{2} F_{t}(\dpxi') \dph \\ &=F_{t}\left(\dpw_{t}\right)+{\dpellhat}_{t}^{\top} \dph +\frac{1}{2 \eta_{t}}\|\dph\|_{\nabla^2 \Psi(\dpxi')}^{2} \\ 
    & \geq F_{t}\left(\dpw_{t}\right)+{\dpellhat}_{t}^{\top}\dph +\frac{1}{2 \eta_{t}}\|\dph\|_{\dpH_t}^{2}\left(1-\left\|\dpw_{t}-\dpxi'\right\|_{\dpH_{t}}\right)^{2} \tag{\pref{lem:Dikin}}\\
    &\geq F_{t}\left(\dpw_{t}\right)+{\dpellhat}_{t}^{\top}\dph+\frac{1}{160 \eta} \tag{$\|\dph\|_{\dpH_{t}}=\frac{1}{2}$, $\left\|\dpw_{t}-\dpxi'\right\|_{\dpH_{t}}\leq \frac{1}{2}$, and \pref{lem:LB-increase}}\\
    &\geq F_{t}\left(\dpw_{t}\right)-\|{\dpellhat}_{t}\|_{\dpH_{t}^{-1}}\|\dph\|_{\dpH_{t}}+\frac{1}{160 \eta} \tag{\Holder's inequality}\\
    &\geq F_{t}\left(\dpw_{t}\right)-\frac{d}{2}+\frac{1}{160 \eta}. \tag{$\|\dpellhat_t\|_{\dpH^{-1}_t}\le d\left|\inner{\wtilde_t, \ell_t}\right|\le d$}
    \end{align*}
    
    
    Under the condition  $\eta\le \frac{1}{80d}$ we have thus shown $F_{t}(\dpw')\ge F_{t}(\dpw_t)$
     and consequently $\|\dpw_t-\dpw_{t+1}\|_{\dpH_{t}}\le \frac{1}{2}$ and $\|\dpw_t-\dpxi\|_{\dpH_t}  \leq \frac{1}{2}$.
     Now according to \pref{lem:Dikin} again, we have
    \begin{align*}
     \|\dpw_t-\dpw_{t+1}\|_{\nabla^2 \Psi(\dpxi)}\ge \|\dpw_t-\dpw_{t+1}\|_{\dpH_t}(1-\|\dpw_t-\dpxi\|_{\dpH_t})\ge \frac{1}{2}\|\dpw_t-\dpw_{t+1}\|_{\dpH_t}.
    \end{align*}
    Plugging it into \pref{eq:LB-stab-w} and combining \pref{eq:LB-stab-lower} give
    \begin{align*}
    \|\dpellhat_t\|_{\dpH_t^{-1}}\ge\frac{1}{8\eta_t}\|\dpw_t-\dpw_{t+1}\|_{\dpH_t}\ge \frac{1}{40\eta}\|\dpw_t-\dpw_{t+1}\|_{\dpH_t},
    \end{align*} 
    where the last inequality uses \pref{lem:LB-increase}.
    Rearranging finishes the proof.
\end{proof}

Now we are ready to prove the bound for $\regterm$ stated in \pref{lem:LB-REG}.

\begin{proof}[\textbf{Proof of \pref{lem:LB-REG}}]
	We first verify that $\dpu$ is in $\Omega'$. Indeed, according to the definition of $\dpu$, we have
	\begin{align*}
	\dpw_1+\frac{1}{1-\frac{1}{T}}\cdot(\dpu-\dpw_1) = \dpu^\star \in \Omega,
	\end{align*}
	which by the definition of Minkowsky function shows that $\pi_{\dpw_1}(\dpu)\leq 1-1/T$ and thus $\dpu \in \Omega'$.
	According to the standard analysis of Online Mirror Descent, for example, Lemma~6 of \citep{wei2018more}, we then have
	\begin{align}
		\inner{\dpw_t,\dpellhat_t}-\inner{\dpu,\dpellhat_t} &\le D_{\Psi_t}(\dpu, \dpw_t)-D_{\Psi_t}(\dpu, \dpw_{t+1})+\inner{\dpw_t-\dpw_{t+1}, \dpellhat_t}.\label{eq:LB-OMD}
	\end{align}
	We first focus on the term $D_{\Psi_t}(\dpu, \dpw_t)-D_{\Psi_t}(\dpu, \dpw_{t+1})$. Taking a summation over $t=1,2,\dots,T$, we have
	\begin{align}
		\sum_{t=1}^TD_{\Psi_t}(\dpu,\dpw_t)-D_{\Psi_t}(\dpu,\dpw_{t+1}) &\le D_{\Psi_1}(\dpu,\dpw_1)+\sum_{t=1}^{T-1}\left(D_{\Psi_{t+1}}(\dpu,\dpw_{t+1})-D_{\Psi_t}(\dpu,\dpw_{t+1})\right) \nonumber\\
		&\leq D_{\Psi_1}(\dpu,\dpw_1) +\sum_{i=2}^n\left(\frac{1}{\eta_{t_i}}-\frac{1}{\eta_{t_i-1}}\right)D_{\Psi}(\dpu,\dpw_{t_i}),\nonumber
    \end{align}
    where we recall the definition of $t_1, \ldots, t_n$ defined in the beginning of the proof of \pref{lem:LB-increase}.
    The first term can be bounded by
    \begin{align*}
    D_{\Psi_1}(\dpu,\dpw_1)=\frac{1}{\eta}D_{\Psi}(\dpu,\dpw_1)&=\frac{\Psi(\dpu)-\Psi(\dpw_1)}{\eta}-\frac{1}{\eta}\cdot\inner{\nabla\Psi(\dpw_1),\dpu-\dpw_1} \\
    &\le \frac{\Psi(\dpu)-\Psi(\dpw_1)}{\eta} \tag{by first order optimality of $\dpw_1$}\\
    &\le\frac{ 800\nu\ln T}{\eta}.\tag{\pref{lem:MFprpty}}
    \end{align*}
    For the second term, using $1-\kappa\le-\frac{1}{\Aconst d\ln(\nu T)}$ for $\Aconst=100$ and \pref{lem:LB-increase}, we have
    \[
    \frac{1}{\eta_{t_i}}-\frac{1}{\eta_{t_i-1}} \le\frac{1-\kappa}{\eta_{t_i}}\le-\frac{1}{5\eta \Aconst d\ln(\nu T)}.
    \]
    Therefore,
    \begin{align*}
		&\sum_{t=1}^TD_{\Psi_t}(\dpu,\dpw_t)-D_{\Psi_t}(\dpu,\dpw_{t+1}) \\
        &\le \frac{ 800\nu\ln T}{\eta} -\sum_{i=2}^n\frac{1}{5\eta \Aconst d\ln(\nu T)}\cdot D_{\Psi}(\dpu,\dpw_{t_i})\\
        &\le \mathcal{\tilde{O}}\left(\frac{\nu}{\eta}\right)-\frac{1}{5\eta \Aconst d\ln(\nu T)}\cdot \sum_{i=2}^n\left( \|\dpu\|_{\dpH_{t_i}}-800\nu-800\nu\ln(800\nu T) \right) \nonumber\tag{\pref{lem:LB-Bregman}} \\
		&= \mathcal{\tilde{O}}\left(\frac{\nu}{\eta}\right)-\frac{1}{5\eta \Aconst d\ln(\nu T)}\sum_{i=2}^n\|\dpu\|_{\dpH_{t_i}} \label{eq:LB-reg}.
	\end{align*}
	For the second term in \pref{eq:LB-OMD}, that is, $\inner{\dpw_t-\dpw_{t+1}, \dpellhat_t}$, taking summation over $t\in[T]$ we have
	\begin{align*}
		\sum_{t=1}^T\inner{\dpw_t-\dpw_{t+1}, \dpellhat_t} &\le \sum_{t=1}^T\|\dpw_t-\dpw_{t+1}\|_{\dpH_t}\|\dpellhat_t\|_{\dpH_t^{-1}} \tag{H\"{o}lder's inequality}\\
		&\le 40\eta\sum_{t=1}^T \|\dpellhat_t\|_{\dpH_t^{-1}}^2 \tag{\pref{lem:LB-stab}}\\
		&= 40\eta\sum_{t=1}^T d^2\inner{\wtilde_t, \ell_t}^2 \dps_t^\top \dpH_t^{1/2}\dpH_t^{-1}\dpH_t^{1/2}\dps_t \\
		&\le 40\eta\sum_{t=1}^Td^2\left|\inner{\wtilde_t, \ell_t}\right| = 40\eta d^2\LTbar.
	\end{align*}
Combining everything finishes the proof.
\end{proof}

\subsubsection{\pfref{thm:LB}}\label{sec:thm3.1}
To prove \pref{thm:LB}, we first prove the following main lemma.
\begin{lemma}\label{lem:LB-main}
	\pref{alg:linear bandit d+1} with $\eta\le \frac{1}{640\Aconst Cd^2\ln(\nu T)\ln(C/\delta)}$ guarantees that with probability at least $1-\delta$, 
	\begin{align*}
    &\sum_{t=1}^T\inner{\wtilde_t-u^\star, \ell_t}\\
		&\le\mathcal{\tilde{O}}\left(\frac{\nu}{\eta}+\eta d^2\LTbar+\sqrt{\Lubar\ln(1/\delta)}\right)+(\sqrt{\nu}+\Ustar)\left(161Cd\sqrt{\ln(C/\delta)\LTbarep}-\frac{1}{10\eta\Aconst d\ln(\nu T)}\right),
	\end{align*}
	where $\Aconst=100$, $C=\Theta(\ln^2(d\nu T))$ is defined in \pref{lem:LB-unbiased}, and we recall all other notations defined in Equations~\eqref{eq:LB_def1}-\eqref{eq:LB_def3}.
\end{lemma}

\begin{proof}
	Recall the decomposition of regret shown in~\pref{eq:decomp}.
    Combining the result of \pref{lem:LB-BIAS-1} and \pref{lem:LB-REG}, we have when $\eta\le\frac{1}{80d}$,
	\begin{align}
		\sum_{t=1}^T\inner{\wtilde_t-u^\star, \ell_t}&\le \mathcal{\tilde{O}}\left(\frac{\nu}{\eta}\right) -\frac{\sum_{s\in \calS'}\|\dpu\|_{\dpH_s}}{5\eta\Aconst d\ln(\nu T)}+40\eta d^2\LTbar + 64Cd\left(\sqrt{\nu}+\Ustar\right)\ln(C/\delta) \notag\\
		&\quad +161Cd\sqrt{\left(\nu+\Ustar^2\right)\LTbarep\ln(C/\delta)}+C\sqrt{32\Lubar\ln(C/\delta)}.\notag\\
        &=\otil\left(\frac{\nu}{\eta}+\eta d^2\LTbar+\sqrt{\Lubar\ln(C/\delta)}\right)-\frac{\sum_{s\in \calS'}\|\dpu\|_{\dpH_s}}{5\eta\Aconst d\ln(\nu T)}\notag\\
        &\quad+64Cd\left(\sqrt{\nu}+\Ustar\right)\ln(C/\delta)+161C\sqrt{d^2\left(\nu+\Ustar^2\right)\LTbarep\ln(C/\delta)}. \label{eq:intermediate}
	\end{align}
    Now consider the value of $\rho=\|\dpu\|_{\dpH_{t^\star}}$ where $t^\star \in \argmax_{t\in [T]}\|\dpu\|_{\dpH_{t}}$, compared to the negative term above. 
    Suppose $t^\star \in \calS$, then we have
    \[
    \rho \leq \max\cbr{\|\dpu\|_{\dpH_1}, \sum_{s\in \calS'}\|\dpu\|_{\dpH_s}}
    \leq 800\nu + \sum_{s\in \calS'}\|\dpu\|_{\dpH_s},
    \]
    where we use \pref{lem:LB-uNorm} again to bound $\|\dpu\|_{\dpH_1}$.
    On the other hand, if $t^\star \notin \calS$, then according to the update rule of $\calS$ in \pref{alg:linear bandit d+1}, we have $\dpH_{t^\star}\preceq \dpH_{1} + \sum_{s\in \calS'}\dpH_{s}$, which means 
	\[
	\rho = \sqrt{\|\dpu\|_{\dpH_{t^\star}}^2}\le   \sqrt{\|\dpu\|_{\dpH_1}^2 + \sum_{s\in \calS'}\|\dpu\|_{\dpH_s}^2}\le 800\nu +\sum_{s\in \calS'}\|\dpu\|_{\dpH_s}.
	\] 
	Therefore, we continue to bound the last three terms in \pref{eq:intermediate} as
	\begin{align*}
		 &\frac{800\nu - \rho}{5\eta\Aconst d\ln(\nu T)}+ 64Cd\left(\sqrt{\nu}+\Ustar\right)\ln(C/\delta)+ 161Cd\sqrt{\left(\nu+\Ustar^2\right)\LTbarep\ln(C/\delta)}\\
		&\le \order\left(\frac{\nu}{\eta}\right) -\frac{\sqrt{\nu}+\Ustar}{5\eta\Aconst d\ln(\nu T)}+ 64Cd\left(\sqrt{\nu}+\Ustar\right)\ln(C/\delta)+161Cd\sqrt{\left(\nu+\Ustar^2\right)\LTbarep\ln(C/\delta)}\\
		&\le \order\left(\frac{\nu}{\eta}\right) -\frac{\sqrt{\nu}+\Ustar}{10\eta\Aconst d\ln(\nu T)}+161Cd\sqrt{\left(\nu+\Ustar^2\right)\LTbarep\ln(C/\delta)}\tag{$\eta\le \frac{1}{640\Aconst Cd^2\ln(\nu T)\ln(C/\delta)}$}\\
		&\le \order\left(\frac{\nu}{\eta}\right)+(\sqrt{\nu}+\Ustar)\left(161Cd\sqrt{\ln(C/\delta)\LTbarep}-\frac{1}{10\eta\Aconst d\ln(\nu T)}\right).
	\end{align*}
	Plugging this back into \pref{eq:intermediate} finishes the proof.
\end{proof}

Now we are ready to prove the main theorem. For convenience, we restate the theorem below.
\begin{theorem}
	\pref{alg:linear bandit d+1} with an appropriate choice of $\eta$ ensures that with probability at least $1-\delta$:
	\[
	\Reg = \begin{cases}
	\otil(d^2\nu\sqrt{ T\ln\frac{1}{\delta}}+d^2\nu\ln\frac{1}{\delta}), &\text{against an oblivious adversary;} \\
	\otil(d^2\nu\sqrt{d T\ln\frac{1}{\delta}}+d^3\nu\ln\frac{1}{\delta}), &\text{against an adaptive adversary.}
	\end{cases}
	\]
    Moreover, if $\inner{w, \ell_t} \geq 0$ for all $w\in\Omega$ and all $t$, then $T$ in the bounds above can be replaced by $L^\star = \min_{u\in\Omega}\sum_{t=1}^T \inner{u, \ell_t}$, that is, the total loss of the best action.
\end{theorem}
\begin{proof}
	Using \pref{lem:LB-main} and the fact that $\left|\inner{\wtilde_t,\ell_t}\right|\le 1$ and $ \left|\inner{u,\ell_t}\right|\le 1$ for all $t\in [T]$, we have
	\begin{align*}		&\sum_{t=1}^T\inner{\wtilde_t-u^\star,\ell_t}\le \mathcal{\tilde{O}}\left(\frac{\nu}{\eta}+\eta d^2 T+\sqrt{T\ln\frac{1}{\delta}}\right)+(\sqrt{\nu}+\Ustar)\left(161Cd\sqrt{T\ln(C/\delta)}-\frac{1}{10\eta\Aconst d\ln(\nu T)}\right).
	\end{align*}
	 With 
     \[\eta=\min\left\{\frac{1}{640\Aconst Cd^2\ln(\nu T)\ln (C/\delta)},\frac{1}{1610 \Aconst Cd^2\ln(\nu T)\sqrt{T\ln (C/\delta)}}\right\},\] 
     the last term becomes nonpositive, and we arrive at
	\begin{align}
		\sum_{t=1}^T\inner{\wtilde_t-u^\star,\ell_t}\le \mathcal{\tilde{O}}\left(d^2\nu\sqrt{T\ln\frac{1}{\delta}}+d^2\nu\ln\frac{1}{\delta}\right), \label{eq:LB_oblivious}
	\end{align}
	for any fixed $u^\star \in \Omega$, which completes the proof for the oblivious case.
    To obtain a bound for an adaptive adversary, 
    we discrete the feasible set $\Omega$ and then take a union bound.
    Specifically, define $B_{\Omega}$ as follows:
        \[
            \Bomega \triangleq \lceil\alpha\rceil\lceil\beta\rceil,\quad\alpha\triangleq \max_{w,w'\in \Omega}\norm{w-w'}_\infty,~\beta\triangleq  \max_{\ell\in \Omega^\circ} \|\ell\|_\infty,
        \] 
    where $\Omega^\circ\triangleq\{\ell:|\inner{w,\ell}|\le 1,~\forall w\in \Omega\}$ is the set of feasible loss vectors. 
Then we discretize $\Omega$ into a finite set $\overline{\Omega}$ of  $(B_{\Omega}T)^d$ points, such that for any $u^\star \in \Omega$,  there exists $\overline{u}\in \overline{\Omega}$, such that $\|\overline{u}-u^\star\|_\infty \le\frac{1}{\lceil\beta\rceil T}$.
    This means that 
    \[
    \left
    |\sum_{t=1}^T\inner{\overline{u}-u^\star, \ell_t}\right |\le\sum_{t=1}^T\frac{d}{\lceil\beta\rceil T}\cdot\max_i\ell_{t,i} \le d.
    \]
    Therefore, it suffices to only consider regret against the points in $\overline{\Omega}$.
    Taking a union bound and replacing $\delta$ with $\frac{\delta}{(B_{\Omega}T)^d}$ in  \pref{eq:LB_oblivious} finish the proof for the worst-case bound for adaptive adversaries.

	
	In the remaining of the proof, we show that if $\inner{w,\ell_t}\in [0,1]$ for all $w\in \Omega$ and $t\in [T]$, $T$ can be replaced by $L^\star$ in both bounds.
	As $\inner{w,\ell_t}$ is always positive, we have $\mathbb{E}_t\left[\left|\inner{\wtilde_t,\ell_t}\right|\right]=\mathbb{E}_t\left[\inner{\wtilde_t,\ell_t}\right]=\inner{w_t,\ell_t}$, $\Lubar=\sum_{t=1}^T\inner{u,\ell_t} \leq L^\star + 2$,  and $\LTbar = \LT =\sum_{t=1}^T\inner{\wtilde_t,\ell_t}$. Using standard Freedman's inequality, we have with probability at least $1-\delta$,
	\begin{align*}
	\LTbarep-\LTbar\le \frac{\LTbarep}{2}+3\ln({1}/{\delta}).
	\end{align*}
    Rearranging gives	
    \[
    \LTbarep\le 2\LT+6\ln({1}/{\delta}).
    \]
    Using \pref{lem:LB-main} again, we have
	\begin{align*}
	&\sum_{t=1}^T\inner{\wtilde_t-u^\star,\ell_t} 
    \le \mathcal{\tilde{O}}\left(\frac{\nu}{\eta}+\eta d^2\LT+\sqrt{L^\star \ln\frac{1}{\delta}}\right)\\
    &\quad+(\sqrt{\nu}+\Ustar)\left(161Cd\sqrt{\ln(C/\delta)\left(2\LT+6\ln\frac{1}{\delta}\right)}-\frac{1}{10\eta\Aconst d\ln(\nu T)}\right). 
	\end{align*}
	
	With $\eta = \min\left\{\frac{1}{640\Aconst Cd^2\ln(\nu T)\ln (C/\delta)}, \frac{1}{1610\Aconst Cd^2\ln(\nu T)\sqrt{(2\LT+6\ln(1/\delta))\ln(C/\delta)}}\right\}$, the last term becomes nonpositive, and we arrive at
	\begin{align*}
		\sum_{t=1}^T\inner{\wtilde_t-u^\star, \ell_t} &\le \otil\left(d^2\nu\sqrt{\LT \ln\frac{1}{\delta}}+\sqrt{L^\star\ln\frac{1}{\delta}}+d^2\nu\ln\frac{1}{\delta}\right)
	\end{align*}
	Solving the quadratic inequality in terms of $\sqrt{\LT}$ gives the following high probability regret bound
	\begin{align*}
		\sum_{t=1}^T\inner{\wtilde_t-u^\star, \ell_t}\le \mathcal{\tilde{O}}\left(d^2\nu\sqrt{L^\star\ln\frac{1}{\delta}}+d^2\nu\ln\frac{1}{\delta}\right).
	\end{align*}
	This finishes the proof for the case with oblivious adversaries, and the case with adaptive adversaries is again by taking a union bound as done earlier.
\end{proof}

\begin{remark}\emph{	
The tuning of $\eta$ in the proof above depends on the unknown quantity $\LT$.
In fact, the issue seems even more severe than that pointed out in \pref{rem:doubling_trick} because $\LT$ depends on the algorithm's behavior, which in turns depends on $\eta$ itself.
We point out that, however, this can again be addressed using a doubling trick, making the algorithm completely parameter-free.
	We omit the details but refer the reader to \citet[Algorithm~4]{LLZ20} for very similar ideas.
}
\end{remark}	

\section{Omitted details for Section~\ref{sec:MDP}}\label{app:MDP}

\subsection{Preliminary}\label{app:MDP-pre}
In this section, we introduce  the concept of \emph{occupancy measure} (used in previous works already; see~\citep{jin2019learning}), which helps reformulate adversarial MDP problems in a way very similar to adversarial MAB problems. 
For a state $x$, let $k(x)$ denote the index of the layer to which state $x$ belongs.
Given a policy $\pi$ and a transition function $P$, we define occupancy measure $w^{P,\pi}\in \mathbb{R}^{X\times A\times X}$ as follows:
\[w^{P,\pi}(x,a,x')=\mathbb{P}\left[x_k=x,a_k=a,x_{k+1}=x'|P,\pi\right],\]
where $k=k(x)$. In other words, $w^{P,\pi}(x,a,x')$ is the probability of visiting the triple $(x,a,x')$ if we execute policy $\pi$ in an MDP with transition function $P$. 

According to this definition, we have the following two properties for any occupancy measure $w$. 
First, based on the layered structure, we know that each layer is visited exactly once in each episode, which means for each $k=0,1,\dots, \Lyr$, we have
\begin{align}\label{eq:occu-cond1}
	\sum_{x\in X_k, a\in A, x'\in X_{k+1}}w(x,a,x')=1.
\end{align}

Second, the probability of entering one state when coming from the previous layer equals to the probability of leaving the state to the next layer. Therefore, for each $k=1,2,\dots, \Lyr-1$, we have
\begin{align}\label{eq:occu-cond2}
	\sum_{x'\in X_{k-1}, a\in A}w(x',a,x) = \sum_{x'\in X_{k+1}, a\in A}w(x,a,x'), 
\end{align}
for all $x\in X_k$. 

Moreover, the following lemma shows that if $w$ satisfies the above two properties, then $w$ is an occupancy measure with respect to some transition function $P^w$ and policy $\pi^w$. 

\begin{lemma}[Lemma~3.1 in \citep{pmlr-v97-rosenberg19a}]\label{lem:occu-cstrain} For any $w\in [0,1]^{|X|\times|A|\times|X|}$, it satisfies \pref{eq:occu-cond1} and \pref{eq:occu-cond2} if and only if it is a valid occupancy measure associated with the following induced transition function $P^w$ and policy $\pi^w$:
\begin{align*}
	P^w(x'|x,a)=\frac{w(x,a,x')}{\sum_{y\in X_{k(x)+1}}w(x,a,y)},~
	\pi^w(a|x)=\frac{\sum_{x'\in X_{k(x)+1}}w(x,a,x')}{\sum_{a'\in A}\sum_{x'\in X_{k(x)+1}}w(x,a',x')}.
\end{align*}
\end{lemma}
Following \citep{jin2019learning}, we denote by $\Delta$ the set of all valid occupancy measures.
For a fixed transition function, we denote by $\Delta(P)\subseteq \Delta$ the set of occupancy measures whose induced transition function $P^w$ is exactly $P$. 
In addition, we denote by $\Delta(\calP)\subseteq \Delta$ the set of occupancy measures whose induced transition function $P^w$ belongs to a set of transition functions $\calP$.
With a slightly abuse of notation, we define $w(x,a)=\sum_{x'\in X_{k(x)+1}}w(x,a,x')$ for all $x\ne x_\Lyr$ and $a\in A$. 
Using the notations introduced above, we know that the expected loss of using policy $\pi$ at round $t$ is exactly $\inner{w^{P,\pi},\ell_t} \triangleq \sum_{x,a} w^{P,\pi}(x,a) \ell_t(x,a)$.
Let $\pi_t$ be the policy chosen at round $t$. 
Then the total expected loss (with respect to randomness of the transition function) is $\sum_{t=1}^T\inner{w^{P,\pi_t}, \ell_t}$ and the total regret can be written as:
\begin{align}
	\Reg = \sum_{t=1}^T\ell_t(\pi_t)-\min_{\pi}\sum_{t=1}^T\ell_t(\pi)=\sum_{t=1}^T\inner{\Qt-\Qpistar, \ell_t}=\LT-\Lstar, \label{eq:MDP_reg}
\end{align}
where $\Qpistar = w^{P, \pi^\star}$ is the occupancy measure induced by the optimal policy $\pi^\star=\argmin_{\pi}\sum_{t=1}^T\ell_t(\pi)$, $\Qt = w^{P, \pi_t}$, $\LT\triangleq \sum_{t=1}^T\inner{\Qt, \ell_t}$, and $\Lstar\triangleq\sum_{t=1}^T\inner{\Qpistar, \ell_t}$. 
When the regret is written in this way, it is clear that the problem is very similar to MAB or linear bandits with $\Delta(P)$ being the decision set and $\ell_t$ parametrizing the linear loss function at time $t$.

\subsection{Algorithm for MDPs}\label{app:MDP-alg}

\setcounter{AlgoLine}{0}
\begin{algorithm}[!ht]
	\caption{Upper Occupancy Bound Log Barrier Policy Search}
	\label{alg:MDP}
	\textbf{Input:} state space $X$, action space $A$, learning rate $\eta$, and confidence parameter $\delta$.
	
	\textbf{Define:} $\kappa=e^{\frac{1}{7\ln T}}$, \textsc{Comp-UOB} is Algorithm~3 of \citep{jin2019learning}, and
    \[
    \Omega = \left\{\hat{\QQ}:\hat{\QQ}(x,a,x')\geq \frac{1}{T^3|X|^2|A|},\forall k\in\{0,1,\dots,\Lyr-1\},x\in X_k,a\in A,x'\in X_{k+1}\right\}.
    \]
	\textbf{Initialization:} 
	Set epoch index $i=1$ and confidence set $\calP_1$ as the set of all transition functions. 
	For all $k=0, \ldots, \Lyr-1, (x,a,x')\in X_k\times A\times X_{k+1}$, set
	\[
	\Qhat_1(x,a,x') = \frac{1}{|X_k||A||X_{k+1}|}, \quad \pi_1 = \pi^{\Qhat_1}, \quad \eta_1(x,a)=\eta, \quad \rho_1(x,a)=2|X_k||A|,
	\]
	\[
	 \phi_1(x,a) = \textsc{Comp-UOB}(\pi_{1}, x, a, \calP_{1}), \quad N_0(x,a)=N_1(x,a) =  G_0(x'|x,a)=G_1(x'|x,a)=0.
	 \]
	 
	\nl \For{ $t = 1 ,2,\dots, T$ }{
		
		\nl Execute policy $\pi_t$ for $\Lyr$ steps and obtain trajectory $x_k, a_k, \ell_t(x_k, a_k)$ for $k = 0, \ldots, \Lyr-1$.
		
		\nl Construct loss estimators for all $(x,a)\in X\times A$:
		\begin{equation}\label{eq:MDP_estimator}
		\ellhat_t(x,a) = \frac{\ell_t(x, a)}{\Ut(x,a)}\ind_t(x,a), \quad\text{where}\;\; \ind_t(x,a) = \ind\{x_{k(x)} = x, a_{k(x)} = a\}.
		\end{equation}
		
		\nl Update counters: for each $k=0,1,\dots,\Lyr-1$,
		\begin{align*}
		N_i(x_k, a_k) &\gets N_i(x_k, a_k) + 1, \quad G_i(x_{k+1} | x_k, a_k) \gets G_i(x_{k+1} | x_k, a_k)  + 1.
		\end{align*}

		\nl \If {$\exists k, \  N_i(x_k, a_k) \geq \max\{1, 2N_{i-1}(x_k,a_k)\}$}{
			
			\nl Increase epoch index $i \gets i+1$.
			
			\nl Initialize new counters: $N_i = N_{i-1}, G_i = G_{i-1}$ (copy all entries). 
			
			\nl Compute confidence set
			\begin{equation*}
			\begin{split}
			\mathcal{P}_i = &\big\{\hat{P}:\left|\hat{P}(x'|x,a)-\bar{P}_i(x'|x,a)\right|\le \epsilon_i(x'|x,a), \\
			&\quad\forall (x,a,x')\in X_k\times A\times X_{k+1}, k=0,1,\dots, \Lyr-1\big\},
			\end{split}
			\end{equation*}
			where $\bar{P}_i(x'|x,a)=\frac{G_i(x'|x,a)}{\max\left\{1,N_i(x,a)\right\}}$
            and
			\begin{align*}
			\epsilon_i(x'|x,a)\triangleq 4\sqrt{\frac{\bar{P}_i(x'|x,a)\ln \left(\frac{T|X||A|}{\delta}\right)}{\max\{1,N_i(x,a)-1\}}}+\frac{28\ln \left(\frac{T|X||A|}{\delta}\right)}{3\max\{1, N_i(x,a)-1\}}. 
			\end{align*}
		}
		
		\nl Compute
		$
		\Qhat_{t+1} = \argmin_{\QQ\in\Delta(\calP_i)\cap \Omega}\left\{\langle \QQ, \ellhat_t \rangle + D_{\psi_t}(\QQ,\Qhat_{t})\right\},
		$
		where  
		\begin{equation}\label{eq:MDP_log_barrier}
			\psi_t(\QQ)=\sum_{k=0}^{\Lyr-1}\sum_{(x,a,x')\in X_k\times A\times X_{k+1}}\frac{1}{\eta_t(x,a)}\ln\left(\frac{1}{\QQ(x,a,x')}\right).
		\end{equation}
		
		\nl Update policy $\pi_{t+1} = \pi^{\Qhat_{t+1}}$.
		
		\nl \For{each $(x,a)\in X\times A$}{
			\nl Update upper occupancy bound:
			\begin{equation}\label{eq:upconf_def}
			\upconf_{t+1}(x,a) = \max_{\hat{P}\in \mathcal{P}_{i}}\QQ^{\hat{P},\pi_{t+1}}(x,a) = \textsc{Comp-UOB}(\pi_{t+1}, x, a, \calP_{i}).
			\end{equation}
			
			\nl \lIf{$\frac{1}{\upconf_{t+1}(x,a)} \geq \rho_t(x,a)$}{ $\rho_{t+1}(x,a)=\frac{2}{\upconf_{t+1}(x,a)}$, $\eta_{t+1}(x,a)=\eta_t(x,a)\cdot \kappa$.		
			}
			\nl \lElse{	$\rho_{t+1}(x,a)=\rho_t(x,a)$, $\eta_{t+1}(x,a)=\eta_t(x,a)$.
			}
		}
	}
\end{algorithm}

In this section, we introduce our algorithm that achieves high-probability small-loss regret bound for the MDP setting. The full pseudocode of the algorithm is shown in \pref{alg:MDP}.
The algorithm is very similar to \UOB introduced in \citep{jin2019learning}, except for the following two modifications. 

First, in~\citep{jin2019learning}, they propose a loss estimator akin to the importance-weighted estimator using the so-called {\it upper occupancy bound}, denoted by $\phi_t(x,a)$ in our notation.
Indeed, the actual probability $w_t(x,a)$ of visiting state-action pair $(x,a)$ is unknown (due to the unknown transition), and thus standard unbiased importance-weighted estimators do not apply directly.
Instead, since the algorithm maintains a confidence set $\calP_i$ (for epoch $i$) of all the plausible transition functions based on observations, one can calculate the largest probability of visiting state-action pair $(x,a)$ under policy $\pi_t$, among all the plausible transition functions,
which is exactly the definition of $\phi_t(x,a)$ and can be computed efficiently via the sub-routine \textsc{Comp-UOB} as shown in~\citep{jin2019learning}.
In addition, \citet{jin2019learning} also apply the idea of  implicit exploration from~\citep{neu2015explore} and introduce an extra bias with a parameter $\gamma > 0$, leading to the following loss estimator:
		\begin{equation*}
		\ellhat_t(x,a) = \frac{\ell_t(x, a)}{\Ut(x,a)+\gamma}\ind_t(x,a),
		\end{equation*}
which is crucial for them to derive a high-probability bound.
As one can see in \pref{eq:MDP_estimator}, the first difference of our algorithm is that we remove this implicit exploration (that is, $\gamma =0$), similarly to our MAB algorithm in \pref{sec:MAB}.
As we later explain in \pref{app:MDP_discussions}, removing this implicit exploration is important for obtaining a small-loss bound.

Second, while \UOB uses the entropy regularizer with a fixed learning rate, 
we use the log-barrier regularizer with time-varying and individual learning rates for each state-action pair, defined in \pref{eq:MDP_log_barrier},
which is a direct generalization of \pref{alg:mab} for MAB.
The way we increase the learning rate is also essentially identical to the MAB case; see the last part of \pref{alg:MDP}.
We also point out that the analogue of the clipped simplex used in \pref{alg:mab} is now $\Delta(\calP_i)\cap \Omega$ where $\Delta(\calP_i)$ is the set of occupancy measures with induced transition functions in the confidence set $\calP_i$, and $\Omega$ (defined at the beginning of \pref{alg:MDP}) contains all $\Qhat$ with each entry not smaller than $1/(T^3|X|^2|A|)$, which ensures that the learning rates cannot be increased by too many times.

\subsection{Proof of \pref{thm:MDP}}\label{app:MDP-thm}

In this section, we analyze \pref{alg:MDP} and prove \pref{thm:MDP}.
We start with decomposing the regret into five terms (recall the definitions of $\Qt$ and $\Qpistar$ in \pref{eq:MDP_reg} and $\Qhat_t$ and $\ellhat_t$ in \pref{alg:MDP}):
\begin{align*}
	\sum_{t=1}^T\inner{\Qt-\Qpistar, \ell_t} &= \underbrace{\sum_{t=1}^T\inner{\Qt-\Qhat_t, \ell_t}}_{\errorterm}
	 + \underbrace{\sum_{t=1}^T\inner{\Qhat_t, \ell_t-\ellhat_t}}_{\bias}
	 + \underbrace{\sum_{t=1}^T\inner{\Qhat_t - \Qstar, \ellhat_t}}_{\regterm}
	\\
    &\quad + \underbrace{\sum_{t=1}^T\inner{\Qstar, \ellhat_t-\ell_t}}_{\biastwo}+\underbrace{\sum_{t=1}^T\left<\Qstar-\Qpistar, \ell_t\right>}_{\biasthree}.
\end{align*}

Here $\Qstar$ is defined as
\begin{equation}\label{eq:MDP_u_def}
\Qstar=\left(1-\frac{1}{T}\right)\Qpistar+\frac{1}{T|A|}\sum_{a\in A}\QQ^{P_0, \pi_a},
\end{equation}
where $\pi_a$ is the policy that chooses action $a$ at every state, and the definition of the transition function $P_0$ is deferred to \pref{lem:p0xxa}.
Note that $\Qstar$ is random in the case with adaptive adversaries.

In the remaining of this subsection,
we first provide a few useful lemmas in \pref{app:MDP-lemmas},
and then bound \errorterm in \pref{app:MDP-error}, \bias in \pref{app:MDP-biasone}, \biastwo in \pref{app:MDP-biastwo}, and \regterm in \pref{app:MDP-regterm}.
Note that \biasthree~can be trivially bounded by $\Lyr$ as
	\begin{align}\label{eq:bias3}
    \biasthree = \sum_{t=1}^T\left\langle \Qstar-\Qpistar, \ell_t\right\rangle \le \frac{1}{T|A|}\sum_{a\in A}\sum_{t=1}^T\left\langle \QQ^{\pi_a,P_0},\ell_t\right\rangle \le \Lyr.
    \end{align}
We finally put everything together and prove \pref{thm:MDP} in \pref{app:MDP-together}.    
    
\subsubsection{Useful lemmas}\label{app:MDP-lemmas}    

The first two lemmas are from \citep{jin2019learning}.

\begin{lemma}[Lemma~2 in \citep{jin2019learning}]\label{lem: MDP-PinP}
	With probability at least $1-4\delta$, we have for all $k = 0,1,\dots, \Lyr-1$ and $ (x,a,x')\in X_k\times A\times X_{k+1}$,
	\begin{equation}\label{eq:MDP-PinP}
	\left|P(x'|x,a)-\bar{P}_i(x'|x,a)\right|\le \frac{\epsilon_i(x'|x,a)}{2}.
	\end{equation}
	Consequently, we have $P\in \calP_i$ for all $i$.
\end{lemma}

\begin{lemma}[Lemma~10 in \citep{jin2019learning}]\label{lem: L10-jin2019}
	With probability at least $1-\delta$, we have for all $k=0,\dots,\Lyr-1$,
	\begin{align*}
		\sum_{t=1}^T\sum_{x\in X_k, a\in A}\frac{\QQ_t(x,a)}{\max\{1,N_{i_t}(x,a)\}}&=\otil\left(|X_k|\cdot|A|+\ln\frac{1}{\delta}\right)\\
		\sum_{t=1}^T\sum_{x\in X_k, a\in A}\frac{\QQ_t(x,a)-\ind_t(x,a)}{\sqrt{\max\{1,N_{i_t}(x,a)\}}}&\le \sum_{t=1}^T\sum_{x\in X_k, a\in A}\frac{\QQ_t(x,a)}{\max\{1, N_{i_t}(x,a)\}}+\otil\left(\ln\frac{1}{\delta}\right) \nonumber\\
		&\le \otil\left(|X_k|\cdot|A|+\ln\frac{1}{\delta}\right),
	\end{align*}
	where $i_t$ is the index of the epoch to which episode $t$ belongs. 
\end{lemma}



Next, we prove a lemma showing that there exists a transition function $P_0$ that always lies in the confidence set $\calP_i$ of the algorithm, such that for any action $a\in A$ and any two states $x,x'$ in consecutive layers, the probability of reaching $x'$ by taking action $a$ at state $x$ is at least $\frac{1}{T|X|}$. 

\begin{lemma}\label{lem:p0xxa}
	With probability at least $1-4\delta$, there exists $P_0\in \cap_i \mathcal{P}_i$ such that for all $k < \Lyr, x\in X_k, a\in A$, and $x'\in X_{k+1}$, we have $P_0(x'|x,a)\geq \frac{1}{T|X|} $.
\end{lemma}

\begin{proof}
   The construction of $P_0$ is as follows.
   First we start with $P_0 = P$.
   Then for each fixed $(x,a)$, we focus on the distribution $P_0(\cdot |x,a)$.
   In particular, for all $x' \in X_{k(x)+1}$ such that $P_0(x' |x,a) < \frac{1}{T|X|}$, we move the weight from the largest entry of $P_0(\cdot |x,a)$ to this entry so that $P_0(x' |x,a) = \frac{1}{T|X|}$ and $P_0(\cdot |x,a)$ remains a valid distribution.
   Repeat the same for all $(x,a)$ pairs finishes the construction of $P_0$.
   
   Clearly, $P_0$ satisfies $P_0(x'|x,a)\geq \frac{1}{T|X|}$, and it remains to show $P_0 \in \calP_i$ for all $i$.
   To this end, we first note that 
   \begin{equation*}
	\left|P_0(x'|x,a)-P(x'|x,a)\right| \le \frac{|X_{k(x')}|}{T|X|} \leq \frac{1}{T} \leq \frac{\epsilon_i(x'|x,a)}{2}
	\end{equation*}
	holds  for all $k = 0,1,\dots, \Lyr-1$ and $ (x,a,x')\in X_k\times A\times X_{k+1}$.
	Combining this with \pref{eq:MDP-PinP} then shows that $\left|P_0(x'|x,a)-\bar{P}_i(x'|x,a)\right|  \leq \epsilon_i(x'|x,a)$, indicating that $P_0$ is indeed in $\calP_i$ by the definition of $\calP_i$.
\end{proof}

The next lemma shows that the upper occupancy bound for each state-action pair is lower bounded.

\begin{lemma}\label{lem:upper-occu-bound}
	We have $\Ut(x,a)\ge\frac{1}{T^3|X|^2|A|}$ for all $x\in X$ and $a\in A$.
\end{lemma}

\begin{proof}
   This is simply by the definition of $\Ut$ in \pref{eq:upconf_def}  and the definition of $\Omega$: 
$
	\Ut(x,a) \geq \Qhat_t(x, a) \ge \frac{1}{T^3|X|^2|A|}.
$
\end{proof}

The last lemma is an improvement of~\citep[Lemma~4]{jin2019learning} and is important for bounding \errorterm and \bias in terms of $\sqrt{L_T}$, as opposed to $\sqrt{T}$ (which is the case in~\citep{{jin2019learning}}).
    
\begin{lemma}\label{lem:MDP-small-loss-lemma}
With probability at least $1-\cprob\delta$, for any $t$ and any collection of
	transition functions $\left\{P_t^x\right\}_{x\in X}$ such that $P_t^x \in \calP_{i_t}$ for all $x$ (where $i_t$ is the index of the epoch to which episode $t$ belongs), we have
	\begin{align*}
		&\sum_{t=1}^T\sum_{x\in X,a\in A}\left|w^{P_t^x,\pi_t}(x,a)-w_t(x,a)\right| \ell_t(x,a)\\ &= \mathcal{\tilde{O}}\left(|X|\sqrt{\Lyr|A| \LT\ln\frac{1}{\delta}}+|X|^5|A|^2\ln \frac{1}{\delta}+|X|^4|A|\ln^2\frac{1}{\delta}\right).
 	\end{align*}
\end{lemma}
\begin{proof}
   The proof is technical but mostly follows the same ideas of that for~\citep[Lemma~4]{jin2019learning}.
	We first assume that the events of \pref{lem: MDP-PinP} and \pref{lem: L10-jin2019} hold, which happens with probability at least $1-5\delta$. According to the proof of \citep[Lemma~4]{jin2019learning} (specifically their Eq.~(15)), we have for any pair $(x,a)$,
	\begin{align}\label{eq: error-1}
	&\left|w^{P_t^x, \pi_t}(x,a)-\Qt(x,a)\right|\cdot \ell_t(x,a) \nonumber\\
	&\le \sum_{m=0}^{k(x)-1}\sum_{x_m,a_m,x_{m+1}}\epsilon_{i_t}^\star(x_{m+1}|x_m,a_m)\Qt(x_m,a_m)w^{P_t^x, \pi_t}(x,a|x_{m+1})\cdot \ell_t(x,a),
	\end{align}
    where $\epsilon_{i_t}^{\star}\left(x^{\prime} | x, a\right)= \mathcal{O}\left(\sqrt{\frac{P\left(x^{\prime} | x, a\right) \ln \left(\frac{T|X||A|}{\delta}\right)}{\max \left\{1, N_{i_t}(x, a)\right\}}}+\frac{\ln \left(\frac{T|X||A|}{\delta}\right)}{\max \left\{1, N_{i_t}(x, a)\right\}}\right)$, and for an occupancy measure $w$, $w(x,a|x')$ denotes the probability of encountering the pair $(x,a)$ given that $x'$ was visited earlier, under policy $\pi^w$ and $P^w$.
	By their Eq.~(16), we also have
	\begin{align}\label{eq: error-2}
	&|w^{P_t^x, \pi_t}(x,a|x_{m+1})-\Qt(x,a|x_{m+1})| \nonumber\\
	&\le \pi_t(a|x)\sum_{h=m+1}^{k(x)-1}\sum_{x_h',a_h',x_{h+1}'}\epsilon_{i_t}^\star(x_{h+1}'|x_h',a_h')\Qt(x_h',a_h'|x_{m+1}),
	\end{align}
    Combining \pref{eq: error-1} and \pref{eq: error-2}, summing over all $t$ and $(x,a)$ and using shorthands $z_m\triangleq(x_m, a_m, x_{m+1})$ and $z_h'\triangleq(x_h',a_h',x_{h+1}')$, we have
	
	\begin{align*}
	&\sum_{t=1}^T\sum_{x\in X,a\in A}|w^{P_t^x, \pi_t}(x,a)-\Qt(x,a)|\cdot \ell_t(x,a)\\
	&\le \sum_{t,x,a}\sum_{m=0}^{k(x)-1}\sum_{z_m}\epsilon_{i_t}^\star(x_{m+1}|x_m,a_m)\Qt(x_m,a_m)\Qt(x,a|x_{m+1})\ell_t(x,a)\\
	&\quad + \sum_{t,x,a}\sum_{m=0}^{k(x)-1}\sum_{z_m}\epsilon_{i_t}^\star(x_{m+1}|x)\Qt(x_m,a_m)\cdot \left(\pi_t(a|x)\sum_{h=m+1}^{k(x)-1}\sum_{z_h'}\epsilon_{i_t}^\star(x_{h+1}'|x_h',a_h')\Qt(x_h',a_h'|x_{m+1})\right) \tag{$\ell_t(x,a)\le 1$}\\
	&= \sum_{t}\sum_{k<\Lyr}\sum_{m=0}^{k-1}\sum_{z_m}\epsilon_{i_t}^\star(x_{m+1}|x_m,a_m)\Qt(x_m,a_m)\sum_{x\in X_k, a\in A}\Qt(x,a|x_{m+1})\ell_t(x,a)\\
	&\quad + \sum_{t}\sum_{0\le m<h<k<\Lyr}\sum_{z_m,z_h'}\epsilon_{i_t}^\star(x_{m+1}|x)\Qt(x_m,a_m)\epsilon_{i_t}^\star(x_{h+1}'|x_h',a_h')\Qt(x_h',a_h'|x_{m+1}) \cdot \left(\sum_{x\in X_k, a\in A}\pi_t(a|x)\right)\\
	&= \sum_{0\leq m<k<\Lyr}\sum_{t, z_m}\epsilon_{i_t}^\star(x_{m+1}|x_m,a_m)\Qt(x_m,a_m)\sum_{x\in X_k, a\in A}\Qt(x,a|x_{m+1})\ell_t(x,a)\\
		&\quad +  \sum_{0\le m<h<k<\Lyr}|X_k|\sum_{t, z_m,z_h'}\epsilon_{i_t}^\star(x_{m+1}|x)\Qt(x_m,a_m)\epsilon_{i_t}^\star(x_{h+1}'|x_h',a_h')\Qt(x_h',a_h'|x_{m+1})  \\
			&\leq \sum_{0\leq m<k<\Lyr}\sum_{t, z_m}\epsilon_{i_t}^\star(x_{m+1}|x_m,a_m)\Qt(x_m,a_m)\sum_{x\in X_k, a\in A}\Qt(x,a|x_{m+1})\ell_t(x,a)\\
		&\quad + |X| \sum_{0\le m<h<\Lyr}\sum_{t,z_m,z_h'}\epsilon_{i_t}^\star(x_{m+1}|x)\Qt(x_m,a_m)\epsilon_{i_t}^\star(x_{h+1}'|x_h',a_h')\Qt(x_h',a_h'|x_{m+1})  \\
	&\triangleq B_1+ |X| B_2.
	\end{align*}
	
	It remains to bound $B_1$ and $B_2$.
	First, $B_2$ is exactly the same as in the proof of~\citep[Lemma~4]{jin2019learning}.
	Below, we outline the proof with the dependence on all parameters explicit (indeed, this is hidden in their proof).
	First, according to their analysis, $B_2$ is bounded by 
	\begin{align*}
		&\otil\left(\sum_{0\le m<h<\Lyr}\sum_{t,z_m, z_h'}\sqrt{\frac{P(x_{m+1}|x_m,a_m)\ln\frac{1}{\delta}}{\max\{1, N_{i_t}(x_m,a_m)\}}}\QQ_t(x_m,a_m)\sqrt{\frac{P(x_{h+1}'|x_h',a_h')\ln\frac{1}{\delta}}{\max\{1, N_{i_t}(x_h',a_h')\}}}\QQ_t(x_h',a_h'|x_{m+1})\right) \\
		&+\otil\left(\sum_{0\le m<h<\Lyr}\sum_{t,z_m,z_h'}\frac{\QQ_t(x_m,a_m)\ln\frac{1}{\delta}}{\max\{1,N_{i_t}(x_m,a_m)\}}\right)+\otil\left(\sum_{0\le m<h<\Lyr}\sum_{t,z_m,z_h'}\frac{\QQ_t(x_h',a_h')\ln\frac{1}{\delta}}{\max\{1,N_{i_t}(x_h',a_h')\}}\right).
	\end{align*}
	
	They show that the first term is bounded by $\otil(|X|^2|A|\ln^2(1/\delta))$. For the second term, we have
	\begin{align*}
		&\sum_{0\le m<h<\Lyr}\sum_{t,z_m,z_h'}\frac{\QQ_t(x_m,a_m)\ln\frac{1}{\delta}}{\max\{1,N_{i_t}(x_m,a_m)\}} \\
		&\le \left(\sum_{h=0}^{J-1}|X_h|\cdot |A|\cdot|X_{h+1}|\ln\frac{1}{\delta}\right)\sum_{m=0}^{\Lyr-1}|X_{m+1}|\cdot \sum_{t,x\in X_m,a\in A}\frac{\QQ_t(x_m,a_m)}{\max\{1,N_{i_t}(x_m,a_m)\}}\\
		&\le \mathcal{O}\left(|X|^2|A|\ln\frac{1}{\delta}\right)\cdot \otil\left(|X|^2|A|+|X|\ln\frac{1}{\delta}\right) \tag{\pref{lem: L10-jin2019}}\\
		&\le \otil\left(|X|^4|A|^2\ln\frac{1}{\delta}+|X|^3|A|\ln\frac{1}{\delta}\right).
	\end{align*}
	The third term can be bounded in the exact same way. Therefore, we arrive at
	\begin{align}\label{eq: MDP-B2}
		|X|B_2\le \otil\left(|X|^5|A|^2\ln(1/\delta)+|X|^4|A|\ln^2(1/\delta)\right).
	\end{align}
	
	Next we show that $B_1$ is bounded by $\otil(|X|\sqrt{\Lyr|A|\LT\ln(1/\delta)}+|X|^3|A|\ln(1/\delta))$. According to the definition of $\epsilon_{i_t}^\star$, we have
	\begin{align}\label{eq: error-B1}
	B_1 &= \mathcal{O}\left(\sum_{0\le m<k<\Lyr}\sum_{t,z_m}\Qt(x_m,a_m)\left(\sum_{x\in X_k,a\in A}\Qt(x,a|x_{m+1})\ell_t(x,a)\right)\right.\nonumber \\
	&\quad\cdot\left.\sqrt{\frac{P(x_{m+1}|x_m,a_m)\ln \left(\frac{T|X||A|}{\delta}\right)}{\max\{1, N_{i_t}(x_m,a_m)\}}}\right)+ \mathcal{O}\left(\sum_{0\le m<k<\Lyr}\sum_{t,z_m}\frac{\Qt(x_m,a_m)\ln\left(\frac{T|X||A|}{\delta}\right)}{\max\{1,N_{i_t}(x_m,a_m)\}}\right).
	\end{align}
	
	According to \pref{lem: L10-jin2019},  the second term is bounded as
	\begin{align}\label{eq: error-3}
	\mathcal{O}\left(\sum_{0\le m<k<\Lyr}\sum_{t,z_m}\frac{\Qt(x_m,a_m)\ln\left(\frac{T|X||A|}{\delta}\right)}{\max\{1,N_{i_t}(x_m,a_m)\}}\right)\le \otil\left(\Lyr|X|^2|A|+\Lyr|X|\ln\frac{1}{\delta}\right).
	\end{align}
	
	In the following, we define $\ell_t(k|x,a)\triangleq\sum_{x_k\in X_k, a_k\in A}\ell_{t}(x_k,a_k) \Qt(x_k,a_k|x,a)$ where $\Qt(x',a'|x,a)$ is the probability of encountering pair $(x',a')$ given that pair $(x,a)$ was encountered earlier, under policy $\pi_t$ and transition $P$. 
	For the first term of \pref{eq: error-B1}, we then have 
	\begin{align}
	&\sum_{0\le m<k<\Lyr}\sum_{t,z_m}\Qt(x_m,a_m)\left(\sum_{x\in X_k,a\in A}\Qt(x,a|x_{m+1})\ell_t(x,a)\right)\cdot\sqrt{\frac{P(x_{m+1}|x_m,a_m)\ln \left(\frac{T|X||A|}{\delta}\right)}{\max\{1, N_{i_t}(x_m,a_m)\}}} \nonumber\\
	&\le \sum_{0\le m<k<\Lyr}\sum_{t,z_m}\Qt(x_m,a_m)\sqrt{\left(\sum_{x\in X_k,a\in A}\Qt(x,a|x_{m+1})\ell_t(x,a)\right)\cdot \frac{P(x_{m+1}|x_m,a_m)\ln \left(\frac{T|X||A|}{\delta}\right)}{\max\{1, N_{i_t}(x_m,a_m)\}}} \nonumber\\
	&\le \sum_{0\le m<k<\Lyr}\sum_{t,x_m,a_m}\Qt(x_m,a_m)\sqrt{|X_{m+1}|\cdot \frac{\ell_t(k|x_m,a_m)\ln \left(\frac{T|X||A|}{\delta}\right)}{\max\{1, N_{i_t}(x_m,a_m)\}}} \tag{Cauchy-Schwarz inequality} \\
	& \le \sum_{0\le m<k<\Lyr}\sqrt{|X_{m+1}|\ln\left(\frac{T|X||A|}{\delta}\right)} \nonumber\\
	&\quad \cdot\sum_{t,x_m,a_m}\left(\ind_t(x_m,a_m)\sqrt{\frac{\ell_t(k|x_m,a_m)}{\max\{1,N_{i_t}(x_m,a_m)\}}}+\frac{\Qt(x_m,a_m)-\ind_t(x_m,a_m)}{\sqrt{\max\{1,N_{i_t}(x_m,a_m)\}}}\right).\label{eq:bias1-term1}
	\end{align}
	
	According to \pref{lem: L10-jin2019} again, we have for all $m=0,1,\dots,\Lyr-1$, 
	\begin{align*}
	\sum_{t=1}^T\sum_{x_m,a_m}\frac{\Qt(x_m,a_m)-\ind_t(x_m,a_m)}{\sqrt{\max\{1, N_{i_t}(x_m,a_m)\}}} \le \otil \left(|X_m||A|+ \ln(1/\delta)\right).
	\end{align*}
	For the term $\sum_{t,x_m,a_m}\ind_t(x_m,a_m)\sqrt{\frac{\ell_t(k|x_m,a_m)}{\max\{1,N_{i_t}(x_m,a_m)\}}}$, using Cauchy-Schwarz inequality, we have
	\begin{align} 
	&\sum_{t,x_m,a_m}\ind_t(x_m,a_m)\sqrt{\frac{\ell_t(k|x_m,a_m)}{\max\{1,N_{i_t}(x_m,a_m)\}}} \nonumber\\
	&\le \sum_{x_m,a_m}\sqrt{\sum_{t=1}^T\frac{\ind_t(x_m,a_m)}{\max\{1,N_{i_t}(x_m,a_m)\}}}\cdot \sqrt{\sum_{t=1}^T\ind_t(x_m,a_m)\ell_t(k|x_m,a_m)} \tag{Cauchy-Schwarz inequality} \nonumber\\
	&\le \mathcal{O}\left(\sqrt{|X_m||A| \left(\sum_{t=1}^T\sum_{x\in X_m, a\in A}\ind_t(x_m,a_m)\ell_t(k|x,a)\right) \cdot \ln T}\right), \label{eq:bias1-term2}
	\end{align}
	where the last step uses Cauchy-Schwarz inequality again and the fact $\sum_{t=1}^T\frac{\ind_t(x_m,a_m)}{\max\{1,N_{i_t}(x_m,a_m)\}}\le \mathcal{O}(\ln T)$.
	Combining \pref{eq:bias1-term1} and \pref{eq:bias1-term2}, we have
	\begin{align}\label{eq: error-4}
	&\sum_{0\le m<k<\Lyr}\sum_{t,z_m}\Qt(x_m,a_m)\left(\sum_{x\in X_k,a\in A}\Qt(x,a|x_{m+1})\ell_t(x,a)\right)\cdot\sqrt{\frac{P(x_{m+1}|x_m,a_m)\ln \left(\frac{T|X||A|}{\delta}\right)}{\max\{1, N_{i_t}(x_m,a_m)\}}} \nonumber\\
	&\le \otil\left(\sum_{0\le m<k<\Lyr}\sqrt{|X_m||A||X_{m+1}|\ln \frac{1}{\delta}} \sqrt{\sum_{t=1}^T\sum_{x\in X_m, a\in A}\ind_t(x_m,a_m)\ell_t(k|x,a)}\right) \nonumber\\
	&\le \otil\left(\sum_{m=0}^{\Lyr-1}\sqrt{\Lyr|X_m||A||X_{m+1}|\sum_{t=1}^T\sum_{k>m}\sum_{x\in X_m, a\in A}\ind_t(x_m,a_m)\ell_t(k|x,a)\ln \frac{1}{\delta}}\right). \tag{Cauchy-Schwarz inequality}
	\end{align}
	Further note that 
  		\begin{align*}
  			&\mathbb{E}_t\left[\sum_{k>m}\sum_{x\in X_m, a\in A}\ind_t(x_m,a_m)\ell_t(k|x,a)\right]\\
  			&= \sum_{k>m} \sum_{x\in X_m, a\in A}\Qt(x,a)\ell_t(k|x,a)\\
  			&= \sum_{x\in X_m, a\in A}\Qt(x,a)\sum_{k>m}\sum_{x'\in X_k, a'\in A}\Qt(x',a'|x,a)\ell_{t}(x',a') \\
  			&= \sum_{k>m}\sum_{x'\in X_k, a'\in A}\Qt(x',a')\ell_t(x',a')\\
              &\le \left<\Qt, \ell_t\right>
        \end{align*}
        and
        \begin{align*}
  			\mathbb{E}_t\left[\left(\sum_{k>m}\sum_{x\in X_m, a\in A}\ind_t(x_m,a_m)\ell_t(k|x,a)\right)^2\right]&\le \Lyr\left<\Qt, \ell_t\right>.
  		\end{align*}
  		Using Freedman inequality $\Lyr$ times with parameter $\delta/\Lyr$ for $m=0,1,\dots,\Lyr-1$ and taking a union bound, we have with probability $1-\delta$, for all $m=0,1,\dots,\Lyr-1$,   		\begin{align*}
  			&\sum_{t=1}^T\sum_{k>m}\sum_{x\in X_m, a\in A}\ind_t(x_m,a_m)\ell_t(k|x,a)-\sum_{t=1}^T\left<\Qt, \ell_t\right>\\
              &\le \otil\left(\sqrt{\Lyr\sum_{t=1}^T\left<\Qt, \ell_t\right>\ln \frac{1}{\delta}}+\Lyr\ln\frac{1}{\delta}\right) =\otil\left(\sqrt{\Lyr\LT\ln \frac{1}{\delta}}+\Lyr\ln\frac{1}{\delta}\right) .	
  		\end{align*}
        Therefore, using AM-GM inequality, we have
        \begin{align*}
        \sum_{t=1}^T\sum_{k>m}\sum_{x\in X_m, a\in A}\ind_t(x_m,a_m)\ell_t(k|x,a)\le\otil\left(\LT+\Lyr\ln\frac{1}{\delta}\right).
        \end{align*}
        Combining the  results above and \pref{eq: error-3}, we know that with probability at least $1-\delta$,
        \begin{align}\label{eq: MDP-B1}
        	B_1&\le \otil\left(|X|\sqrt{\Lyr|A|\LT\ln\frac{1}{\delta}}+\Lyr|X|\sqrt{|A|}\ln\frac{1}{\delta}\right)+\otil\left(\Lyr|X|^2|A|+\Lyr|X|\ln\frac{1}{\delta}\right) \nonumber\\
        	& \le \otil\left(|X|\sqrt{\Lyr|A|\LT\ln\frac{1}{\delta}}+|X|^3|A|\ln\frac{1}{\delta}\right).
        \end{align}
        
        Finally, combining \pref{eq: MDP-B2} and \pref{eq: MDP-B1} and considering the probability of the events of \pref{lem: MDP-PinP} and \pref{lem: L10-jin2019}, we have with probability $1-\cprob \delta$,
        \[
        B_1+|X|B_2\le\otil\left(|X|\sqrt{\Lyr|A| \LT\ln\frac{1}{\delta}}+|X|^5|A|^2\ln\frac{1}{\delta}+|X|^4|A|\ln^2\frac{1}{\delta}\right),
        \]
        finishing the proof.
\end{proof}

\subsubsection{Bounding \errorterm}\label{app:MDP-error}

\begin{lemma}\label{lem: error} With probability at least $1-\cprob \delta$, we have
	\begin{align*}
		\errorterm = \sum_{t=1}^T\inner{\Qt-\Qhat_t, \ell_t} &= \mathcal{\tilde{O}}\left( |X|\sqrt{\Lyr|A| \LT\ln\frac{1}{\delta}}+|X|^5|A|^2\ln \frac{1}{\delta}+|X|^4|A|\ln^2\frac{1}{\delta} \right),
	\end{align*}
\end{lemma}
\begin{proof}
	Note that according to the definition of $\wh{w}_t$, the transition function $P^{\wh{w}_t}$ induced by $\wh{w}_t$ is in $\calP_{i_t}$. Therefore, applying \pref{lem:MDP-small-loss-lemma}, we know that with probability at least $1-\cprob \delta$,
	\begin{align*}
		\errorterm &= \sum_{t=1}^T\inner{\wh{w}_t-\Qt, \ell_t} \\
		&\le \sum_{t=1}^T\sum_{x\in X, a\in A}\left|\wh{w}_t(x,a)-\Qt(x,a)\right| \ell_t(x,a) \\
		& \le \mathcal{\tilde{O}}\left(|X|\sqrt{\Lyr|A| \LT\ln\frac{1}{\delta}}+|X|^5|A|^2\ln \frac{1}{\delta}+|X|^4|A|\ln^2\frac{1}{\delta}\right),
	\end{align*}
	completing the proof.
\end{proof}

\subsubsection{Bounding \bias}\label{app:MDP-biasone}

\begin{lemma}\label{lem: bias-1}
	With probability at least $1-7\delta$, we have \[\bias = \sum_{t=1}^T\inner{\Qhat_t, \ell_t-\ellhat_t} \le \mathcal{\tilde{O}}\left(|X|\sqrt{\Lyr|A| \LT\ln \frac{1}{\delta}}+|X|^5|A|^2\ln \frac{1}{\delta}+|X|^4|A|\ln^2\frac{1}{\delta}\right).\]
\end{lemma}

\begin{proof}
First we write
	\begin{align*}
		&\sum_{t=1}^T\left\langle\Qhat_{t},\ell_t-\ellhat_t \right\rangle = \sum_{t=1}^T\left\langle\Qhat_{t},\mathbb{E}_t\left[\ellhat_{t}\right]-\ellhat_t \right\rangle+\sum_{t=1}^T\left\langle\Qhat_{t},\ell_t-\mathbb{E}_t\left[\ellhat_t\right] \right\rangle.
	\end{align*}
	Since $\Qhat_{t}(x,a) \leq \Ut(x,a)$ by the definition of $\Ut$, we have
	\begin{align*}
		\left\langle\Qhat_{t},\ellhat_t \right\rangle &\le \sum_{k=1}^\Lyr\sum_{x\in X_k, a\in A}\frac{\Qhat_{t}(x,a)}{\Ut(x,a)}\cdot \ind_t(x,a)\le \Lyr,\\
		\mathbb{E}_t\left[\left\langle\Qhat_{t},\ellhat_t \right\rangle^2\right] &\le \mathbb{E}_t\left[\Lyr\cdot \left\langle\Qhat_{t},\ellhat_t \right\rangle\right] =\Lyr \sum_{x,a}\Qhat_{t}(x,a)\cdot \frac{\ell_t(x,a)}{\Ut(x,a)}\cdot\QQ_t(x,a) \le \Lyr\cdot  \left\langle\QQ_{t},\ell_t \right\rangle,
	\end{align*}
	and thus according to Freedman inequality, we have with probability at least $1-\delta$,
	\begin{align}\label{eq: bias-1-1}
		\sum_{t=1}^T\left\langle\Qhat_{t},\mathbb{E}_t\left[\ellhat_{t}\right]-\ellhat_t \right\rangle \le \mathcal{O}\left(\sqrt{\Lyr\sum_{t=1}^T\left\langle\QQ_{t},\ell_t \right\rangle\ln\frac{1}{\delta}}+\Lyr\cdot\ln\frac{1}{\delta}\right)=\mathcal{O}\left(\sqrt{\Lyr\LT\ln\frac{1}{\delta}}+|X|\ln\frac{1}{\delta}\right).
	\end{align}
	
	For the second term, we have
	\begin{align*}
		\sum_{t=1}^T\left\langle\Qhat_{t},\ell_t-\mathbb{E}_t\left[\ellhat_t\right] \right\rangle = \sum_{t,x,a}\Qhat_{t}(x,a)\ell_{t}(x,a)\cdot\left(1-\frac{\QQ_t(x,a)}{\Ut(x,a)}\right)\le \sum_{t,x,a}|\Ut(x,a)-\QQ_t(x,a)|\cdot \ell_t(x,a).
	\end{align*}
	
	By the definition of $\upconf_t$, one has $\upconf_t=w^{P_t^x,\pi_t}$ for $P_t^x=\argmax_{\hat{P}\in \calP_{i_t}}\sum_a w^{\hat{P},\pi_t}(x,a)$. Therefore, according to \pref{lem: error}, we have with probability at least $1-\cprob \delta$,
	\begin{align}\label{eq: bias-1-2}
		\sum_{t=1}^T\left\langle\Qhat_{t},\ell_t-\mathbb{E}_t\left[\ellhat_t\right] \right\rangle \le \mathcal{\tilde{O}}\left(|X|\sqrt{\Lyr|A| \LT\ln\frac{1}{\delta}}+|X|^5|A|^2\ln \frac{1}{\delta}+|X|^4|A|\ln^2\frac{1}{\delta}\right).
	\end{align}
	
	Combining \pref{eq: bias-1-1} and \pref{eq: bias-1-2} proves the lemma.
	
\end{proof}

\subsubsection{Bounding \biastwo}\label{app:MDP-biastwo}

\begin{lemma}\label{lem: bias-2}
	With probability at least $1-5\delta$, 
	we have 
	\begin{align*}
	\biastwo= \sum_{t=1}^T\inner{\Qstar, \ellhat_t-\ell_t} &\le C\sum_{x\in X, a\in A}u(x,a)\sqrt{8\rho_T(x,a)\sum_{t=1}^T\ell_t(x,a)\ln \frac{C|X||A|}{\delta}} \\
	&\qquad +2C\inner{u, \rho_T}\ln \frac{C|X||A|}{\delta},
	\end{align*}
for some constant $C = \otil(1)$.
\end{lemma}

\begin{proof}
First we write
	\begin{align*}
		\sum_{t=1}^T\left\langle\Qstar,\ellhat_t-\ell_t \right\rangle = \sum_{t=1}^T \left\langle\Qstar,\mathbb{E}_t\left[\ellhat_t\right]-\ell_t \right\rangle+\sum_{t=1}^T\left\langle\Qstar,\ellhat_t-\mathbb{E}_t\left[\ellhat_t\right] \right\rangle.
	\end{align*}
	
	The first term is nonpositive under the event of \pref{lem: MDP-PinP} as for any $(x,a)\in X\times A$, $\QQ_t(x,a) \leq \Ut(x,a)$ by the definition of $\Ut$ and thus
	\begin{align}\label{eq: biastwo-1}
		\mathbb{E}_t\left[\ellhat_t(x,a)\right]-\ell_t(x,a) = \QQ_t(x,a)\cdot \frac{\ell_t(x,a)}{\Ut(x,a)}-\ell_t(x,a)\le 0.
	\end{align}
	
	For the second term, note that for each $(x,a)\in X\times A$, we have
	\begin{align*}
		\ellhat_t(x,a) = \frac{\ell_t(x,a)}{\Ut(x,a)}\cdot \ind_t(x,a)
		\le T^3|X|^2|A|, \tag{\pref{lem:upper-occu-bound}}\\
		\ellhat_t(x,a) = \frac{\ell_t(x,a)}{\Ut(x,a)}\cdot \ind_t(x,a)\le \rho_t(x,a),
    \end{align*}
    and
    \begin{align*}
		\sum_{t=1}^T\mathbb{E}_t\left[\ellhat_t(x,a)^2\right]&\le \sum_{t=1}^T\mathbb{E}_t\left[\frac{\ell_{t}(x,a)}{\Ut(x,a)^2}\cdot\ind_t(x,a)\right] \le \rho_T(x,a)\sum_{t=1}^T\ell_t(x,a).
	\end{align*}
	
	Therefore, using 
	\pref{thm:Freedman} with $X_t = \ellhat_t(x,a)-\E_t\sbr{\ellhat_t(x,a)}$, $B_t = \rho_t(x,a)$, $B^\star = \rho_T(x,a)$, $b=T^3|X|^2|A|$, $C=\ceil{\log_2b}\ceil{\log_2b^2T}=\otil(1)$, we have with probability at least $1-\frac{\delta}{|X||A|}$,
	\begin{align*}
		\sum_{t=1}^T\ellhat_t(x,a)-\E_t\sbr{\ellhat_t(x,a)} \le C\left(\sqrt{8\rho_T(x,a)\sum_{t=1}^T\ell_t(x,a)\ln\frac{C|X||A|}{\delta}}+2\rho_T(x,a)\ln\frac{C|X||A|}{\delta}\right).
	\end{align*}
	
	Taking a union bound over all $(x,a)\in X\times A$, multiplying both sides by $u(x,a)$, and summing up all these inequalities, we have with probability at least $1-\delta$,
	\begin{align}\label{eq: biastwo-2}
		&\sum_{t=1}^T\inner{u, \ellhat_t-\mathbb{E}_t\left[\ellhat_t\right]} \nonumber\\
		&\le C\sum_{x\in X, a\in A}u(x,a)\left(\sqrt{8\rho_T(x,a)\sum_{t=1}^T\ell_t(x,a)\ln\frac{C|X||A|}{\delta}}+2\rho_T(x,a)\ln\frac{C|X||A|}{\delta}\right).
	\end{align}
	
	Combining \pref{eq: biastwo-1} and \pref{eq: biastwo-2} finishes the proof.

\end{proof}

\subsubsection{Bounding \regterm}\label{app:MDP-regterm}
\begin{lemma}\label{lem: reg}
With probability at least $1-4\delta$, 
we have 
	 \[\regterm = \sum_{t=1}^T\inner{\Qhat_t - \Qstar, \ellhat_t} \le \mathcal{\tilde{O}}\left(\frac{|X|^2|A|}{\eta}\right)+5\eta\LTbar-\frac{\inner{u,\rho_T}}{70\eta\ln T},\]
	where $\LTbar=\sum_{t=1}^T \sum_{x\in X, a\in A}\ind_t(x,a)\ell_t(x,a)$.
\end{lemma}

\begin{proof}
We condition on the event of \pref{lem: MDP-PinP}.
	First, we prove that $\Qstar\in \Delta(\calP_i)\cap\Omega$ for all $i$ (recall its definition in \pref{eq:MDP_u_def}). Indeed, for any fixed $(x,a,x')\in X_k\times A\times X_{k+1}$, $k=0,1,\dots,\Lyr-1$, we have (with $\QQ^{P_0, \pi_a}(x)$ being the probability of visiting $x$ under $P_0$ and $\pi_a$)
	\begin{align*}
		\Qstar(x,a,x')&\ge \frac{1}{T|A|}\QQ^{P_0, \pi_a}(x,a,x') \\
		& = \frac{1}{T|A|}\QQ^{P_0, \pi_a}(x)P_0(x'|x,a) \\
		&\ge \frac{1}{T|A|}\left(\sum_{x''\in X_{k(x)-1}}\QQ^{P_0, \pi_a}(x'')\cdot P_0(x|x'',a)\right)\cdot \frac{1}{T|X|}\tag{\pref{lem:p0xxa}}\\
		&\ge \frac{1}{T^3|X|^2|A|}\left(\sum_{x''\in X_{k(x)-1}}\QQ^{P_0, \pi_a}(x'')\right), \tag{\pref{lem:p0xxa} again} \\
		&= \frac{1}{T^3|X|^2|A|},
	\end{align*}
	which shows $\Qstar\in \Omega$.
	On the other hand, since $P\in \calP_i$ under \pref{lem: MDP-PinP} and $P_0 \in \calP_i$ as well by \pref{lem:p0xxa}, we have $u^\star \in \Delta(\calP_i)$ and $\QQ^{P_0, \pi_a} \in \Delta(\calP_i)$,
	which indicates that, as a convex combination of $u^\star$  and $\QQ^{P_0, \pi_a}$ for all $a$, $u$ has to be in $\Delta(\calP_i)$ as well.
	
	Therefore, by standard OMD analysis (e.g.,~\citep[Lemma~12]{agarwal2017corralling}), we have
    \begin{align*}
		&\left<\Qhat_t-\Qstar, \ellhat_t\right> \\
        &\le D_{\psi_t}(\Qstar, \Qhat_t)-D_{\psi_t}(\Qstar, \Qhat_{t+1})+\sum_{k=0}^{\Lyr-1}\sum_{(x,a,x')\in X_k \times  A \times X_{k+1}}\eta_t(x,a)\Qhat_{t}^2(x,a,x')\ellhat_{t}^2(x,a)\\
        &\le D_{\psi_t}(\Qstar, \Qhat_t)-D_{\psi_t}(\Qstar, \Qhat_{t+1})+\sum_{k=0}^{\Lyr-1}\sum_{(x,a)\in X_k \times  A }\eta_t(x,a)\Qhat_{t}^2(x,a)\ellhat_{t}^2(x,a)\tag{$\sum_{x'\in X_{k+1}}\Qhat_{t}(x,a,x')^2\le \Qhat_{t}(x,a)^2$}\\ 
        &\leq D_{\psi_t}(\Qstar, \Qhat_t)-D_{\psi_t}(\Qstar, \Qhat_{t+1})+\sum_{x\in X,a\in A}\eta_t(x,a)\ind_t(x,a)\ell_t(x,a) \tag{$\Qhat_{t}(x,a) \leq \upconf_t(x,a)$}.
	\end{align*}
	Summing over $t$ gives
		\begin{align}
		\sum_{t=1}^T\left<\Qhat_t-\Qstar, \ellhat_t\right>&\le D_{\psi_1}(\Qstar,\Qhat_1)+\sum_{t=1}^{T-1}D_{\psi_{t+1}}(\Qstar,\Qhat_{t+1})-D_{\psi_t}(\Qstar,\Qhat_{t+1}) \notag \\
        &\quad+\sum_{t=1}^T\sum_{x\in X,a\in A}\eta_t(x,a)\ind_t(x,a)\ell_t(x,a). \label{eq:MDP-stability}
    \end{align}
	
	Next, for a fixed $(x,a)$ pair, let $n(x,a)$ be the total number of times the learning rate for $(x,a)$ has increased, such that $\eta_T(x,a) = \eta\kappa^{n(x,a)}$, and let $t_1,\dots,t_{n(x,a)}$ be the rounds where $\eta_t(x,a)$ is increased, such that $\eta_{t_i+1}(x,a) = \eta_{t_i}(x,a)\kappa$.
	Then since $\frac{1}{\upconf_{t_{n(x,a)}+1}(x,a)}>\rho_{t_{n(x,a)}}(x,a)>2\rho_{t_{n(x,a)-1}}(x,a)>\dots>2^{n(x,a)-1}\rho_1(x,a)>2^{n(x,a)}|A|$ and $\frac{1}{\upconf_{t_{n(x,a)}+1}(x,a)}\le T^3|X|^2|A|$ (\pref{lem:upper-occu-bound}), we have $n\le \log_2\left(T^3|X|^2\right)\le 7\log_2 T$.
	
Therefore, we have $\eta_t(x,a)\le \eta e^{\frac{7\log_2 T}{7\ln T}}\le 5\eta$ for any $t$, $x\in X$, and $a\in A$, and the last term in \pref{eq:MDP-stability} is thus bounded by $5\eta\LTbar$. 
For the second term, with $h(y)=y-1-\ln y$ , we have
    \begin{align*}
        &\sum_{t=1}^{T-1}D_{\psi_{t+1}}(\Qstar,\Qhat_{t+1})-D_{\psi_t}(\Qstar,\Qhat_{t+1})\\
		&\le \sum_{t=1}^{T-1}\sum_{k=0}^{\Lyr-1}\sum_{x\in X_k, a\in A, x'\in X_{k+1}}\left(\frac{1}{\eta_{t+1}(x,a)}-\frac{1}{\eta_{t}(x,a)}\right)h\left(\frac{\Qstar(x,a,x')}{\Qhat_{t+1}(x,a,x')}\right)\\
		&\le \sum_{k=0}^{\Lyr-1}\sum_{x\in X_k, a\in A, x'\in X_{k+1}}\frac{1-\kappa}{\eta\cdot \kappa^{n(x,a)}}\cdot h\left(\frac{\Qstar(x,a,x')}{\Qhat_{t_{n(x,a)}+1}(x,a,x')}\right) \\
		&\le -\frac{1}{35\eta\ln T}\sum_{k=0}^{\Lyr-1}\sum_{x\in X_k, a\in A, x'\in X_{k+1}}h\left(\frac{\Qstar(x,a,x')}{\Qhat_{t_{n(x,a)}+1}(x,a,x')}\right) \tag{$1-\kappa\le -\frac{1}{7\ln T}$ and $\kappa^{n(x,a)}\le e^{\frac{7\log_2 T}{7\ln T}}\le 5$}\\
		&= -\frac{1}{35\eta\ln T}\sum_{k=0}^{\Lyr-1}\sum_{x\in X_k, a\in A, x'\in X_{k+1}}\left(\frac{\Qstar(x,a,x')}{\Qhat_{t_{n(x,a)}+1}(x,a,x')}-1-\ln\frac{u(x,a,x')}{\Qhat_{t_{n(x,a)}+1}(x,a,x')}\right)\\
		&\le \frac{|X|^2|A|(1+6\ln T)}{35\eta\ln T}-\frac{1}{35\eta\ln T}\sum_{k=0}^{\Lyr-1}\sum_{x\in X_k, a\in A, x'\in X_{k+1}} \frac{\Qstar(x,a,x')}{\Qhat_{t_{n(x,a)}+1}(x,a,x')} \tag{$\ln\frac{\Qstar(x,a,x')}{\Qhat_{t_{n(x,a)}+1}(x,a,x')} \le 6\ln T$}\\
		&\le \frac{|X|^2|A|}{5\eta}-\frac{1}{35\eta\ln T}\sum_{k=0}^{\Lyr-1}\sum_{x\in X_k, a\in A, x'\in X_{k+1}} \frac{\Qstar(x,a,x')}{\upconf_{t_{n(x,a)}+1}(x,a)} \\
		&= \frac{|X|^2|A|}{5\eta}-\frac{1}{35\eta\ln T}\sum_{k=0}^{\Lyr-1}\sum_{x\in X_k, a\in A} \frac{\Qstar(x,a)}{\upconf_{t_{n(x,a)}+1}(x,a)} \\
		&= \frac{|X|^2|A|}{5\eta}-\frac{\inner{u,\rho_T}}{70\eta\ln T} \tag{$\rho_T(x,a)=\frac{2}{\upconf_{t_{n(x,a)}+1}(x,a)}$}.
	\end{align*}

Finally, we bound the first term in \pref{eq:MDP-stability}:
    	\begin{align*}
    		D_{\psi_1}(\Qstar,\Qhat_1)&=\frac{1}{\eta}\left(\sum_{k=0}^{\Lyr-1}\sum_{(x,a,x')\in X_k\times A\times X_{k+1}}h\left(\frac{\Qstar(x,a,x')}{\Qhat_1(x,a,x')}\right)\right)\\
    		&=\frac{1}{\eta}\left(\sum_{k=0}^{\Lyr-1}\sum_{(x,a,x')\in X_k\times A\times X_{k+1}}h\left(|X_k|\cdot|A|\cdot |X_{k+1}|\cdot \Qstar(x,a,x')\right)\right)\\
    		&=  \frac{1}{\eta}\left(\sum_{k=0}^{\Lyr-1}\sum_{(x,a,x')\in X_k\times A\times X_{k+1}}\ln\left( \frac{1}{|X_k|\cdot|A|\cdot|X_{k+1}|\cdot \Qstar(x,a,x')}\right)\right) \\
    		&\le \mathcal{\tilde{O}}\left(\frac{|X|^2|A|}{\eta}\right).
    \end{align*}
    Combining all the bounds finishes the proof.
\end{proof}

\subsubsection{Putting everything together}\label{app:MDP-together}

Now we are ready to prove \pref{thm:MDP}. For completeness, we restate the theorem below.

\begin{theorem}
	\pref{alg:MDP} with a suitable choice of $\eta$ ensures that with probability at least $1-\delta$, $
	\Reg = \otil\left(|X|\sqrt{\Lyr|A| L^{\star}\ln \frac{1}{\delta}}+|X|^5|A|^2\ln^2\frac{1}{\delta}\right).
	$
\end{theorem}
\begin{proof}
		First, note that 
		\begin{align*}		
		\mathbb{E}_t\left[\sum_{k=0}^{\Lyr-1}\sum_{x\in X_k, a\in A}\ind_t(x,a)\cdot \ell_t(x,a)\right]&=\left<\Qt, \ell_t\right>\le \Lyr,\\
		\mathbb{E}_t\left[\left(\sum_{k=0}^{\Lyr-1}\sum_{x\in X_k, a\in A}\ind_t(x,a)\cdot \ell_t(x,a)\right)^2\right]&\le \Lyr\cdot \left<\Qt, \ell_t\right>.
		\end{align*}
		 Therefore, using Freedman's inequality, we have with probability at least $1-\delta$
		 \begin{align*}
		 \LTbar -\LT \le 2\sqrt{\Lyr \LT\ln\frac{1}{\delta}}+\Lyr\ln \frac{1}{\delta},
		 \end{align*}
		 where $\LTbar$ is defined in \pref{lem: reg}.
		 Furthermore, using AM-GM inequality, we have with probability at least $1-\delta$,
		 \begin{align}\label{eq:LtarLT}
		 \LTbar\le 2\LT+2\Lyr\ln \frac{1}{\delta}.
		 \end{align} 
		Choosing $\eta \le \frac{1}{280C\ln(C|X||A|/\delta)\ln T}$, combining \pref{lem: error}, \pref{lem: bias-1}, \pref{lem: bias-2} and \pref{lem: reg} and letting $L_{\Qstar}\triangleq \sum_{t=1}^T\left<\Qstar,\ell_t\right>$, we have with probability at least $1-22\delta$:
		\begin{align*}
			&\LT-\Lstar\\
			&\le \mathcal{\tilde{O}}\left(|X|\sqrt{\Lyr|A| \LT\ln \frac{1}{\delta}}+\frac{|X|^2|A|}{\eta}\right)+5\eta\LTbar+\underbrace{\left(2C\inner{u,\rho_T}\ln\frac{C|X||A|}{\delta}-\frac{\inner{u,\rho_T}}{140\eta\ln T}\right)}_{\textsc{term1}}\\
            &\quad+\sum_{x\in X, a\in A}u(x,a)\underbrace{\left(C\sqrt{8\rho_T(x,a)\sum_{t=1}^T\ell_t(x,a)\ln\frac{C|X||A|}{\delta}}-\frac{\rho_T(x,a)}{140\eta\ln T}\right)}_{\textsc{term2}}\\
            &\quad+\otil\left(|X|^5|A|^2 \ln \frac{1}{\delta}+|X|^4|A|\ln^2\frac{1}{\delta}\right)\\
			&\le \mathcal{\tilde{O}}\left(|X|\sqrt{\Lyr|A|\LT\ln \frac{1}{\delta}}+\frac{|X|^2|A|}{\eta}+\eta L_{\Qstar}\ln\frac{1}{\delta}+|X|^5|A|^2 \ln \frac{1}{\delta}+|X|^4|A|\ln^2\frac{1}{\delta}\right)+10\eta\LT \tag{\textsc{term1} is nonpositive, AM-GM inequality for \textsc{term2}, and \pref{eq:LtarLT}}\\
			&\le \mathcal{\tilde{O}}\left(|X|\sqrt{\Lyr|A|\LT\ln \frac{1}{\delta}}+\frac{|X|^2|A|}{\eta}+\eta\Lstar\ln\frac{1}{\delta}+|X|^5|A|^2\ln \frac{1}{\delta}+|X|^4|A|\ln^2\frac{1}{\delta}\right)+10\eta\LT\tag{\pref{eq:bias3}}.
		\end{align*}
		As $\eta\le \frac{1}{280C\ln(C|X||A|/\delta)\ln T}\le \frac{1}{20}$, rearranging the terms gives
		\begin{align*}
			\LT-\Lstar\le \mathcal{\tilde{O}}\left(|X|\sqrt{\Lyr|A|\LT\ln \frac{1}{\delta}}+\frac{|X|^2|A|}{\eta}+\eta \Lstar\ln\frac{1}{\delta}+|X|^5|A|^2\ln \frac{1}{\delta}+|X|^4|A|\ln^2\frac{1}{\delta}\right).
		\end{align*}
		Finally, choosing $\eta = \min\left\{\sqrt{\frac{|X|^2|A|}{\Lstar\ln \frac{1}{\delta}}}, \frac{1}{280C\ln(C|X||A|/\delta)\ln T}\right\}$, $\delta=\delta'/22$, and solving the quadratic inequality, we have with probability at least $1-\delta'$,
		\begin{align*}
			\LT-\Lstar\le \mathcal{\tilde{O}}\left(|X|\sqrt{\Lyr|A|\Lstar\ln\frac{1}{\delta'}}+|X|^5|A|^2\ln \frac{1}{\delta'}+|X|^4|A|\ln^2\frac{1}{\delta'}\right),
		\end{align*}
		finishing the proof.
	\end{proof}

\begin{remark}\emph{
Similarly to the MAB case, the proof above requires tuning the initial learning rate $\eta$ in terms of the unknown quantity $\Lstar$, and again, using standard doubling trick can remove this restriction, as pointed out in \pref{rem:doubling_trick}.
}
\end{remark}

\subsection{Issues of other potential approaches}\label{app:MDP_discussions}

In this section, we discuss why the idea of clipping \cite{AllenbergAuGyOt06} or implicit exploration \citep{neu2015explore} may not be directly applicable to achieve near-optimal high-probability small-loss bounds. 

\paragraph{Implicit exploration.}
First, we consider the idea of implicit exploration. 
As mentioned in \pref{app:MDP-alg}, this means using the following loss estimator:
$\ellhat_t = \frac{\ell_t(x,a)}{\upconf_t(x,a)+\gamma}\cdot \ind_t\{x,a\}$ for all $x\in X$ and $a\in A$ and some parameter $\gamma > 0$, and without using our increasing learning schedule. 
The concentration results of~\citep[Lemma~12]{jin2019learning} show that the deviation contains a term of order $1/\gamma$, meaning that $\gamma$ cannot be too small.

Repeating the same analysis, one can see that the main difficulty of obtaining high-probability small-loss bounds in this case is to bound \biastwo~by the loss of the algorithm $\LT = \sum_{t=1}^T\inner{w_t, \ell_t}$ or $L^\star$, instead of the number of episodes $T$. Indeed, consider the term $\sum_{t=1}^T\inner{\what_t, \ell_t-\mathbb{E}_t\left[\ellhat_t\right]}$:
\begin{align*}
	\sum_{t=1}^T\inner{\what_t, \ell_t-\mathbb{E}_t\left[\ellhat_t\right]} &= \sum_{t=1}^T\sum_{x\in X, a\in A}\what_t(x,a)\ell_t(x,a)\cdot \left(1-\frac{w_t(x,a)}{\upconf_t(x,a)+\gamma}\right) \\
	&\le \sum_{t=1}^T\sum_{x\in X, a\in A}|\upconf_t(x,a)-w_t(x,a)|\ell_t(x,a)+\frac{\gamma}{\gamma+\upconf_t(x,a)}\cdot \what_t(x,a)\ell_t(x,a).
\end{align*}

The first term can still be bounded by $\mathcal{O}\left(|X|\sqrt{\Lyr|A|\LT}+|X|^5|A|^2\ln \frac{1}{\delta}+|X|^4|A|\ln^2 \frac{1}{\delta}\right)$ according to \pref{lem: bias-1}. 
For the second term,
while it is at most $\gamma\sum_{t=1}^T\sum_{x,a}\ell_t(x,a) \leq \gamma|X||A|T$, 
it is not clear at all how to bound it in terms of $\LT$ or $L^\star$.
For MAB (where there is only one state $x_0$), it is possible to show that $\sum_{t=1}^T\ell_t(x_0, a) \leq \sum_{t=1}^T\ell_t(x_0, a^\star) + \otil(\frac{1}{\eta}+\frac{1}{\gamma})$ for all $a\neq a^\star$ where $a^\star$ is the best action, making it possible to connect the second term with $L^\star$.
However, we do not see a way of doing similar analysis for general MDPs.


\paragraph{Clipping.}
On the other hand, the idea of clipping for MAB is to clip all small probabilities so that actions with probability smaller than $\gamma$ are never selected.
Even from an algorithmic perspective, it is not clear how to generalize this idea to MDPs,
because it is possible that for a state $x$, $\Qhat_t(x,a)$ is smaller than $\gamma$ for {\it all} $a$.
In this case, the clipping idea suggests not to ``pick'' $(x,a)$ at all for any $a$, but there is no way to ensure that if the transition function is such that $x$ can always be visited with some positive probability regardless of the policy we execute.

Moreover, even if there is a way to fix this, the analysis of clipping for MAB is also similar to the idea of implicit exploration in terms of obtaining small-loss bounds of order $\otil(\sqrt{\Lstar})$, and as we argued already, even for implicit exploration there are difficulties in generalizing the analysis to MDPs.

\end{document}